\documentclass{article}

\usepackage[final]{neurips_2024} 

\usepackage[utf8]{inputenc} 
\usepackage[T1]{fontenc}    
\usepackage{hyperref}       
\hypersetup{colorlinks,citecolor=blue,linkcolor=red,urlcolor=blue}
\usepackage{url}            
\usepackage{booktabs}       
\usepackage{amsfonts}       
\usepackage{nicefrac}       
\usepackage{microtype}      
\usepackage{xcolor}         
\usepackage{subfigure}
\usepackage{wrapfig}
\usepackage{makecell}
\usepackage{multirow}
\usepackage{tablefootnote}
\usepackage{graphicx}

\usepackage{tcolorbox}
\usepackage{tcolorbox}
\colorlet{myred}{red!20!white}
\definecolor{myblue}{rgb}{0.0, 0.0, 0.65}
\colorlet{mygreen}{green!20!white}

\usepackage{algorithm}
\usepackage{algorithmic}
\usepackage{amsmath}
\usepackage{amssymb}
\usepackage{amsthm}
\usepackage[capitalise]{cleveref}
\usepackage{dsfont}
\usepackage{enumitem}
\usepackage{thm-restate}
\usepackage{xspace}
\usepackage{tikz}
\usepackage{wrapfig}
\usetikzlibrary{bayesnet}

\newcommand{\indicator}{\mathds{1}}

\newcommand{\bA}{\mathbb{A}}
\newcommand{\bE}{\mathbb{E}}
\newcommand{\bN}{\mathbb{N}}
\newcommand{\bP}{\mathbb{P}}
\newcommand{\bR}{\mathbb{R}}

\newcommand{\cA}{\mathcal{A}}

\newcommand{\cC}{\mathcal{C}}
\newcommand{\cD}{\mathcal{D}}
\newcommand{\cE}{\mathcal{E}}
\newcommand{\cG}{\mathcal{G}}
\newcommand{\cH}{\mathcal{H}}
\newcommand{\cN}{\mathcal{N}}
\newcommand{\cF}{\mathcal{F}}

\newcommand{\diag}{\mathrm{diag}}
\newcommand{\Norm}{\mathrm{Norm}}

\newcommand{\TV}{\mathrm{TV}}
\newcommand{\KL}{\mathrm{KL}}

\newcommand{\Rdesignmat}{\mathbf{Z}}
\newcommand{\designmat}{\mathbf{V}}
\newcommand{\diagcounts}{\mathbf{N}}
\newcommand{\SSigma}{\mathbf{\Sigma}}
\newcommand{\diagmat}{\mathbf{d}}
\newcommand{\diagI}{\mathbf{I}}
\newcommand{\Sq}{\mathbf{S}}

\newcommand\OLSUCBC{\texttt{OLS-UCB-C}\xspace}
\newcommand\COSV{\texttt{COS-V}\xspace}
\newcommand\CUCB{\texttt{CUCB}\xspace}
\newcommand\OLSUCB{\texttt{OLS-UCB}\xspace}
\newcommand\ESCBC{\texttt{ESCB-C}\xspace}
\newcommand{\UCB}{\texttt{UCB}\xspace}
\newcommand{\UCBV}{\texttt{UCBV}\xspace}
\newcommand{\TS}{\texttt{TS}\xspace}
\newcommand{\CTSG}{\texttt{CTS-Gaussian}\xspace}
\newcommand{\CTS}{\texttt{CTS}\xspace}

\newcommand{\LinUCB}{\texttt{LinUCB}\xspace}

\makeatletter
\newenvironment{framework}[1][htb]{%
    \renewcommand{\ALG@name}{Framework}
    \begin{algorithm}[#1]%
    }{\end{algorithm}
}
\makeatother

\title{Towards Efficient and Optimal Covariance-Adaptive Algorithms for Combinatorial Semi-Bandits}

\author{%
  Julien Zhou 
  \vspace{2mm} \\
  Criteo AI Lab\\
  Paris, France
  \vspace{2mm}\\
  Univ. Grenoble Alpes, Inria, \\
  CNRS, Grenoble INP, LJK, \\
  38000 Grenoble, France 
  \vspace{2mm} \\
  \texttt{julien.zhou@inria.fr} \\  
  \And
  Pierre Gaillard 
  \vspace{2mm} \\
  Univ. Grenoble Alpes, Inria, \\
  CNRS, Grenoble INP, LJK, \\
  38000 Grenoble, France \\
  \And
  Thibaud Rahier 
  \vspace{2mm} \\
  Criteo AI Lab, \\
  Paris, France \\
  \And
  Houssam Zenati 
  \vspace{2mm} \\
  Université Paris-Saclay, Inria, \\
  Palaiseau, France \\
  \And
  Julyan Arbel 
  \vspace{2mm} \\
  Univ. Grenoble Alpes, Inria, \\
  CNRS, Grenoble INP, LJK, \\
  38000 Grenoble, France \\
}

\begin{document}

\maketitle

\begin{abstract}
We address the problem of stochastic combinatorial semi-bandits, where a player selects among $P$ actions from the power set of a set containing $d$ base items. Adaptivity to the problem's structure is essential in order to obtain optimal regret upper bounds. As estimating the coefficients of a covariance matrix can be manageable in practice, leveraging them should improve the regret. We design ``optimistic''  covariance-adaptive algorithms relying on online estimations of the covariance structure, called \OLSUCBC and \COSV (only the variances for the latter). They both yield improved gap-free regret. Although \COSV can be slightly suboptimal, it improves on computational complexity by taking inspiration from Thompson Sampling approaches. It is the first sampling-based algorithm satisfying a $O(\sqrt{T})$ gap-free regret (up to poly-logs). We also show that in some cases, our approach efficiently leverages the semi-bandit feedback and outperforms bandit feedback approaches, not only in exponential regimes where $P\gg d$ but also when $P\leq d$, which is not covered by existing analyses.
\end{abstract}

\section{Introduction}
\label{sec:Introduction}

In sequential decision-making, the bandit framework has been extensively studied and was instrumental to several applications, e.g. A/B testing \citep{guo2020deep}, online advertising and recommendation services \citep{zeng2016online}, network routing \citep{tabei2023design}, demand-side management \citep{bregere19a}, etc.
Its popularity stems from its relative simplicity, allowing it to model and analyze a wide range of challenging real-world settings. Reference books like \citet{bubeck_regret_2012} or \citet{lattimore2020bandit} offer a wide perspective on the subject. 

In this framework, a \emph{decision-maker} or \emph{player} must make choices and receives associated rewards, but it lacks prior knowledge of its environment. This naturally leads to an exploration-exploitation trade-off: the player must explore different actions to determine the best one, but an inefficient exploration strategy may harm the cumulative rewards. Efficient algorithms rely on exploiting the environment's structure, such as estimating the parameters of a reward function instead of exploring every action.

This paper focuses on the stochastic combinatorial semi-bandit framework. At each round, the player chooses a subset of \emph{base items} and receives a feedback for each item chosen.
The action set is included in the base items' power set, and can therefore be exponentially big and difficult to explore. 
The main challenge in this framework is to effectively combine the information collected through different actions
(that may share common base items).

\paragraph{Problem formulation.}

We consider a set of \(d\in\bN^*\) \emph{base items}, each item \(i\in[d]=\{1, \dots, d\}\) yielding stochastic rewards. A \emph{player} accesses these rewards through a set \(\cA\subseteq\{0, 1\}^d\) of $P\in\bN^*$ \emph{actions}, each corresponding to a subset of at most $m\geq 5$ items\footnote{Throughout the paper, the term \emph{item} (or \emph{base item}) refers to an element in the set $[d]$, while an \emph{action} denotes a subset of base items in $\cA$.}. We refer to actions $a\in\cA$ using their components vector \(a = (a_i)_{i\in[d]} \in\{0, 1\}^d\) where for all \(j\in[d]\), \(a_j=1\) if and only if action $a$ contains base item $j$.

The player interacts with an \emph{environment} over a sequence of \(T \in \bN^*\) \emph{rounds}. At each round \(t\in[T]\), the player chooses an action \(A_t\in\cA\), the environment samples a reward vector $Y_t \in \bR^d$, the player observes the realization for every item contained in $A_t$, and receives their sum. 
The interactions between the player and the environment are summarized in Framework~\ref{alg:procedure}.

\begin{framework}[ht]
   \caption{Stochastic Combinatorial Semi-Bandit}
   \label{alg:procedure}
    \begin{algorithmic}
        \STATE For each \(t\in\{1, \dots, T\}\):
        \STATE \begin{itemize}
            \item The player chooses an action \(A_t \in \cA\).
        \end{itemize}
        \STATE \begin{itemize}
            \item The environment samples a vector of rewards \(Y_t\in\bR^d\) from a fixed unknown distribution.
        \end{itemize} 
        \STATE \begin{itemize}
            \item The player receives the  reward \(\langle A_t, Y_t\rangle = \sum_{i}A_{t,i}Y_{t,i}\).
        \end{itemize}  
        \STATE \begin{itemize}[nosep]
            \item The player observes \(Y_{t,i}\) for all $i\in[d]$ s.t. \(A_{t,i}=1\).
        \end{itemize}
    \end{algorithmic}
\end{framework}

\paragraph{Assumptions.}
We make the following assumptions.
For all $t\in[T]$, \(Y_t\) is independent of the past rewards and the player's decision \(\sigma(A_1, Y_1, \dots, A_{t-1}, Y_{t-1}, A_t)\). There exists a mean reward vector \(\bE[Y_t] = \mu\in\bR^d\) and a second order moment matrix $\Sq=\bE[Y_tY_t^\top]\in M_d(\bR)$. The positive semi-definite covariance matrix is denoted $\SSigma \in M_d(\bR)$, with $\SSigma = \Sq-\mu\mu^\top$.There exists a known \textit{vector} \(B\in\bR_+^d\) such that for all $t\in[T]$ and \(i\in[d]\), \(|Y_{t,i}| \leq B_i/2\) almost surely (and $|Y_{t,i}-\mu_i|\leq B_i)$.

The objective of the decision-maker is to minimize the  expected cumulative pseudo-regret defined as:
\begin{equation}\label{eq:regret}
    \textstyle \bE[R_T]
    =\bE\left[\sum_{t=1}^T\langle a^*-A_t, \mu\rangle\right] 
    = \sum_{t=1}^T\Delta_{A_t},
\end{equation}
where \( \langle \cdot, \cdot \rangle\) denotes the usual inner product in \(\bR^d\), \(a^* \in \arg\max_{a\in\cA} \langle a, \mu\rangle\) is an optimal action, and \(\Delta_a = \langle a^*-a, \mu\rangle\) is the \emph{sub-optimality gap} for action \(a \in \cA\).

\subsection{Existing work and limitations}

Combinatorial semi-bandit problems have been extensively studied by the bandit community since their introduction by \cite{chen2013combinatorial}. Here, we only highlight key earlier works related to this paper. For a comprehensive introduction to this literature, we refer the interested reader to the monograph by \cite{lattimore2020bandit}.

A first line of works considers deterministic algorithms based on the optimistic principle and upper confidence bounds (UCBs). 
\cite{chen2013combinatorial} first designed \CUCB, computing UCBs for the items' average rewards, converting these into UCBs for the actions' rewards, and choosing the action with the highest one. 
It was later analyzed by \cite{kveton2015tight}, who proved a regret upper bound uniform over all possible covariance matrices $\smash \SSigma$ (hence, paying the worst-case). 
\cite{combes2015combinatorial} highlighted the importance of designing $\smash{\SSigma}$-adaptive algorithms by showing that the regret could be improved by a factor of $m$ when the items' average rewards are independent. 
Subsequently, \cite{degenne2016combinatorial} developed \OLSUCB, an algorithm intended to leverage the covariance structure.
However, \OLSUCB requires prior knowledge of a positive semi-definite covariance-proxy matrix ${\mathbf{\Gamma}}$, such that for all $t\geq 1$ and for all $u\in \bR^d$, $\bE[\exp(\langle u, Y_t-\mu\rangle)] \leq \exp(\smash{\frac{1}{2}}\Vert u\Vert_{\mathbf{\Gamma}}^2)$. 
Estimating $\mathbf{\Gamma}$ in practice is challenging and leads to regret bounds depending on it instead of the ``true'' covariance matrix $\SSigma$, potentially resulting in significantly looser bounds.
This issue was addressed by \cite{perrault2020covariance}, who proposed a covariance-adaptive algorithm, \ESCBC, with asymptotically optimal gap-dependent regret upper bounds. Yet, it suffers from an additive constant of order $\Delta_{\min}^{-2}$ , which prevents its conversion into an $\smash{\tilde O(\sqrt{T})}$\footnote{We denote $\tilde O$ for $O$ when $T\rightarrow \infty$, up to poly-logarithmic terms.} gap-free bound.  
Thus, none of the above works proposes a $\smash{\tilde O(\sqrt{T})}$ gap-free and covariance-adaptive regret bound, which is one of the key contributions of this paper.
A common drawback of these works is also their potentially prohibitive computational complexity, due to the need to solve a maximization step over a large action set $\cA \subseteq \{0,1\}^d$ that can be exponentially large. 
Some works, such as  \cite{cuvelier2021statistically} or \cite{liu2022batch}, propose solutions to achieve polynomial time complexity, for example by applying UCB at the item level only rather than the action level. However, these approaches only work for independent rewards or under specific assumptions on their distribution, making the analysis for generic and unknown distributions extremely challenging. Another approach to tackle the computational burden in combinatorial semi-bandits is to resort to sampling algorithms, which we detail below.

A second line of works for stochastic combinatorial semi-bandits considers randomized algorithms inspired by Thompson Sampling (\TS) for multi-armed bandits \citep{thompson1933likelihood}. These algorithms involve sampling a random vector $\tilde \mu_t\in\bR^d$ at each round $t+1\in[T-1]$ from a distribution representing a ``belief'' over the parameter $\mu$, taking a decision $A_{t+1}\in\arg\max_{a\in\cA} \langle a,\tilde \mu_t\rangle$, and updating the belief distribution using the observations. The main appeal of these approaches lies in their computational complexity, especially when solving a linear maximization problem in particular action spaces (such as matroids). 
Recent works have designed and analyzed such algorithms. Notably, \citet{wang2018thompson} consider independent item's rewards. \citet{perrault2020statistical} refine it and assume a known variance-proxy $\mathbf{\Gamma}$ and therefore suffers from the same drawbacks as \cite{degenne2016combinatorial}. Their technical analysis also yields a gap-dependent regret bound with an undesirable $\smash{\Delta_{\min}^{-m}}$ term, preventing a $\smash{\tilde O(\sqrt{T})}$ gap-free rate.
A central contribution in our paper is the combination of the computational efficiency for sampling algorithms with the covariance-adaptivity $\smash{\tilde O(\sqrt{T})}$ gap-free from our UCB approach. 

Besides, the literature concerning our setting has historically mostly focused on cases where the action set is exponentially large, namely \(P \gg d\), and the way to get quasi-optimal regret rates in these instances. However, outside of these regimes, the commonly derived regret bounds are too rough and fail to show the benefit of the semi-bandit feedback. While the conventional stochastic combinatorial semi-bandit regret upper bound grows as \(\smash{\tilde O(\sqrt{mdT})}\) \citep{kveton2015tight}, a \(\smash{\tilde O(\sqrt{mPT})}\) could be achieved using bandit feedback only \citep{auer2002nonstochastic}. Intriguingly, the latter appears to outperform the semi-bandit rate as soon as $P < d$, making the extra information obtained through a richer feedback seemingly useless. Fine-grained analyses, clearly taking the structure into account, are therefore needed.

\subsection{Contributions}

\paragraph{A new deterministic optimism-based algorithm (\cref{sec:OLSUCBC}).} 
We present \OLSUCBC (Online Least Squares Upper Confidence Bound with Covariance estimation), relying on the optimism principle. The analysis of \OLSUCBC sketched in \cref{sec:TechnicalOLSUCBC} shows the following properties:
    \begin{itemize}[nosep, leftmargin=1cm]
        \item \textit{First optimal gap-free regret upper bound}. \OLSUCBC yields a similar gap-dependent regret bound as \ESCBC from \citet{perrault2020covariance} up to logarithmic factors, and \textit{the first optimal covariance-adaptative gap-free} $\smash{\tilde O(\sqrt{T})}$ regret bound (\cref{thm:RegretOLSUCBC}). 
        \item \textit{Improved performance over UCB in all regimes of $P/d$}. Under some conditions on the covariance matrix $\SSigma$, we prove that \OLSUCB has a uniformly better regret than \UCB, showing that properly leveraging semi-bandit feedback indeed consistently offers an advantage on (simple) bandit algorithms, which is not straightforward from existing analyses.
        \item \textit{Improved complexity over concurrent algorithms}. \OLSUCBC circumvents the convex optimization problem that \ESCBC requires to solve at each round and is therefore more efficient, despite suffering in the very large $P$ regime as many other deterministic algorithms.
    \end{itemize}
    
\paragraph{The first stochastic optimism-based algorithm (\cref{sec:COSV}).} 
We introduce \COSV (Combinatorial Optimistic Sampling with Variance estimation), a \TS-inspired algorithm exploiting the ``frequentist'' confidence regions derived in \cref{sec:Estimators}. It satisfies the following:
    \begin{itemize}[nosep, leftmargin=1cm]
        \item \textit{Improved complexity for $P \gg 1$ compared to other deterministic semi-bandit algorithms}. \COSV can be efficient in the very large $P$ regime, which is the main blind spot of \OLSUCBC.
        \item \textit{First gap-free $\smash{\tilde O(\sqrt{T})}$ regret upper bound for a sampling algorithm}. The analysis we provide in \cref{sec:TechnicalCOSV} exploits the common structure of the \OLSUCBC and \COSV algorithms. It enables the derivation of a gap-dependent bound for \COSV that does not involve the $\Delta^{-m}$ term we typically find in the analysis for other \TS algorithms \citep{wang2018thompson, perrault2020statistical}, consequently leading to a new $\smash{\tilde O(\sqrt{T})}$ variance-adaptive gap-free regret upper bound for a sampling algorithm.
    \end{itemize}
    
\paragraph{A novel gap-free lower bound (\cref{sec:LowerBound}).} We show a gap-free lower bound on the regret for stochastic combinatorial semi-bandits, explicitly involving the structure of the problem (the items forming each action) and the covariance matrix $\SSigma$. This lower bound highlights the optimality of the gap-free upper bound we establish for \OLSUCBC.

Technical details are deferred to Section~\ref{sec:Estimators}, Section~\ref{sec:Regret}, and the Appendix.

\begin{table*}[ht]
    \caption{\small Asymptotic $\tilde{O}(\cdot)$ regret bounds 
    and per-round time complexities
    up to poly-logarithmic terms in $d$, for the following deterministic algorithms: 
    \UCB \citep{auer_finite_2002},
    \UCBV \citep{audibert2009exploration},
    \CUCB \citep{kveton2015tight},
    \OLSUCBC \citep{degenne2016combinatorial},
    \ESCBC \citep{perrault2020covariance},
    and \OLSUCBC (\textcolor{myblue}{ours}); as well as the two stochastic algorithms:
    \CTSG \citep{perrault2020statistical} and
    \COSV (\textcolor{myblue}{ours}). \\
    Notations: \(a\) refers to actions; \(i\) and \(j\) refer to items; \(m\) denotes the maximum number of items per action; $B$ is a vector of bounds on the items' rewards; $\mathbf{\Gamma}$ is a covariance-proxy matrix; $\gamma$ is the maximum of ``correlations-proxy''; we abbreviate $\max\{x,0\}$ to $(x)_+$ for any $x \in \mathbb{R}$ ; $C_{1/T}^{\text{opt}}$ refers to the complexity of the optimisation step needed in \ESCBC. 
    }
    \label{tab:summary}
    \centering
    \begin{tabular}{ccccc}
    \toprule
    \makecell{Fdbck.} & \makecell{Algorithm} & \makecell{Info.} & \makecell{Time Complexity} & \makecell{Gap-Free Asymptotic Regret}\\
    \midrule
    \multirow{2}{*}{Bndt.} & \UCB & $B$ & \(P\) & \(\big(T\sum_{a}(a^\top B)^2\big)^{1/2}\)  \\
    & \UCBV & $\varnothing$ & \(P\) & \(\big(T\sum_{a}a^\top\SSigma a\big)^{1/2}\)  \\
    \midrule
    \multirow{5}{*}{S-Bndt.}& \CUCB & $B$ & \(mP\) & \(\big(Tmd\big)^{1/2}\|B\|_{\infty}\)  \\
    & \OLSUCB\footnotemark[1] & $\mathbf{\Gamma}$ & \(m^2 + Pd^2\) & $\varnothing$ \\
    & \ESCBC & $\varnothing$ & \(m^2 + P \ C_{1/T}^{\text{opt}}\) & \(\varnothing\) \\
    & \textcolor{myblue}{\textbf{\OLSUCBC}} & $\varnothing$ & \textcolor{myblue}{\(m^2 + Pd^2\)} & \textcolor{myblue}{\({\big(T\sum_i \max_{a / i\in a}\sum_{j\in a}(\SSigma_{i,j})_+}\big)^{1/2}\)} \\
    \midrule
    \multirow{2}{*}{S-Bndt.}& \CTSG\footnotemark[2] & $\mathbf{\Gamma}$ & \(\text{poly}(d)\) & \(\varnothing\) \\
    & \textcolor{myblue}{\textbf{\COSV}}\footnotemark[2] & $\varnothing$ & \textcolor{myblue}{ \( \text{poly}(d)\)} & \textcolor{myblue}{$\big(Tm\sum_{i\in[d]}\SSigma_{i,i}\big)^{1/2}$} \\
    \bottomrule    
    \end{tabular}
\end{table*}

\footnotetext[1]{Note that \OLSUCB incur a $\Delta^{-2}$ term in its gap-dependent bound. This was overlooked by the authors but yields a $T^{2/3}$ gap-free bound.}
\footnotetext[2]{Assuming $\cA$ has matroid structure, the computational complexity is improved compared to a $O(P)$ for large $P$.}

\section{Covariance-adaptative deterministic algorithm: \OLSUCBC}\label{sec:OLSUCBC}

In this section, we design a new algorithm that efficiently leverages the semi-bandit feedback. It approximates the coefficients of the covariance matrix $\SSigma$ online. The approximation is symmetric by construction and yields a coefficient-wise upper bound of $\SSigma$, but it is not necessarily positive semi-definite, a constraint that can be challenging to impose in practice.

\subsection{Algorithm: \OLSUCBC}\label{ssec:OLSUCBC}

We present \OLSUCBC described in Alg.~\ref{alg:OLSUCBC} and detail below the successive steps it performs.

\paragraph{Initial exploration.} The algorithm first explores by choosing every base item $i\in[d]$ and every ``reachable'' couple $(i,j) \in [d]^2$ at least once. 

\begin{wrapfigure}{r}{0.48\linewidth}
\vspace{-0.3cm}
\centering
\begin{minipage}{\linewidth}
\begin{algorithm}[H]
   \caption{\OLSUCBC}
   \label{alg:OLSUCBC}
    \begin{algorithmic}
    \STATE \textbf{Input} $\delta > 0$, $B \in \mathbb{R}_+^d$.
    \FOR{\(t=1, \dots, T\)}
        \IF{$\big\{a \in \cA \vert \min_{(i,j)\in a} n_{t,(i,j)} < 1 \big\} \neq \emptyset$}
        \STATE Choose any $A_t$ in the above set.
        \ELSE
        \STATE Compute $\hat \mu_{t-1}$ from \eqref{eq:EmpiricalMean}.
        \STATE Compute $\hat{\SSigma}_{t-1}$ from \eqref{eq:CovEstimation} and \eqref{eq:SSigmaHat}.
        \STATE Compute $\hat \Rdesignmat_{t-1}$ from \eqref{eq:RDesignMat}.
        \STATE Choose \(A_t \in \cA\) from \eqref{eq:OLSUCBC}.
        \STATE Environment samples \(Y_t\in\bR^d\).
        \STATE Receive reward \(\langle A_t, Y_t\rangle = \sum_{i}A_{t,i}Y_{t,i}\).
        \ENDIF
    \ENDFOR
\end{algorithmic}
\end{algorithm}
\end{minipage}
\vspace{-0cm}
\end{wrapfigure}

\paragraph{Rewards means estimation.} At each round $t+1\in[T-1]$, the algorithm uses an empirical mean $\hat{\mu}_{t}$ for $\mu$ defined as
\begin{equation}
\label{eq:EmpiricalMean}
    \textstyle \hat{\mu}_t 
    = \diagcounts_t^{-1}\sum_{s=1}^t \diagmat_{A_s} Y_s
    ,
\end{equation}
where \(\diagmat_a = \diag(a)\in M_d(\bR)\) is the diagonal matrix of the elements in \(a\in\cA\); $n_{t,(i,j)}$ is the number of times items $i$ and $j$ (with possibly $i=j$) have been chosen together; \(\diagcounts_t = \diag((n_{t,(i,i)})_{i\in[d]})\in M_d(\bR)\) is the diagonal matrix of item counts.

\paragraph{Rewards covariances estimation.}

The covariances are estimated by
\begin{equation}\label{eq:CovEstimation}
    \hat{\chi}_{t,(i,j)} = \textstyle \hat{\Sq}_{t,(i,j)} - \hat\mu_{t,i}\hat\mu_{t,j}\,,
\end{equation}
where $\hat{\Sq}_{t,(i,j)}=\frac{1}{n_{t,(i,j)}}\sum_{s=1}^t A_{s,i}A_{s,j}Y_{s,i}Y_{s,j}$. The algorithm uses $\hat\SSigma_t$, a coefficient-wise upper-confidence bound of $\SSigma$ whose coefficients are defined for a fixed $\delta > 0$ as
\begin{align}
    \label{eq:SSigmaHat}
    \hat{\SSigma}_{t,(i,j)} &= \hat{\chi}_{t,(i,j)} + \frac{B_iB_j}{4}\bigg(\frac{5h_{t,\delta}}{\sqrt{n_{t,(i,j)}}}+\frac{h_{t,\delta}^2}{n_{t,(i,j)}}+\frac{1}{n_{t,(i,j)}^2}\bigg)\,,
\end{align}
where $\smash h_{t,\delta} = \Big(1+2\log(1/\delta)+2\log\big(t\log(t)^2d(d+1)\big)+\log(1+t)\Big)^{1/2}$.

\paragraph{Optimistic action choice.} Following the `optimistic' principle of \UCB-like algorithms, the estimated rewards $(\langle \hat{\mu}_t, a\rangle)_{a\in\cA}$ are inflated by bonuses, yielding corresponding upper confidence bounds. The bonuses involve the history through a \emph{regularized empirical design matrix} (with empirical covariances):
\begin{equation}
    \label{eq:RDesignMat}
    \hat \Rdesignmat_t = \sum_{s=1}^t\diagmat_{A_s}\hat{\SSigma}_t\diagmat_{A_s} + \diagmat_{\hat{\SSigma}_t} \diagcounts_t+\|B\|^2\diagI\,,
\end{equation}
where \(\smash{\diagmat_{\hat{\SSigma}_t} = \diag(\hat{\SSigma}_t)}\in M_d(\bR)\), \(\diagI\) is the identity matrix and $\hat{\SSigma}_t$ is the coefficient-wise upper bound for the covariance matrix defined in \eqref{eq:SSigmaHat}. Formally, \OLSUCBC chooses
\begin{equation} \label{eq:OLSUCBC}
   A_{t+1} \in \arg\max_{a\in \cA} \Big\{\langle a, \hat \mu_{t} \rangle +  f_{t, \delta}\ \big\|\diagcounts_t^{-1}a\big\|_{\hat \Rdesignmat_t}\Big\}\,,
\end{equation}
where $\smash  f_{t,\delta} = 6\log(1/\delta) + 6\Big(\log(t)+ (d+2)\log(\log(t))\Big) + 3d\Big(2\log(2) +\log(1+e)\Big)$.

\paragraph{Efficiency improvement.}
While \citet{perrault2020covariance} use an axis-realignment technique to derive their confidence regions, our approach builds ellipsoidal confidence regions. This simplifies the computation of an upper confidence bound for each action as we have a closed-form expression. In comparison, \citet{perrault2020covariance} need to solve linear programs in convex sets at each iteration.

\subsection{Regret upper bounds}

\begin{restatable}{theorem}{RegretOLSUCBC}
    \label{thm:RegretOLSUCBC}
    Let \(T\in\bN^*\) and $\delta>0$.
    
    Then, {\normalfont \OLSUCBC} (Alg.~\ref{alg:OLSUCBC}) satisfies the \emph{gap-dependent} regret upper bound
    \[ 
        \bE[R_T] 
        = \tilde{O}\bigg(\log(m)^2\sum_{i=1}^d\max_{a\in\cA / i\in a, \Delta_a>0} \frac{ \sigma_{a,i}^2}{\Delta_a} \bigg)\, ,
    \]
    where $\sigma_{a,i}^2 = \sum_{j \in a} \max\{\SSigma_{i,j},0\}$, and the \emph{gap-free} regret upper bound
    \[
        \bE[R_T] = \tilde{O}\bigg(\log(m)\sqrt{T}\sqrt{\textstyle\sum_{i=1}^d\max\limits_{a\in\cA / i\in a}  \sigma_{a,i}^2  } \bigg) \,.
    \]
\end{restatable}

The proof is outlined in \cref{sec:Regret} and the specific details are presented in \cref{app:OLSUCBC}.

\paragraph{Optimal gap-free bound.} This result shows that \OLSUCBC yields the same gap-dependent regret upper bound as \ESCBC \citep{perrault2020covariance} (up to poly-logarithmic factors) and more importantly yields a novel covariance-adaptive and optimal $\smash{\tilde O(\sqrt{T})}$ gap-free bound, 
as shown by the following lower-bound proven in \cref{app:LowerBound}. Unfortunately, only the positive coefficients of \(\SSigma\) are considered in our bound but the inclusion of negative correlations could be advantageous to reduce the rate at which the regret increases. However, it could complicate the analysis greatly and is thus deferred to future research.
\label{sec:LowerBound}

\begin{restatable}{theorem}{LowerBound}
    \label{thm:LowerBound}
    Let $d, m\in \bN^*$ such that $d/m\geq2$ is an integer, $T\in\bN^*$, and $\SSigma\succeq0$ a covariance matrix. Then, there exists a stochastic combinatorial semi-bandit with $d$ base items, and a reward distribution with covariance matrix $\SSigma$ on which for any policy $\pi$, the pseudo regret satisfies
    \begin{equation*}
        \bE[R_T] \geq \frac{1}{8}\bigg(T\sum_{i\in[d]}\max_{a\in\cA, i\in a}\sum_{j\in a}\SSigma_{i,j}\bigg)^{1/2}.
    \end{equation*}
\end{restatable}

\paragraph{Improvement over \CUCB.} \label{par:DiagCov} Our gap-free and gap-dependent bounds outperform those of \CUCB \citep{kveton2015tight} no matter the covariance structure, as $(Tmd)^{1/2}\|B\|_{\infty} \gtrsim (\mathrm{Tr}(\SSigma)T)^{1/2}$.\footnote{We denote $\gtrsim$ for $\geq$ up to a constant factor.} Besides, in the particular case of a diagonal $\SSigma$, our gap-free upper bound gains a factor at least $\sqrt{m}$ over the one of $\CUCB$. In this scenario, $\sigma_{a,i}^2 = \SSigma_{i,i}$ for all 
$a\in\cA$ and $i \in a$. Our gap-dependent and gap-free upper bounds are then roughly bounded as
\[
    \sum_{i=1}^d \frac{\SSigma_{i,i}}{\min\limits_{a\in\cA/i\in a} \Delta_a}
\ \text{and} \ 
    \sqrt{\mathrm{Tr}(\SSigma)T}\,, 
\]
respectively.

\paragraph{Improvement over \UCBV.} Assuming that $\SSigma_{i,j} \geq 0$ for all $i,j$, our upper-bound uniformly improves the one of \UCBV of order \(\smash{\big(T\sum_{a}a^\top\SSigma a\big)^{1/2}}\), since in this case
$\smash{\sum_{i=1}^d \max_{a \in \cA\backslash i\in \cA} \sigma^2_{a,i} \leq \sum_{a}\Vert a\Vert_{\mathbf{\SSigma}}}$. Existing semi-bandit analyses could only leverage semi-bandit feedback in the regime $P\gg d$, which is natural in combinatorial bandits but not systematic in real-world applications.

\section{New sampling algorithm for combinatorial semi-bandits: \COSV}\label{sec:COSV}

\begin{wrapfigure}{r}{0.48\linewidth}
    \vspace{-0.7cm}
    \centering
    \begin{minipage}{\linewidth}
    \begin{algorithm}[H]
    \caption{\COSV}
    \label{alg:COSV}
    \begin{algorithmic}
    \STATE \textbf{Input} $\delta > 0$, $B \in \mathbb{R}_+^d$.
    \FOR{\(t\in[T]\)}
        \IF{$\big\{a \in \cA \text{ s.t } \min_{(i,j)\in a} n_{t,(i,j)} < 1 \big\} \neq \emptyset$}
        \STATE Choose $A_t$ in the above set.
        \ELSE
        \STATE Compute $\hat \mu_{t-1}$ \eqref{eq:EmpiricalMean}.
        \STATE Compute $\smash(\hat{\SSigma}_{t-1,(i,i)})_{i\in[d]}$ \eqref{eq:SSigmaHat}.
        \STATE Compute $\smash(\hat \Rdesignmat_{t-1,(i,i)})_{i\in[d]}$ from \eqref{eq:RDesignMat}.
        \STATE Sample $\tilde\mu_{t-1}$ from \eqref{eq:SampleMu}
        \STATE Choose \(A_t \in\arg\max_{a\in\cA}\langle a, \Tilde{\mu}_{t-1}\rangle\).
        \STATE Environment samples \(Y_t\in\bR^d\).
        \STATE Receive reward \(\langle A_t, Y_t\rangle = \sum_{i}A_{t,i}Y_{t,i}\).
        \ENDIF
    \ENDFOR
\end{algorithmic}
\end{algorithm}
    \end{minipage}
    \vspace{-1.2cm}
\end{wrapfigure}

In this section, we introduce a randomized algorithm inspired from \TS, enabling to get potentially computational complexity at the cost of not leveraging off-diagonal covariances.

The difficulty in designing and analysing \TS algorithms generally stems from controlling the random exploration. To that end, we parametrize the exploration distribution using the same estimators as \OLSUCBC.

\subsection{Algorithm: \COSV}

We propose a sampling strategy using ``frequentist'' estimators, \COSV, described in \cref{alg:COSV}. 

The algorithm begins with the same exploration phase as \OLSUCBC. Thereafter at each round $t+1\in[T-1]$, we sample parameters $(\tilde\mu_{i,t})_{i\in[d]}$ using $1$-dimensional normal distributions biased toward the positive orthant. Formally, for all $i\in[d]$, we sample
\begin{equation}
    \label{eq:SampleMu}
    \textstyle \tilde\mu_{t,i} \sim \cN \bigg(\hat{\mu}_{t,i}+(1+g_{t,\delta})f_{t,\delta}\frac{\hat{\Rdesignmat}_{t,(i,i)}^{1/2}}{n_{t,(i,i)}},\,\,
    f_{t,\delta}^2\frac{\hat{\Rdesignmat}_{t,(i,i)}}{n_{t,(i,i)}^2}\bigg),
\end{equation}
where $g_{t,\delta}=\Big(1+2\log\big(2dt\log(t)^2/\delta\big)\Big)^{1/2}$ and $f_{t, \delta}$ is the same as for \OLSUCBC.

\subsection{Regret upper bound}

\begin{restatable}{theorem}{RegretCOSV}
    \label{thm:RegretCOSV}
    Let \(T\in\bN^*\), and $\delta>0$. 
    
    Then, {\normalfont \COSV} (Alg.~\ref{alg:COSV}) satisfies the \emph{gap-dependent} regret upper bound
    \begin{equation}
    \label{eq:COSVGapDep}
        \bE[R_T] 
        = \tilde{O}\bigg(\log(m)^2\sum_{i=1}^d\frac{m\SSigma_{i,i}}{\Delta_{i,\min}} \bigg)\,,
    \end{equation}
    where $\Delta_{i,\min}=\min\{\Delta_a,\ a\in\cA \text{ such that } i\in a \}$, and the \emph{gap-free} regret upper bound
    \begin{equation}
        \bE[R_T] = \tilde{O}\bigg(\log(m)\sqrt{T}\sqrt{\textstyle m \sum_{i=1}^d\SSigma_{i,i}} \bigg) \,.
    \end{equation}
\end{restatable}

The proof is outlined in \cref{sec:Regret} and the specific details can be found in \cref{app:COSV}.

\paragraph{Novel variance-dependent bound.}
\cref{thm:RegretCOSV} presents the first variance-dependent bound for a sampling-based semi-bandit algorithm. Unfortunately, integrating the covariances $\SSigma_{i,j}$ in the leading term is still an open problem. Possible leads include exploring other biasing strategies for sampling, or using oversampling approaches like \citet{abeille2017linear} which inflate the confidence regions in the linear bandits setting. 

\paragraph{Novel gap-free regret bound.} An important novelty of our gap-dependent bound \cref{eq:COSCGapDep} is the absence of $\Delta_{\min}^{-m}$ terms present in the previous analyses of \CTS \citep{wang2018thompson, perrault2020statistical}. In particular, this improvement yields the first $\tilde O(\sqrt{T})$ gap-free regret upper bound for a sampling strategy.

\section{Mean and covariance estimation}
\label{sec:Estimators}

In this section, we present concentration results for $\hat{\mu}_t$ (rewards means) and $\hat{\SSigma}_t$ (rewards covariances, estimated with $\hat{\chi}_t$) used in \OLSUCBC and \COSV, which are central to prove \cref{thm:RegretOLSUCBC} and \cref{thm:RegretCOSV} (sketched in \cref{sec:Regret}).

\subsection{Covariance-aware confidence region for the average reward}
\label{sec:ConfidenceMean}

\paragraph{Average reward estimation.} Let \(a\in\cA\), \(t\geq d(d+1)/2\), as introduced in \cref{ssec:OLSUCBC}, the least square estimator for the mean reward vector \(\mu\) using all the data gathered after round $t$ is
\begin{equation*}
    \textstyle \hat{\mu}_t 
    = \diagcounts_t^{-1}\sum_{s=1}^t \diagmat_{A_s} Y_s
    = \mu + \diagcounts_t^{-1}\sum_{s=1}^t \diagmat_{A_s} \eta_s\,,
\end{equation*}
where the $\eta_s$ denote the deviations $Y_s-\mu$.

\paragraph{Confidence region design.} We design confidence regions inspired from \LinUCB literature \citep{rusmevichientong2010linearly,filippi2010parametric, abbasi2011improved} and the work of \citet{degenne2016combinatorial}. Major differences with those works include using Bernstein's style concentration inequalities involving the covariance matrix $\SSigma$, assuming a multidimensional noise term, and combining them with a covering argument to relax dependence in $d$  \citep[peeling trick from][]{degenne2016combinatorial}.
We introduce the \textit{regularized design matrix} defined by
\begin{equation*}
    \textstyle \Rdesignmat_t = \designmat_t + \diagcounts_t\diagmat_{\SSigma}+\|B\|^2\diagI\,,
\end{equation*}
where $\designmat_t = \sum_{s=1}^t\diagmat_{A_s}\SSigma\diagmat_{A_s}$ is the design matrix (of which the \OLSUCBC  and \COSV use an empirical version). 
Let $S_t = \diagcounts_t (\hat{\mu}_t  - \mu)$, the deviations of $\langle a, \hat\mu_t\rangle$ are bounded as
\begin{equation}
\label{eq:deviations}
    \textstyle|\langle a, \hat{\mu}_t-\mu\rangle|
    \leq \|\diagcounts_t^{-1}a\|_{\Rdesignmat_t} \ \|S_t\|_{\Rdesignmat_t^{-1}}\,.
\end{equation}
Designing a confidence region for $\|S_t\|_{\Rdesignmat_t^{-1}}$ therefore allows to control the deviations $\textstyle|\langle a, \hat{\mu}_t-\mu\rangle|$ uniformly on $\cA$. Let $\delta>0$, we define the event
    \begin{equation} \label{eq:Gt}
    \cG_t = \big\{\ \|S_t\|_{\Rdesignmat_t^{-1}} \leq f_{t,\delta} \big\}\,,
    \end{equation}
with $\smash{f_{t,\delta} = 6\log(1/\delta) + 6[\log(t)+ (d+2)\log(\log(t))] + 3d[2\log(2) +\log(1+e)]}$.

This event can also be written $\cG_t = \big\{\ \|\hat{\mu}_t - \mu\|_{\diagcounts_t \Rdesignmat_t^{-1} \diagcounts_t} \leq f_{t,\delta} \big\}$ and is therefore equivalent to $\hat{\mu}_t$ belonging to an ellipsoid around the true reward mean vector $\mu$. 

\paragraph{Confidence region probability.} The following result proven in \cref{app:EventGt} presents an upper bound for $\bP(\cG_t^c)$.

\begin{restatable}{proposition}{EventGt}
    \label{prop:EventGt}
    Let $t\geq d(d+1)/2$ and $\delta>0$. Then, 
$\bP(\cG_t^c) 
        \leq \delta/(t\log(t)^2)\,.$
\end{restatable}

Proving this result relies on an argument adapted from \citet{faury2020improved} and a covering trick from \citet{degenne2016combinatorial}. 

\subsection{Confidence interval for covariances estimator}
\label{sec:ConfidenceCov}

\paragraph{Rewards covariances estimator.} Let $t\geq d(d+1)/2$ and a ``reachable'' couple $(i,j)\in [d]^2$ . 
The coefficients of $\SSigma$ can be estimated online by \(\hat{\chi}_t\) as introduced in \cref{ssec:OLSUCBC}
\begin{equation*}
    \hat{\chi}_{t,(i,j)} = \textstyle \hat{\Sq}_{t,(i,j)}-\hat{\mu}_{t,i}\hat\mu_{t,j}\,.
\end{equation*}

\paragraph{Rewards covariances upper confidence bound.} Let $\delta>0$. We use the following coefficient-wise upper estimates of $\SSigma$ in our algorithms 
\begin{align*}
    \hat{\SSigma}_{t,(i,j)} &= \hat{\chi}_{t,(i,j)} + \frac{B_iB_j}{4}\bigg(\frac{5h_{t,\delta}}{\sqrt{n_{t,(i,j)}}}+\frac{h_{t,\delta}^2}{n_{t,(i,j)}}+\frac{1}{n_{t,(i,j)}^2}\bigg)\,,
\end{align*}
with $\smash h_{t,\delta} = \Big(1+2\log(1/\delta)+2\log\big(t\log(t)^2d(d+1)\big)+\log(1+t)\Big)^{1/2}$.

\paragraph{Favorable event design.}
We define $\cC_t$ as the event where all the coefficients of $\hat{\SSigma}_{t}$ are indeed upper bounding those of $\SSigma$:
\begin{equation}\label{eq:covarianceUCB-favorable-event}
    \cC_t = \big\{\forall (i,j)\in[d]^2 \text{ ``reachable''}, \ \hat{\SSigma}_{t,(i,j)} \geq  \SSigma_{i,j} \big\}\,.
\end{equation}

\paragraph{Favorable event probability.}
The following result proven in Appendix~\ref{app:Covariance} presents an upper bound for $\bP(\cC^c_t)$.

\begin{restatable}{proposition}{FavorableEventCovariance}
\label{prop:ChiDeviations2}
Let $t\geq d(d+1)/2$ and $\delta>0$. Then, 
$\bP(\cC_t^c) 
        \leq \delta/(t\log(t)^2)\,.$
\end{restatable}

\section{Regret upper bounds}
\label{sec:Regret}

In this section, we provide a sketch of the proof for \cref{thm:RegretOLSUCBC} and \cref{thm:RegretCOSV}. For both \OLSUCBC and \COSV, the idea to bound the regret is to find a sequence of \textit{favorable events} $(\cE_t)_{t\geq d(d+1)/2}$ that are true with high probability, and under which the regret grows logarithmically with time.

\subsection{Template bound} 

Let $(\cE_t)_{t\in[T]}$ be a sequence of events, then for both \OLSUCBC and \COSV standard derivations yield
\begin{equation}
    \label{eq:Template}
    \textstyle\bE[R_T] \leq \Delta_{\max}\Big(d(d+1)/2+\sum_{t=d(d+1)/2}^{T-1}\bP(\cE_t^c)\Big) + \bE\Big[\sum_{t=d(d+1)/2}^{T-1} \indicator\{\cE_{t}\}\Delta_{A_{t+1}}\Big]\,.
\end{equation}

Assuming that the sequence of events $(\cE_t)_{t\geq d(d+1)/2}$ happens with high enough probability, it is sufficient to control what happens conditionally to it. In particular, \cref{prop:HPRegret} in \cref{app:HPEvent} states that if we can bound $\Delta_{A_{t+1}}^2$ with a linear combination of terms evolving as $n_{t,(i,j)}^{-k}$ for every couple $(i,j)\in A_{t+1}$ and different $k\geq 1$, then we can infer a worst-case behaviour, which yields \cref{thm:RegretOLSUCBC} and \cref{thm:RegretCOSV}.

In the following, we will refer to the term $\smash{\sum_{t=d(d+1)/2}^{T-1}\bP(\cE_t^c)}$ as the \textit{unfavorable event probability} and to the term $\smash{\bE\big[\sum_{t=d(d+1)/2}^T \indicator\{\cE_{t}\}\Delta_{A_{t+1}}\big]}$ as the \textit{high-probability regret}.

\subsection{Regret of \OLSUCBC}
\label{sec:TechnicalOLSUCBC}
 
For \OLSUCBC we consider the sequence of events $\cE_t = \{\cG_t \cap \cC_t\}$ for all $t\geq d(d+1)/2$, corresponding the confidence regions of $\smash{(\hat{\mu}_t)_{t\geq d(d+1)/2}}$ and of $\smash{(\hat{\SSigma}_{t,(i,j)})_{t\geq d(d+1)/2,\ (i,j)\in[d]^2}}$ defined in \cref{sec:Estimators}. Under these events, we can upper-bound the high-probability regret from \cref{eq:Template} with the following proposition (proven in \cref{app:OLSUCBCHPRegret}).

\begin{restatable}{proposition}{OLSUCBCHPRegret}
\label{prop:OLSUCBCHPRegret}
Let \(\delta>0\). Then, {\normalfont \OLSUCBC} yields
\begin{equation*}
\begin{aligned}
&\bE\bigg[\sum_{t=d(d+1)/2}^{T-1}\Delta_{A_{t+1}}\indicator\big\{\cG_{t}\cap \cC_t\big\}\bigg] = O \Bigg(  \log(T)^{2}\log(m)^2\sum_{i=1}^d \max_{a\in\cA/i\in a}\frac{\sigma_{a,i}^2}{\Delta_{a}} \Bigg)\,,
\end{aligned}
\end{equation*}
as $T\rightarrow\infty$, where $\sigma_{a,i}^2 = \sum_{j \in a} (\SSigma_{i,j})_+$.
\end{restatable}

\paragraph{Conclusion of the proof.} 

Injecting results from Proposition~\ref{prop:OLSUCBCHPRegret} (high-probability regret) as well as Proposition~\ref{prop:EventGt} and Proposition~\ref{prop:ChiDeviations2} (unfavorable event probability) into the template bound \eqref{eq:Template}, we get
\begin{equation}\label{eq:olsucbc-gapdependent-bound}
    \bE[R_T] = O \Bigg( \log(T)^2\log (m)^2\sum_{i\in[d]}\max_{a\in\cA/i\in a}\frac{\sigma_{a,i}^2}{\Delta_{a}}\Bigg)\,,
\end{equation}
for \OLSUCBC as $T\rightarrow \infty$. This provides the gap-dependent bound of Theorem~\ref{thm:RegretOLSUCBC}. The gap-free bound is detailed in Appendix~\ref{app:OLSUCBCGapFree}. It is enabled by the fact that our gap-dependent bound does not incur any term in $\Delta_{\mathrm{min}}^{-2}$, unlike \citet{perrault2020covariance,degenne2016combinatorial}. 

\subsection{Regret of \COSV}
\label{sec:TechnicalCOSV}

For the analysis of our stochastic algorithm \COSV, we need to consider events related to the sampling distributions in addition to the events $\cG'_t$ and $\cC_t$ introduced in the precedent section.
For this purpose, we denote the event $\cH_t$ defined as
\begin{equation}
    \textstyle \cH_t= \Bigg\{\forall i\in [d],\ \bigg|\bigg(\hat\mu_{t,i}+(1+g_{t,\delta})f_{t,\delta}\frac{(\hat{\Rdesignmat}_{t,i})^{1/2}}{n_{t,i}}\bigg)-\tilde\mu_{t,i}\bigg|\leq g_{t,\delta}f_t\frac{(\hat{\Rdesignmat}_{t,i})^{1/2}}{n_{t,i}} \Bigg\}\,.
\end{equation}

The high-level idea of the event $\cH_t$ is to ensure that the sampled rewards $\tilde \mu_{t,i}$ upper-bound the true mean $\mu_i$ while not being too far for all the items $i\in a^*$. 
Showing that the event $\cH_t$ indeed occurs with high-probability (\cref{lem:ProbaH} in \cref{app:COSV}) and setting the events $\cE_t = \{\cG_t\cap\cC_t\cap\cH_t\}$, we can upper-bound the high-probability regret in the following proposition (proof is in Appendix~\ref{app:COSVHPRegret}).

\begin{restatable}{proposition}{COSVHPRegret}
    \label{prop:COSVHPRegret}
   Let $\delta>0$. Then {\normalfont \COSV} yields
\begin{align*}
    \bE\bigg[\sum_{t=d(d+1)}^{T-1}\Delta_{A_{t+1}}\indicator\big\{\cG_{t}\cap \cC_t\cap\cH_t\big\}\bigg] = O \Bigg(  \log(T)^{3}\log (m)^2\Big(\sum_{i=1}^d \frac{m\SSigma_{i,i}}{\Delta_{i, \min}} \Big)\Bigg)\,.
\end{align*}
\end{restatable}

\paragraph{Conclusion of the proof.} 

Injecting results from Proposition~\ref{prop:COSVHPRegret} (high-probability regret) as well as \cref{lem:ProbaH}, Proposition~\ref{prop:EventGt} and Proposition~\ref{prop:ChiDeviations2} (unfavorable event probability) into the template bound \eqref{eq:Template} yields
\begin{equation}\label{eq:COSCGapDep}
    \bE[R_T] = O \Bigg( \log(T)^3 \log (m)^2\sum_{i\in[d]}\frac{m\SSigma_{i,i}}{\Delta_{i, \min}} \Bigg)\,,
\end{equation}
as $T\rightarrow \infty$. This provides the gap-dependent bound of Theorem~\ref{thm:RegretCOSV}. As it does not incur any term in $\Delta_{\mathrm{min}}^{-m}$ as in \citet{wang2018thompson, perrault2020statistical}, this result can be used to derive a $\smash{\tilde O(\sqrt{T})}$ gap-free bound for a sampling-based combinatorial semi-bandit algorithm.

\section{Concluding remarks}

We propose and analyze two algorithms for combinatorial semi-bandits. \OLSUCBC is a deterministic, covariance-adaptive algorithm. Compared to other existing approaches, our algorithm is typically less computationally demanding and yields the first $\tilde O(\sqrt{T})$ gap-free regret rate that explicitly depends on the covariance of the base item rewards and the structure. \COSV is a variance-adaptive, \TS-like algorithm. Its complexity is significantly lower under certain types of constraints, but its regret is suboptimal as it assumes worst-case correlations. However, leveraging the analysis of \OLSUCBC, it also yields the first $\tilde O(\sqrt{T})$ gap-free regret upper bound among sampling-based approaches.

\bibliographystyle{apalike}
\bibliography{references}


\newpage
\appendix

\section*{Appendix}

The Supplementary is organized as follows:
\begin{itemize}[label=$-$, nosep]
    \item \cref{app:LowerBound} proves the lower bound from \cref{thm:LowerBound},
    \item \cref{app:EventGt} outlines proofs concerning the concentration of the average estimator (Propositions~\ref{prop:EventGt})
    \item \cref{app:Covariance} presents those for the covariance estimator (Proposition~\ref{prop:ChiDeviations2}),
    \item \cref{app:HPEvent} establishes general propositions used to upper-bound the number of times each item is chosen,
    \item \cref{app:OLSUCBC} and~\cref{app:COSV} detail proofs for \OLSUCBC and \COSV,
    \item \cref{app:Expe} presents some experimental results.
\end{itemize}

\section{Proof of the lower bound (\texorpdfstring{\cref{thm:LowerBound})}{Theorem 2}}
\label{app:LowerBound}
\LowerBound*
\begin{proof}
    We follow the methodology of \citet{auer2002nonstochastic}, modifying it to account for the different variances among actions.
    
    Let $d, m\in\bN^*$ such that $d/m\geq 2$ is an integer, $T\in\bN^*$, and a covariance matrix $\SSigma\succeq 0$. We consider the structure where $\cA = \{a_1, \dots, a_{d/m}\}\subset\{0, 1\}^d$ contains $d/m$ disjoint actions each having $m$ base elements. We consider that for all $p \in [d/m]$, $\smash{(a_p)_{i\in[d]} = \big(\indicator\big\{(p-1)m < i \leq pm\big\}\big)_{i\in[d]}}$. Let $\pi$ be a policy. As all the actions are disjoints, we can reduce ourselves to a multi-armed bandit with $d/m$ actions, where for all $p\in[d/m]$ the variance of the $p$-th action is $\langle a_p,\SSigma_p\rangle$.

    Let $\SSigma'\in M_{d/m}(\bR)$ be the diagonal matrix where for all $p\in[d/m]$, $\SSigma'_{p,p} = a_p^\top\SSigma a_p$. Let $c>0$, and
    \begin{equation}
    \label{eq:Delta_LowerBound}
        \Delta = 2c\sqrt{\SSigma'_{\text{min}}\frac{\sum_{k=1}^{d/m}\SSigma'_{k,k}}{T}}\,,
    \end{equation}
    where $\SSigma'_{\text{min}} = \min_{p\in[d/m]}\SSigma'_{p,p}$.
    
    We denote $G_0\sim\cN(0, \SSigma')$ a $(d/m)$-dimensional centered Gaussian distribution with covariance matrix $\SSigma'$.
    Let $p\in[d/m]$, we consider the mean vector $\mu^{(p)}\in\bR^{d/m}$ having coordinate 0 everywhere and $\Delta$ at coordinate $p$, for all $i\in[d/m]$, $\mu^{(p)}_i = \Delta\indicator\{i=p\}$. We introduce the Gaussian reward distributions $G_p \sim \cN(\mu^{(p)}, \SSigma')$ and denote $T_p = \sum_{t=1}^T\indicator\{A_t = p\}$. Then, using policy $\pi$, and considering the reward distributions $G_p$ and $G_0$, the average number of times action $p$ has been chosen satisfies
    \begin{align}
        \label{eq:TV_LowerBound}
        \Big|\bE_{\pi, G_p}[T_p] - \bE_{\pi, G_0}[T_p] \Big|
        \leq T\ \TV\Big((\pi, G_0), (\pi, G_p)\Big) 
        \leq T\sqrt{\frac{1}{2}\KL\Big((\pi, G_0), (\pi, G_p)\Big)}\,,
    \end{align}
    where $\TV$ denotes the total variation distance, $\KL$ denotes the Kullback--Leibler divergence and the last inequality uses Pinsker's inequality. Then, using the divergence decomposition between multi-armed bandits \citep[Lemma 15.1 in][]{lattimore2020bandit},
    \begin{align*}
        \KL\big((\pi, G_0), (\pi, G_p)\big) 
        &= \sum_{k=1}^{d/m} \bE_{\pi, G_0}\big[T_k\big] \ \KL\big(\cN(0, \SSigma'), \cN(\mu^{(p)}, \SSigma')\big)\\
        & = \sum_{k=1}^{d/m} \bE_{\pi, G_0}\big[T_k\big] \ \frac{\big(\mu^{(p)}_{k}\big)^2}{2\SSigma'_{k,k}}\,.
    \end{align*}
    Reinjecting this expression into \cref{eq:TV_LowerBound}, we get
    {\allowdisplaybreaks
    \begin{align*}
        \bE_{\pi, G_p}[T_p] 
        &\leq \bE_{\pi, G_0}[T_p] + \frac{T}{2}\sqrt{\sum_{k=1}^{d/m}\frac{\big(\mu^{(p)}_k\big)^2}{\SSigma'_{k,k}}\bE_{\pi, G_0}[T_k]}\\
        &= \bE_{\pi, G_0}[T_p] + \frac{T}{2}\sqrt{\frac{1}{\SSigma_{p,p}'}\Delta^2\bE_{\pi, G_0}[T_p]} \\
        &= \bE_{\pi, G_0}[T_p] + c\sqrt{T\bE_{\pi, G_0}[T_p]\frac{\SSigma_{\text{min}}'}{\SSigma_{p,p}'}\sum_{k=1}^{d/m}\SSigma'_{k,k}} \qquad \leftarrow \text{reinjecting \cref{eq:Delta_LowerBound}}\\
        &\leq \bE_{\pi, G_0}[T_p] + c\sqrt{T\bE_{\pi, G_0}[T_p]\sum_{k=1}^{d/m}\SSigma'_{k,k}}\,.
    \end{align*}}
    Now, summing over the actions $p$,
    \begin{align}
    \label{eq:SumDistrib_LowerBound}
        \sum_{p=1}^{d/m} \bE_{\pi, G_p}[T_p] 
        &\leq \sum_{p=1}^{d/m}\bE_{\pi, G_0}[T_p] + c\sqrt{T\sum_{k=1}^{d/m}\SSigma'_{k,k}}\sum_{p=1}^{d/m}\sqrt{\bE_{\pi, G_0}[T_p]} \nonumber\\
        &\leq T + c\sqrt{T\sum_{k=1}^{d/m}\SSigma'_{k,k}}\sqrt{\frac{d}{m}}\sqrt{\sum_{p=1}^{d/m}\bE_{\pi, G_0}[T_p]} \qquad \leftarrow \text{Cauchy--Schwarz} \nonumber\\
        &\leq T + cT\sqrt{\frac{d}{m}\sum_{k=1}^{d/m}\SSigma'_{k,k}}\,.
    \end{align}

    We denote $R^{(p)}_T$ the average cumulative regret incurred with the reward distribution $G_p$, then
    \begin{align*}
        \sum_{p=1}^{d/m} R_T^{(p)} 
        &= \Delta\sum_{p=1}^{d/m}(T-\bE_{\pi, G_p}[T_p])\\
        &= 2c\sqrt{\SSigma'_{\text{min}}\frac{\sum_{k=1}^{d/m}\SSigma'_{k,k}}{T}}\bigg(\frac{d}{m}T - \sum_{p=1}^{d/m}\bE_{\pi, G_p}[T_p] \bigg) \qquad \leftarrow \text{reinjecting \cref{eq:Delta_LowerBound}}\\
        &\geq 2c\sqrt{\SSigma'_{\text{min}}\frac{\sum_{k=1}^{d/m}\SSigma'_{k,k}}{T}}\bigg(\frac{d}{m}T - T - cT\sqrt{\frac{d}{m}\sum_{k=1}^{d/m}\SSigma'_{k,k}}\bigg) \qquad \leftarrow \text{from \cref{eq:SumDistrib_LowerBound}}\\
        &= 2c\sqrt{\SSigma'_{\text{min}}}\ \frac{d}{m}\sqrt{T\sum_{k=1}^{d/m}\SSigma'_{k,k}}\ \bigg(1-\frac{m}{d}-c\sqrt{\frac{1}{d/m}\sum_{k=1}^{d/m}\SSigma'_{k,k}}\bigg)\\
        &\geq 2c\sqrt{\SSigma'_{\text{min}}} \frac{d}{m}\sqrt{T\sum_{k=1}^{d/m}\SSigma'_{k,k}}\ \bigg(1-\frac{m}{d}-c\sqrt{\SSigma'_{\text{min}}}\bigg)\,.
    \end{align*}
    Taking $c=\frac{1}{2}\frac{1}{\sqrt{\SSigma_{\text{min}}'}}(1-\frac{m}{d})$, 
    \begin{align*}
        \sum_{p=1}^{d/m} R_T^{(p)} 
        &\geq \frac{d}{m}\sqrt{T\sum_{k=1}^{d/m}\SSigma'_{k,k}}\ \frac{1}{2}\bigg(1-\frac{m}{d}\bigg)^2 \qquad\\
        &\geq \frac{1}{8}\frac{d}{m}\sqrt{T\sum_{k=1}^{d/m}\SSigma'_{k,k}}  \leftarrow \text{as } m/d\leq 1/2\,.
    \end{align*}

    Therefore, there exists at least one instance $p^*\in[d/m]$ such that 
    \begin{equation*}
        R_T^{(p^*)}\geq \frac{1}{8}\sqrt{T\sum_{k=1}^{d/m}\SSigma'_{k,k}}\,.
    \end{equation*}

    Now, decomposing
    \begin{equation*}
        \sum_{k=1}^{d/m}\SSigma'_{k,k} = \sum_{k=1}^{d/m}\Big(\sum_{i\in a_k}\sum_{j\in a_k}\SSigma_{i,j}\Big) = \sum_{i\in[d]} \max_{a\in\cA, i\in a}\sum_{j\in a}\SSigma_{i,j}\,,
    \end{equation*}
    we get 
    \begin{equation*}
        R_T^{(p^*)}\geq \frac{1}{8}\sqrt{T\sum_{i\in[d]} \max_{a\in\cA, i\in a}\sum_{j\in a}\SSigma_{i,j}}\,.
    \end{equation*}
\end{proof}

\section{Concentration of the average rewards estimations (Proposition~\ref{prop:EventGt})}
\label{app:EventGt}
\EventGt*

\begin{proof}

Let $t\geq d(d+1)/2$ and $\delta>0$. 

We have 
\begin{align*}
    f_{t,\delta} 
    &= 6\log(1/\delta) + 6\Big(\log(t)+ (d+2)\log(\log(t))\Big) + 3d\Big(2\log(2) +\log(1+e)\Big)\\
    & = 6\log\Bigg(\frac{t\log(t)^2}{\delta}\bigg(\frac{\log(t)}{\log(1+(e-1))}\bigg)^d + \bigg(6d\log(2)+3d\log(2+(e-1))\bigg)\Bigg)\,.
\end{align*}

\paragraph{Covering argument (\emph{Peeling trick}).}
The peeling trick consists in separating the space of trajectories up to round \(t\) into an exponentially large number of parts, each having an exponentially small probability. 

Formally, let \(0<\epsilon<1\). For each \(p\in\bN^{d}\) we associate the set 
\begin{align}
\cD_p = \big\{x \in &\bR^d \text{ s.t. } \forall i \in [d],\ (1+\epsilon)^{p_i}\leq x_i < (1+\epsilon)^{p_i+1}\big\}\,.
\end{align}
As an abuse of notation, we denote by \((t\in\cD_p)\) the event \(\Big(\big(n_{t,(i,i)}+1\big)_{i\in[d]}\in\cD_p\Big)\). 

Setting $P_{t,\epsilon} = \big\lfloor\frac{\log(t)}{\log(1+\epsilon)}\big\rfloor$, we define for each $p\in [P_{t,\epsilon}]^d$

\begin{align}
    \tilde\diagcounts_p &= \diag\Big(\big((1+\epsilon)^{p_i}\big)_{i\in[d]}\Big)\in M_d(\bR)\,, \notag\\
    \Rdesignmat_{t,p} &= \designmat_t + \tilde\diagcounts_p\diagmat_{\SSigma}  + \|B\|^2\diagI\,.
\end{align}
In particular, under the event $(t\in\cD_p)$, $\diagcounts_t \preceq (1+\epsilon) \tilde\diagcounts_p$.

Using this covering, we decompose
\begin{align*}
    \bP(\cG_t^c) 
    &= \bP\Bigg(\bigg\|\sum_{s=1}^t \diagmat_{A_s} \eta_s\bigg\|_{\Rdesignmat_t^{-1}} > f_{t,\delta}\Bigg)\\
    &= \sum_{p\in [P_{t,\epsilon}]^d} \bP\Bigg(\Big(\|S_t\|_{\Rdesignmat_t^{-1}} > f_{t,\delta}\Big) \cap (t\in\cD_p)\Bigg)\\
    &\leq \sum_{p\in [P_{t,\epsilon}]^d} \bP\Bigg(\Big(\|S_t\|_{\Rdesignmat_{t,p}^{-1}} > f_{t,\delta}\Big) \cap (t\in\cD_p)\Bigg)\,.
\end{align*}

We now apply the following Lemmas.

\begin{restatable}{lemma}{MeanDeviations}
    \label{lem:MeanDeviations}
    Let $t\geq d(d+1)/2$, $0<\epsilon<1$, $p\in[P_{t,\epsilon}]^d$, and $\delta>0$. Then,
    \begin{equation}
        \bP\Bigg(\|S_t\|_{\Rdesignmat_{t,p}^{-1}}>6\log\Big(\frac{\Norm_p}{\Norm_{t,p}}\Big)+6\log(1/\delta)\Bigg)\leq \delta\,,
    \end{equation}
    where
    \begin{align*}
    \Norm_p &= \int_{\lambda\in\bR^d} \indicator\Big\{\|\Rdesignmat_{0, p}^{1/2}\lambda\|\leq \frac{1}{2}\Big\} \exp\{-\|\lambda\|_{\Rdesignmat_{0, p}}^2\}d\lambda\,,\\
    \Norm_{t,p} &= \int_{\lambda\in\bR^d} \indicator\Big\{\|\Rdesignmat_{t, p}^{1/2}\lambda\|\leq \frac{1}{2}\Big\} \exp\{-\|\lambda\|_{\Rdesignmat_{t, p}}^2\}d\lambda\,.
    \end{align*}
\end{restatable}

\begin{restatable}{lemma}{Dimension}
\label{lem:Dimension}
Let $t\geq d(d+1)/2$, $0<\epsilon<1$ and $p\in[P_{t,\epsilon}]^d$. Then,
    \begin{equation}
        \log\Big(\frac{\Norm_p}{\Norm_{p,t}}\Big) \leq d\log(2) + \frac{1}{2}\log\Bigg(\frac{\det(\Rdesignmat_{t,p})}{\det(\Rdesignmat_{0, p})}\Bigg)\,.
    \end{equation}
Moreover, under event $(t\in\cD_p)$,
\begin{equation}
    \log\Big(\frac{\Norm_p}{\Norm_{p,t}}\Big) \leq d\log(2) + \frac{1}{2}d\log(2+\epsilon)\,.
\end{equation}
\end{restatable}

In our case, setting $\epsilon=e-1$, they yield
\begin{align*}
    \bP(\cG_t^c) 
    &\leq \sum_{p\in [P_{t,\epsilon}]^d} \bP\Bigg(\bigg\{\|S_t\|_{\Rdesignmat_{t,p}^{-1}} > 6\log\bigg(\frac{t\log(t)^2}{\delta}\log(t)^d\bigg) + \bigg(6d\log(2)+3d\log(1+e)\bigg)\bigg\} \\
    &\hspace{3cm}\cap (t\in\cD_p)\Bigg)\\
    &\leq \sum_{p\in [P_{t,\epsilon}]^d} \bP\Bigg(\bigg\{\|S_t\|_{\Rdesignmat_{t,p}^{-1}} > 6\log\bigg(\frac{t\log(t)^2}{\delta}\log(t)^d\bigg) + 6\log\bigg(\frac{\Norm_p}{\Norm_{t,p}}\bigg)\bigg\}\\
    &\hspace{3cm}\cap (t\in\cD_p)\Bigg) \leftarrow \text{ \cref{lem:Dimension}}\\
    &\leq \sum_{p\in [P_{t,\epsilon}]^d} \bP\Bigg(\|S_t\|_{\Rdesignmat_{t,p}^{-1}} > 6\log\bigg(\frac{t\log(t)^2}{\delta}\log(t)^d\bigg) + 6\log\bigg(\frac{\Norm_p}{\Norm_{t,p}}\bigg)\Bigg)\\
    &\leq \sum_{p\in [P_{t,\epsilon}]^d} \delta\frac{1}{t\log(t)^2}\frac{1}{\log(t)^d} \leftarrow \text{ \cref{lem:MeanDeviations}}\\
    &= \delta\frac{1}{t\log(t)^2}\frac{1}{\log(t)^d} \log(t)^d\\
    &= \frac{\delta}{t\log(t)^2}\,.
\end{align*}
\end{proof}

\subsection{Proof of Lemma~\ref{lem:MeanDeviations}}
\label{app:MeanDeviations}
\MeanDeviations*

\begin{proof}
    We adapt the proofs from \citet{faury2020improved}, which adapts \citet{abbasi2011improved} itself.
    Let $t\geq d(d+1)/2$, $0<\epsilon<1$, $p\in[P_{t,\epsilon}]^d$, and $\delta>0$.
    
    Let $\lambda\in\bR^d$ such that $\|\lambda\|\leq \frac{1}{2\|B\|}$ and $s\in[t]$. We denote $\cF'_{t-1} = \sigma(A_1, Y_1, \dots, A_{t-1}, Y_{t-1}, A_t)$.
    
    Then, $\|\lambda^\top\diagmat_{A_s}\eta_s\|\leq 1/2$ and
\begin{equation*}
    \bE\bigg[\exp\Big(\lambda^\top\diagmat_{A_s}\eta_s-\lambda^\top\diagmat_{A_s}\SSigma\diagmat_{A_s}\lambda\Big)\bigg|\cF'_{s-1}\bigg] \leq 1
\end{equation*}
which yields that $\Big(M_k(\lambda)\Big)_{k\in\bN^*}=\Big(\exp\big(\lambda^\top S_k-\|\lambda\|_{V_k}^2\big)\Big)_{k\in\bN^*}$ is a $\cF'_k$-supermartingale.

Let $p\in[P_{t,\epsilon}]^d$, we consider the density $g_{p}$ of a $d$-dimensional Gaussian with covariance matrix $\frac{1}{2}\big(\tilde\diagcounts_p\diagmat_{\SSigma}+\|B\|^2\diagI\big)^{-1} = \frac{1}{2} \Rdesignmat_{0, p}^{-1}$, truncated in the ellipsoid $\big\{x\in\bR^d,\ \|\Rdesignmat_{0,p}^{1/2}x\|\leq \frac{1}{2}\big\}$,
\begin{equation*}
    g_p(x) = \frac{\indicator\big\{x\in\bR^d,\ \|\Rdesignmat_{0,p}^{1/2}x\|\leq \frac{1}{2}\}}{\Norm_p}\exp\Big(-\|x\|_{\Rdesignmat_{0, p}}^2\Big)\,,
\end{equation*}
where $\Norm_p$ is the normalisation constant. 

We integrate $\Big(M_k(\lambda)\Big)_{k\in\bN^*}$ for $\lambda \sim g_p$, and define $(\bar M_{p,k})_{k\in\bN^*}$ as
\begin{align*}
    \bar M_{p,k} = \int_{\lambda \in\bR^d} M_{k}(\lambda) d\lambda = \int_{\lambda \in\bR^d} \frac{\indicator\{\|\Rdesignmat_{0,p}^{1/2}\lambda\|\leq \frac{1}{2}\}}{\Norm_p} \exp\Big(\lambda^\top S_k-\|\lambda\|^2_{\Rdesignmat_{k,p}}\Big)d\lambda\,,
\end{align*}
which is still a supermartingale.

Let $\lambda_{t,p}^* \in \arg\max_{\{\Rdesignmat_{0,p}^{1/2}\|\lambda\|\leq \frac{1}{4}\}}(\lambda^\top S_t-\|\lambda\|^2_{\Rdesignmat_{t,p}})$. Then
\begin{align*}
    \bar M_{p,k} &= \frac{\exp(\lambda_{t,p}^{*\top} S_k- \|\lambda_{t,p}^*\|^2_{\Rdesignmat_{k,p}})}{\Norm_p}\int_{\|\Rdesignmat_{0,p}^{1/2}\lambda\|\leq\frac{1}{2}}\exp\Bigg((\lambda-\lambda_{t,p}^*)^\top S_t-\|\lambda\|_{\Rdesignmat_{k,p}}^2+\|\lambda_{t,p}^*\|_{\Rdesignmat_{k,p}}^2\Bigg)d\lambda\\
    &=\frac{\exp(\lambda_{t,p}^{*\top} S_k- \|\lambda_{t,p}^*\|^2_{\Rdesignmat_{t,p}})}{\Norm_p}\int_{\|\Rdesignmat_{0,p}^{1/2}\lambda+\Rdesignmat_{0,p}^{1/2}\lambda_{t,p}^*\|\leq\frac{1}{2}}\exp\Bigg(\lambda^\top S_k -\|\lambda\|_{\Rdesignmat_{k,p}}^2 - 2\lambda^\top\Rdesignmat_{k,p}\lambda^*_{t,p}\Bigg)d\lambda\\
    &\geq\frac{\exp(\lambda_{t,p}^{*\top} S_k- \|\lambda_{t,p}^*\|^2_{\Rdesignmat_{t,p}})}{\Norm_p}\int_{\|\Rdesignmat_{0,p}^{1/2}\lambda\|\leq\frac{1}{4}}\exp\Bigg(\lambda^\top S_k -\|\lambda\|_{\Rdesignmat_{k,p}}^2 - 2\lambda^\top\Rdesignmat_{k,p}\lambda^*_{t,p}\Bigg)d\lambda\\
    &=\frac{\exp(\lambda_{t,p}^{*\top} S_k- \|\lambda_{t,p}^*\|^2_{\Rdesignmat_{k,p}})}{\Norm_p}\int_{\|\Rdesignmat_{0,p}^{1/2}\lambda\|\leq\frac{1}{4}}\exp\Bigg(\lambda^\top \Big(S_k -2\Rdesignmat_{k,p}\lambda_{t,p}^*\Big)-\|\lambda\|^2_{\Rdesignmat_{k,p}}\Bigg)d\lambda\\
    &=\frac{\exp(\lambda_{t,p}^{*\top} S_k- \|\lambda_{t,p}^*\|^2_{\Rdesignmat_{k,p}})}{\Norm_p}\Norm_{k,p}\\
    &\qquad \qquad \int_{\lambda\in\bR^d}\frac{\indicator\big\{\|\Rdesignmat_{0,p}^{1/2}\lambda\|\leq \frac{1}{4}\big\}}{\Norm_{k,p}}\exp\Bigg(\lambda^\top \Big(S_k -2\Rdesignmat_{k,p}\lambda_{t,p}^*\Big)-\|\lambda\|_{\Rdesignmat_{k,p}}^2\Bigg)d\lambda\,,
\end{align*}

where we can recognize $g_{k,p}$ the density of a $d$-dimensional Gaussian with covariance matrix $\frac{1}{2}\Rdesignmat_{k,p}^{-1}$, truncated in the ellipsoid $\big\{x\in\bR^d,\ \|\Rdesignmat_{0,p}^{1/2}x\|\leq\frac{1}{4}\big\}$,
\begin{equation*}
    g_{k,p}(x) = \frac{\indicator\big\{\|\Rdesignmat_{0,p}^{1/2}x\|\leq \frac{1}{4}\big\}}{\Norm_{k,p}}\exp\Big(-\|x\|_{\Rdesignmat_{k,p}}^2\Big)\,,
\end{equation*}
with $\Norm_{k,p}$ the normalisation constant. 

Besides, Jensen's inequality yields
\begin{align*}
    \int_{\lambda\in\bR^d}\frac{\indicator\{\|\Rdesignmat_{0,p}^{1/2}\lambda\|\leq \frac{1}{4}\}}{\Norm_{k,p}}&\exp\Bigg(\lambda^\top \Big(S_k -2\Rdesignmat_{k,p}\lambda^*_{t,p}\Big)-\|\lambda\|_{\Rdesignmat_{k,p}}^2\Bigg)d\lambda\\
    &= \int_{\bR^d}g_{k,p}(\lambda)\exp\Bigg(\lambda^\top \Big(S_k -2\Rdesignmat_{k,p}\lambda^*_{t,p}\Big)\Bigg)d\lambda\\
    &\geq \exp\Bigg(\int_{\bR^d}g_{t,p}(\lambda)\lambda^{\top}\Big(S_k -2\Rdesignmat_{k,p}\lambda_{t,p}^*\Big)d\lambda\Bigg)\\
    &= \exp\Bigg(\Big(S_k -2\Rdesignmat_{k,p}\lambda_{t,p}^*\Big)^\top\int_{\bR^d}g_{t,p}(\lambda)\lambda d\lambda\Bigg)\\
    &=1\,.
\end{align*}

Therefore, for all $k\in\bN^*$
\begin{equation*}
    1\geq \Bar{M}_{p,k} \geq \frac{\Norm_{k,p}}{\Norm_p}\exp(\lambda_{t,p}^{*\top} S_k- \|\lambda_{t,p}^*\|^2_{\Rdesignmat_{k,p}})\,.
\end{equation*}

Markov's inequality yields
\begin{align*}
    \delta 
    &\geq \bP\Big(\bar M_{p,k} \geq \frac{1}{\delta}\Big)\\
    &\geq \bP\Bigg(\frac{\Norm_{k,p}}{\Norm_p}\exp(\lambda_{t,p}^{*\top}S_k-\|\lambda_{t,p}^*\|_{\Rdesignmat_{k,p}}^2)\geq \frac{1}{\delta}\Bigg)\\
    &= \bP\Bigg(\lambda_{t,p}^{*\top}S_k-\|\lambda_{t,p}^*\|^2_{\Rdesignmat_{k,p}}\geq \log\Big(\frac{\Norm_p}{\Norm_{k,p}}\Big)+\log(1/\delta)\Bigg)\,.
\end{align*}

Taking $k=t$ in particular gives
\begin{align*}
   \delta 
    &\geq \bP\Bigg(\max_{\|\Rdesignmat_{0,p}^{1/2}\lambda\|\leq\frac{1}{4}}\lambda^\top S_t-\|\lambda\|^2_{\Rdesignmat_{t,p}}\geq \log\Big(\frac{\Norm_p}{\Norm_{t,p}}\Big)+\log(1/\delta)\Bigg)\,.
\end{align*}

The constraint on $\lambda$ in the inner expression prevent to use the usual optimal value for subgaussian r.v. which could give a bound for $\|S_t\|_{Z_{t,p}^{-1}}^2$. Instead, we introduce
\begin{equation*}
    \lambda_{t,p} = \frac{1}{4}\frac{Z_{t,p}^{-1}S_t}{\|S_t\|_{Z_{t,p}^{-1}}}\,,
\end{equation*}
for which
\begin{align*}
    \|\Rdesignmat_{0,p}^{1/2}\lambda_{t,p}\| 
    &\leq \frac{1}{4}\|\Rdesignmat_{0,p}^{1/2}\Rdesignmat_{t,p}^{-1/2}\|\frac{\|S_t\|_{\Rdesignmat_{t,p}^{-1}}}{\|S_t\|_{\Rdesignmat_{t,p}^{-1}}} \\
    &\leq \frac{1}{4}\,.
\end{align*}

Then
\begin{align*}
    \delta
    &\geq \bP \Bigg(\frac{1}{4} \|S_t\|_{\Rdesignmat_{t,p}^{-1}} - \frac{1}{16}\|S_t\|_{\Rdesignmat_{t,p}^{-1}} \geq\log\Big(\frac{\Norm_p}{\Norm_{t,p}}\Big)+\log(1/\delta)\Bigg)\\
    &= \bP \Bigg(\|S_t\|_{\Rdesignmat_{t,p}^{-1}} \geq\frac{16}{3}\log\Big(\frac{\Norm_p}{\Norm_{t,p}}\Big)+\frac{16}{3}\log(1/\delta)\Bigg)\\
    &\geq \bP \Bigg(\|S_t\|_{\Rdesignmat_{t,p}^{-1}} \geq6\log\Big(\frac{\Norm_p}{\Norm_{t,p}}\Big)+6\log(1/\delta)\Bigg)\,.
\end{align*}
\end{proof}

\subsection{Proof of Lemma~\ref{lem:Dimension}}
\label{app:Dimension}
\Dimension*

\begin{proof}
    Let $t\geq d(d+1)/2$, $0<\epsilon<1$ and $p\in[P_{t,\epsilon}]^d$. 
    Then, following steps from \citet{faury2020improved} yields
    \begin{align*}
    \Norm_p &= \int_{\lambda\in\bR^d} \indicator\Big\{\|\Rdesignmat_{0, p}^{1/2}\lambda\|\leq \frac{1}{2}\Big\} \exp\{-\|\lambda\|_{\Rdesignmat_{0, p}}^2\}d\lambda\\
    &= \frac{1}{\sqrt{\det(\Rdesignmat_{0, p})}}\int_{\lambda\in\bR^d} \indicator\Big\{\|\lambda\|\leq \frac{1}{2}\Big\} \exp\{-\|\lambda\|^2\}d\lambda\,,
\end{align*}
and
\begin{align*}
    \Norm_{t,p} 
    &= \int_{\lambda\in\bR^d} \indicator\Big\{\|\Rdesignmat_{0, p}^{1/2}\lambda\|\leq \frac{1}{4}\Big\} \exp\{-\|\lambda\|_{\Rdesignmat_{t,p}}^2\}d\lambda\\
    &= \frac{1}{\sqrt{\det(\Rdesignmat_{t,p})}} \int_{\lambda\in\bR^d} \indicator\Big\{\|\Rdesignmat_{0, p}^{1/2}\Rdesignmat_{t,p}^{-1/2}\lambda\|\leq \frac{1}{4}\Big\} \exp\{-\|\lambda\|^2\}d\lambda\,.
\end{align*}
Noting that $\|\Rdesignmat_{0,p}^{1/2}\Rdesignmat_{t,p}^{-1/2}\|\leq 1$, we deduce
\begin{align*}
    \Norm_{t,p} &\geq \frac{1}{\sqrt{\det(\Rdesignmat_{t,p})}} \int_{\bR^d} \indicator\Big\{\|\lambda\|\leq \frac{1}{4}\Big\} \exp\{-\|\lambda\|^2\}d\lambda\,.
\end{align*}
Therefore,
\begin{equation*}
    \frac{\Norm_p}{\Norm_{t,p}}\leq \sqrt{\frac{\det(\Rdesignmat_{t,p})}{\det(\Rdesignmat_{0,p})}} \frac{\int_{\bR^d} \indicator\Big\{\|\lambda\|\leq \frac{1}{2}\Big\} \exp\{-\|\lambda\|^2\}d\lambda}{\int_{\bR^d} \indicator\Big\{\|\lambda\|\leq \frac{1}{4}\Big\} \exp\{-\|\lambda\|^2\}d\lambda}\,.
\end{equation*}
We treat the integrals as
\begin{align*}
    \frac{\int_{\bR^d} \indicator\Big\{\|\lambda\|\leq \frac{1}{2}\Big\} \exp\{-\|\lambda\|^2\}d\lambda}{\int_{\bR^d} \indicator\Big\{\lambda\|\leq \frac{1}{4}\Big\} \exp\{-\|\lambda\|^2\}d\lambda}
    &=\frac{\int_{\bR^d} \Bigg(\indicator\Big\{\|\lambda\|\leq \frac{1}{4}\Big\} + \indicator\Big\{\frac{1}{4} < \|\lambda\|\leq \frac{1}{2}\Big\} \Bigg)\exp\{-\|\lambda\|^2\}d\lambda}{\int_{\bR^d} \indicator\Big\{\lambda\|\leq \frac{1}{4}\Big\} \exp\{-\|\lambda\|^2\}d\lambda}\\
    &= 1 + \frac{\int_{\bR^d} \indicator\Big\{\frac{1}{4} < \|\lambda\|\leq \frac{1}{2}\Big\}\exp\{-\|\lambda\|^2\}d\lambda}{\int_{\bR^d} \indicator\Big\{\|\lambda\|\leq \frac{1}{4}\Big\} \exp\{-\|\lambda\|^2\}d\lambda}\\
    &\leq 1 + \frac{\exp(-1/16)}{\exp(-1/16)}\frac{\int_{\bR^d} \indicator\Big\{\frac{1}{4} < \|\lambda\|\leq \frac{1}{2}\Big\}d\lambda}{\int_{\bR^d} \indicator\Big\{\|\lambda\|\leq \frac{1}{4}\Big\}d\lambda}\\
    &=2^d\,.
\end{align*}
Thus
\begin{align*}
    \log\Bigg(\frac{\Norm_p}{\Norm_{t,p}}\Bigg)
    &\leq d\log(2) + \frac{1}{2}\log\Bigg(\frac{\det(\Rdesignmat_{t,p})}{\det(\Rdesignmat_{0,p})}\Bigg)\\
    &= d\log(2) + \frac{1}{2}\log\Bigg(\det\Big(\diagI + \Rdesignmat_{0, p}^{-1/2}\designmat_t\Rdesignmat_{0, p}^{-1/2}\Big)\Bigg)\\
    &\leq d\log(2) + \frac{1}{2}\log\Bigg(\prod_{i\in[d]}\Big(1+\frac{n_{t, (i,i)}\SSigma_{i,i}}{(1+\epsilon)^{p_i}\SSigma_{i,i}+\|B\|}\Big)\Bigg)\\
    &\leq d\log(2) + \frac{1}{2}\log\Bigg(\prod_{i\in[d]}\Big(1+\frac{n_{t, (i,i)}}{(1+\epsilon)^{p_i}}\Big)\Bigg)\,.
\end{align*}
In particular under event $(t\in\cD_p)$,
\begin{equation*}
    \log\Bigg(\frac{\Norm_p}{\Norm_{t,p}}\Bigg) \leq d\log(2) + \frac{d}{2}\log(2+\epsilon)\,.
\end{equation*}
\end{proof}

\section{Concentration of the covariances estimations (Proposition~\ref{prop:ChiDeviations2})}
\FavorableEventCovariance*
\label{app:Covariance}

It is a direct application of the following proposition: 

\begin{restatable}{proposition}{ChiDeviations}

\label{prop:ChiDeviations}
Let $\delta\in(0,1)$. Then with probability $1-\delta$, for all $t\geq d(d+1)/2$ and $(i,j)\in[d]^2$ ``reachable'',
    \begin{align*}
        |\hat{\chi}_{t,(i,j)} -\SSigma_{i,j}| 
        &\leq \frac{B_iB_j}{4}\bigg(\frac{5h_{t,\delta}}{\sqrt{n_{t,(i,j)}}}+\frac{h_{t,\delta}^2}{n_{t,(i,j)}}+\frac{1}{n_{t,(i,j)}^2}\bigg)\,.
    \end{align*}
    where $h_{t,\delta} = (1+2\log(1/\delta)+2\log(d(d+1))+\log(1+t))^{1/2}$.
\end{restatable}

\begin{proof}
    Let $\delta>0$, $t\geq d(d+A)/2$. We remind 
    \begin{align*}
    \cC_t = \big\{\forall (i,j)\in[d]^2 \text{ ``reachable''}, \ \hat{\SSigma}_{t,(i,j)} \geq  \SSigma_{i,j} \big\}\,.
\end{align*}
    
Let $(i,j)\in[d]^2$ ``reachable''. Then 
\begin{align*}
    \hat{\chi}_{t,(i,j)} 
    &=  \hat{\Sq}_{t,(i,j)}-\hat{\mu}_{t,i}\hat\mu_{t,j}\\
    &=\frac{1}{n_{t,(i,j)}}\sum_{s=1}^t A_{s,i}A_{s,j}Y_{s,i}Y_{s,j} - \Big(\frac{1}{n_{t,i}}\sum_{s=1}^t A_{s,i}Y_{s,i}\Big)\Big(\frac{1}{n_{t,i}}\sum_{s=1}^t A_{s,j}Y_{s,j}\Big)\,,
\end{align*}
And,
\begin{align*}
    \hat{\chi}_{t,(i,j)} -\SSigma_{i,j} 
    &=\frac{1}{n_{t,(i,j)}}\sum_{s=1}^t A_{s,i}A_{s,j}Y_{s,i}Y_{s,j} - \Sq_{i,j}-\Bigg[\Big(\frac{1}{n_{t,i}}\sum_{s=1}^t A_{s,i}Y_{s,i}\Big)\Big(\frac{1}{n_{t,j}}\sum_{s=1}^t A_{s,j}Y_{s,j}\Big) - \mu_i\mu_j\Bigg] \\
    &= \frac{1}{n_{t,(i,j)}}\sum_{s=1}^t A_{s,i}A_{s,j}\Big[Y_{s,i}Y_{s,j} - \Sq_{i,j}\Big]-\Bigg[\bigg(\frac{1}{n_{t,i}}\sum_{s=1}^t A_{s,i}\Big[Y_{s,i}-\mu_i\Big]\bigg)\bigg(\frac{1}{n_{t,j}}\sum_{s=1}^t A_{s,j}\Big[Y_{s,j}-\mu_j\Big]\bigg) \\
    &\hspace{2cm} + \mu_j\bigg(\frac{1}{n_{t,i}}\sum_{s=1}^t A_{s,i}\Big[Y_{s,i}-\mu_i\Big]\bigg) + \mu_i\bigg(\frac{1}{n_{t,j}}\sum_{s=1}^t A_{s,j}\Big[Y_{s,j}-\mu_j\Big]\bigg)\Bigg]\,.
\end{align*}
A triangle inequality yields
\begin{align*}
    |\hat{\chi}_{t,(i,j)} -\SSigma_{i,j}|
    &\leq \Bigg|\frac{1}{n_{t,(i,j)}}\sum_{s=1}^t A_{s,i}A_{s,j}\Big[Y_{s,i}Y_{s,j} - \Sq_{i,j}\Big]\Bigg|+\bigg|\frac{1}{n_{t,i}}\sum_{s=1}^t A_{s,i}\Big[Y_{s,i}-\mu_i\Big]\bigg|\ \bigg|\frac{1}{n_{t,j}}\sum_{s=1}^t A_{s,j}\Big[Y_{s,j}-\mu_j\Big]\bigg|\\
    &\hspace{2cm} + \frac{B_j}{2} \bigg|\frac{1}{n_{t,i}}\sum_{s=1}^t A_{s,i}\Big[Y_{s,i}-\mu_i\Big]\bigg|+ \frac{B_i}{2} \bigg|\frac{1}{n_{t,j}}\sum_{s=1}^t A_{s,j}\Big[Y_{s,j}-\mu_j\Big]\bigg|\,.
\end{align*}

We make repeated use of the following Lemma
    \begin{restatable}{lemma}{OnlineDeviations}
    \label{lem:OnlineDeviations}
    Let $(\cH_t)_{t\in\bN^*}$ be a filtration, $(U_t)_{t\in\bN^*}$ be an $\cH_t$ adapted martingales bounded by $C\in\bR_+^*$ with $\bE[U_1]=0$, and $(\indicator\{V_t\})_{t\in\bN^*}$ be a predictable process and $\delta>0$. 
    
    Then with probability at least $1-\delta$, for all $t$                  
    \begin{align*}
    \bP\Bigg(\frac{\sum_{s=1}^t\indicator\{V_s\}U_s}{1+\sum_{s=1}^t\indicator\{V_s\}} > \frac{C}{\sqrt{1+\sum_{s=1}^t\indicator\{V_s\}}}\sqrt{2\log(1/\delta) + \log(1+\sum_{s=1}^t\indicator\{V_s\})}\Bigg)\leq \delta\,.
\end{align*}
\end{restatable}

Therefore, with probability at least $1-\delta/2$, for all $(i,j)$ and $t$,
\begin{align*}
    \Bigg|\frac{1}{n_{t,(i,j)}}\sum_{s=1}^t A_{s,i}A_{s,j}\Big[Y_{s,i}Y_{s,j} - \Sq_{i,j}\Big]\Bigg|
    &\leq \Bigg|\frac{1}{n_{t,(i,j)}+1}\sum_{s=1}^t A_{s,i}A_{s,j}\Big[Y_{s,i}Y_{s,j} - \Sq_{i,j}\Big]\Bigg|+\frac{B_iB_j}{4(n_{t,(i,j)}+1)}\\
    &\leq\frac{B_iB_j}{4}\frac{1}{\sqrt{n_{t,(i,j)}+1}}\sqrt{2\log(1/\delta)+2\log(d(d+1))+\log(1+t)}\\
    &\hspace{2cm}+\frac{B_iB_j}{4n_{t,(i,j)}}\\
    &\leq\frac{B_iB_j}{4\sqrt{n_{t,(i,j)}}}\sqrt{2\log(1/\delta)+2\log(d(d+1))+\log(1+t)}\\
    &\hspace{2cm}+\frac{B_iB_j}{4n_{t,(i,j)}}\,.
\end{align*}
With probability at least $1-\delta/2$, for all $i$ and $t$,
\begin{align*}
    \bigg|\frac{1}{n_{t,i}}\sum_{s=1}^t A_{s,i}\Big[Y_{s,i}-\mu_i\Big]\bigg|
    &\leq \bigg|\frac{1}{n_{t,i}+1}\sum_{s=1}^t A_{s,i}\Big[Y_{s,i}-\mu_i\Big]\bigg| + \frac{B_i}{2n_{t,(i,i)}}\\
    &\leq \frac{B_i}{2\sqrt{n_{t,(i,i)}}}\sqrt{2\log(1/\delta)+2\log(2d)+\log(1+t)} + \frac{B_i}{2n_{t,(i,i)}}\,.
\end{align*}

Therefore, reinjecting those expressions yields that with probability at least $1-\delta$, for all $(i,j)$ and $t$,
\begin{align*}
    |\hat{\chi}_{t,(i,j)} -\SSigma_{i,j}|
    &\leq \frac{B_iB_j}{4\sqrt{n_{t,(i,j)}}}\sqrt{2\log(1/\delta)+2\log(d(d+1))+\log(1+t)}+\frac{B_iB_j}{4n_{t,(i,j)}} \\
    &\hspace{1cm}+ \frac{B_iB_j\big(2\log(1/\delta)+2\log(2d)+\log(1+t)\big)}{4\sqrt{n_{t,(i,i)}n_{t,(j,j)}}}+\frac{B_iB_j}{4n_{t,(i,i)}n_{t,(j,j)}}\\
    &\hspace{1cm}+\frac{B_iB_j}{4}\Big(\frac{1}{n_{t,(j,j)}\sqrt{n_{t,(i,i)}}}+\frac{1}{n_{t,(i,i)}\sqrt{n_{t,(j,j)}}}\Big)\sqrt{2\log(1/\delta)+2\log(2d)+\log(1+t)}\\
    &\hspace{1cm}+\frac{B_iB_j}{4}\Big(\frac{1}{\sqrt{n_{t,(i,i)}}}+\frac{1}{\sqrt{n_{t,(j,j)}}}\Big)\sqrt{2\log(1/\delta)+2\log(2d)+\log(1+t)}\\
    &\leq \frac{B_iB_j}{4\sqrt{n_{t,(i,j)}}}\sqrt{2\log(1/\delta)+2\log(d(d+1))+\log(1+t)}+\frac{B_iB_j}{4n_{t,(i,j)}} \\
    &\hspace{1cm}+ \frac{B_iB_j\big(2\log(1/\delta)+2\log(2d)+\log(1+t)\big)}{4\sqrt{n_{t,(i,i)}n_{t,(j,j)}}}+\frac{B_iB_j}{4n_{t,(i,i)}n_{t,(j,j)}}\\
    &\hspace{1cm}+\frac{B_iB_j}{2}\Big(\frac{1}{\sqrt{n_{t,(i,i)}}}+\frac{1}{\sqrt{n_{t,(j,j)}}}\Big)\sqrt{2\log(1/\delta)+2\log(2d)+\log(1+t)}\,.
\end{align*}
To simplify this expression, using $n_{t,(i,j)}\leq \min\{n_{t,(i,i)}, n_{t,(j,j)}\}$ and $n_{t,(i,j)}\leq \sqrt{n_{t,(i,i)}n_{t,(j,j)}}$ yields
\begin{align*}
    |\hat{\chi}_{t,(i,j)} -\SSigma_{i,j}|
    &\leq 5\frac{B_iB_j}{4}\frac{1}{\sqrt{n_{t,(i,j)}}}\sqrt{2\log(1/\delta)+2\log(d(d+1))+\log(1+t)}\\
    &\hspace{1cm}+\frac{B_iB_j}{4}\frac{1}{n_{t,(i,j)}}\Big(1+2\log(1/\delta)+2\log(d(d+1))+\log(1+t)\Big)\\
    &\hspace{1cm} + \frac{B_iB_j}{4}\frac{1}{n_{t,(i,j)}^2}\,.
\end{align*}
Denoting $\smash h_{t,\delta} = \Big(1+2\log(1/\delta)+2\log\big(t\log(t)^2d(d+1)\big)+\log(1+t)\Big)^{1/2}$, we have with probability at least $1-\frac{2\delta}{d(d+1)t\log(t)^2}$
\begin{align*}
    |\hat{\chi}_{t,(i,j)} -\SSigma_{i,j}|
    &\leq \frac{B_iB_j}{4}\Big(\frac{5h_{t,\delta}}{\sqrt{n_{t,(i,j)}}}+\frac{h_{t,\delta}^2}{n_{t,(i,j)}}+\frac{1}{n_{t,(i,j)}^2}\Big)\,.
\end{align*}

A union bound yields the desired results.
\end{proof}

\subsection{Proof for Lemma~\ref{lem:OnlineDeviations}}
\OnlineDeviations*

\begin{proof}
    Let $t\geq 2$. Then $U_t$ is $C$ sub-Gaussian and for all $\lambda\in$
\begin{equation*}
    \bE\Bigg[\exp\bigg(\lambda\indicator\{V_t\}U_t - \frac{\lambda^2C^2}{2}\indicator\{V_t\}\bigg)\Bigg|\cH_{t-1}\Bigg] \leq 1
\end{equation*}
Then $(W_t(\lambda))_{t\in\bN^*}=\Big(\exp(\lambda\sum_{s=1}^t\indicator\{V_s\}U_s - \frac{\lambda^2C^2}{2}\sum_{s=1}^t\indicator\{V_s\})\Big)_{t\in\bN^*}$ is a supermatringale. We use the Method of Mixtures by integrating for a $\lambda\sim\cN(0, 1/C^2)$. This yield
\begin{align*}
    &\int_{\lambda\in\bR} \frac{C}{\sqrt{2\pi}} \exp\Big(-\frac{\lambda^2C^2}{2}\Big)W_t(\lambda)d\lambda\\
    &= \frac{C}{\sqrt{2\pi}}\int_{\lambda\in\bR}\exp\bigg(\lambda\sum_{s=1}^t\indicator\{V_s\}U_s - \frac{\lambda^2C^2}{2}(1+\sum_{s=1}^t\indicator\{V_s\})\bigg) d\lambda\\
    &= \frac{C}{\sqrt{2\pi}}\int_{\lambda\in\bR}\exp\bigg(\frac{\big(\sum_{s=1}^t\indicator\{V_s\}U_s\big)^2}{2C^2(1+\sum_{s=1}^t\indicator\{V_s\})} \\
    & \hspace{4cm}- \frac{1}{2}\big(\lambda-\frac{\sum_{s=1}^t\indicator\{V_s\}U_s}{C^2(1+\sum_{s=1}^t\indicator\{V_s\})}\big)^2C^2(1+\sum_{s=1}^t\indicator\{V_s\})\bigg) d\lambda\\
    &= \exp\bigg(\frac{\big(\sum_{s=1}^t\indicator\{V_s\}U_s\big)^2}{2C^2(1+\sum_{s=1}^t\indicator\{V_s\})}\bigg)\frac{1}{\sqrt{1+\sum_{s=1}^t\indicator\{V_s\}}}\\
    &\leq 1\,.
\end{align*}

Therefore,
\begin{align*}
    \bP\Bigg(\frac{\sum_{s=1}^t\indicator\{V_s\}U_s}{1+\sum_{s=1}^t\indicator\{V_s\}} > \frac{C}{\sqrt{1+\sum_{s=1}^t\indicator\{V_s\}}}\sqrt{2\log(1/\delta) + \log(1+\sum_{s=1}^t\indicator\{V_s\})}\Bigg)\leq \delta\,.
\end{align*}

Using the stopping time construction from \citet{abbasi2011improved} yields the property for all~$t$.
\end{proof}

\section{Behaviour in the high-probability events (Section~\ref{sec:Regret})}
\label{app:HPEvent}

The following proposition states that under some assumptions on the sequence of events $(\cE_t)$, the regret can be bounded by problem-dependent quantities (including $\SSigma$, $T$, or $d$). They are not all explicitly stated in \cref{prop:HPRegret} to make it adaptive to both algorithms but are hidden in the constants.

\begin{restatable}{proposition}{HPRegret}
    \label{prop:HPRegret}
    Let $r\in\bN$, $e\in(1, +\infty)^r$. Let $(\cE_t)_{t\geq d(d+1)/2}$ be a sequence of events such that for all $t\geq d(d+1)/2$, under $\cE_t$,
    \begin{equation}
        \frac{\Delta_{A_{t+1}}^2}{C} \leq \sum_{i\in A_{t+1}}\frac{C_{A_{t+1},i}}{n_{t,(i,j)}} + \sum_{s\in[r]}\Bigg[\sum_{(i,j)\in A_{t+1}} \frac{C_s}{n_{t,(i,j)}^{e_s}}\Bigg],\label{eq:DeltaAtDecomp} 
    \end{equation}
    where $C$ and $(C_s)_{s\in[r]}$ are problem-dependent positive constants. $C_{A_{t+1}, i}$ is a positive constant depending on $A_{t+1}$ and $i$ so that, for all $a\in\cA$, $C_{a,i}\leq 2m\SSigma_{i,i}$. Let $c\in\bR_+^*$ and $(c_s)_{s\in[r]}\in(\bR_+^*)^r$ be positive constants such that $1/c+\sum_{s\in[r]}1/c_s=1$.
    
    Then,
    \begin{align}
        &\sum_{t=d(d+1)/2}^{T-1} \Delta_{A_{t+1}}\indicator\{\cE_t\}\notag\\
    &\hspace{1cm}\leq 96c_1C\log(m)^2\sum_{i\in [d]}\Big(\max_{a\in\cA / i\in a}\frac{C_{a,i}}{\Delta_{a}}\Big)\notag\\
    &\hspace{1.5cm}+\sum_{s=1}^r \Bigg[\indicator\Big\{e_s=2\Big\}346\Big(c_sCC_s\log(m)\Big)^{1/2}md^2\Bigg(1 + \log\Big(\frac{\Delta_{\max}}{\Delta_{\min}}\Big)\Bigg) \notag\\
    &\hspace{2cm} + \indicator\Big\{1<e_s<2\Big\}60.30^{1/e_s}\Big(c_sCC_s\log(m)\Big)^{1/e_s}d^2m^{2/e_s}\Delta_{\min}^{1-2/e_s}\notag\\
    &\hspace{2cm} + \indicator\Big\{2<e_s\Big\}60.30^{1/e_s}\Big(c_sCC_s\log(m)\Big)^{1/e_s} \frac{e_s}{e_s-2}d^2m^{2/e_s}\Delta_{\max}^{1-2/e_s}\Bigg]\,.
    \end{align}
    where $(\alpha_k)_{k\in\bN^*}$, $(\beta_k)_{k\in\bN^*}$ and $k_0\in\bN^*$ are defined in \cref{app:DefAlphaBeta}.
\end{restatable}

\begin{proof}
   The proof is classical and involves a decomposition of the events \(\cE_t\) (see \citet{kveton2015tight, degenne2016combinatorial, perrault2020covariance}). By considering each of the $r$ sub-sum in \cref{eq:DeltaAtDecomp} and designing sets of event that can happen only a finite number of times. 

    We introduce two sequences $(\alpha_k)_{k\in\bN^*}$ and $(\beta_k)_{k\in\bN^*}$, both begin at $1$ and strictly decrease to $0$ (see \cref{app:DefAlphaBeta} for their definitions). These sequences are introduced to be able to consider the different terms of \cref{eq:DeltaAtDecomp} separately.

    Let $(c_s)_{s\in[r]}\in(\bR_+^*)^r$ such that $\sum_{s\in[r]}1/c_s=1$.

    Let $t\geq d(d+1)/2$, $k\in\bN^*$. We define the set
    \begin{align}
    S_{t,k} &= \Big\{i\in A_{t+1},\quad n_{t,(i,i)}\leq c_1m\alpha_k \ \frac{C}{\Delta_{A_{t+1}}^2}\ \frac{C_{A_{t+1},i}^2}{\SSigma_{i,i}^*}\Big\}\,,\label{eq:Stk}
    \end{align}
    and the event
    \begin{align}
    \bA_{t,k} &= \bigg\{\sum_{i\in S_{t,k}}\frac{\SSigma_{i,i}^*}{C_{A_{t+1},i}} \geq \beta_k m ;\quad  \forall l<k, \sum_{i\in S_{t,l}^1}\frac{\SSigma_{i,i}}{C_{A_{t+1},i}} < \beta_l m\bigg\}\,.\label{eq:Atk}
    \end{align}
    A notable difference from previous approaches is the use of \(\SSigma_{i,i}^*/C_{a,i}\) in $\bA_{t,k}^1$ instead of set cardinals. This enables the explicit appearance of the $C_{a, i}$ coefficients, which will involve the \(\sigma_{a,i}^2\) for the application of this proposition to our algorithms.
    
    For $s\in[r]$, we define
    \begin{align}
    S_{t,k}^s &= \Big\{(i,j)\in A_{t+1},\quad n_{t,(i,j)}^{e_s}\leq c_sm^2\alpha_k\ \frac{C}{\Delta_{A_{t+1}}^2}\ C_s\Big\}\, \label{eq:Stks}
    \end{align}
    and the events
    \begin{align}
        \bA_{t,k}^s &= \Big\{|S_{t,k}^s| \geq \beta_k m^2 ;\quad \forall l<k, |S_{t,l}^s| < \beta_l m^2\Big\}\label{eq:Atks}\,.
    \end{align}

    The following Lemma, proven in \cref{app:EtDecomp}, decomposes $(\cE_t)_{t\geq d(d+1)/2}$ using these events.
   \begin{restatable}{lemma}{EtDecomp}
    \label{lem:EtDecomp}
        Let's consider the assumptions of \cref{prop:HPRegret}. Let $\bA_{t,k}$ and $(\bA_{t,k}^s)_{s\in[r]}$ be the events defined in \cref{eq:Atk} and \cref{eq:Atks}. Let $k_0\in\bN^*$ such that \(0< m\beta_{k_0} < \frac{1}{2m}\) and $t\geq d(d+1)/2$. 
       \begin{equation*}
            \indicator\{\cE_t\} \leq \sum_{k=1}^{k_0}\indicator\{\bA_{t,k}\}+\sum_{s=1}^r\sum_{k=1}^{k_0}\indicator\{\bA_{t,k}^s\}\,.
        \end{equation*}
   \end{restatable}

Using it, we decompose
\begin{align}
\label{eq:SeparationRegret}
    \sum_{t=d(d+1)/2}^{T-1} \Delta_{A_{t+1}}\indicator\{\cE_t\}
    &\leq \sum_{t=d(d+1)/2}^{T-1}\Bigg[ \sum_{k=1}^{k_0}  \Delta_{A_{t+1}}\indicator\{\bA_{t,k}\} + \sum_{s=1}^r \sum_{k=1}^{k_0}  \Delta_{A_{t+1}}\indicator\{\bA_{t,k}^s\}\Big]\nonumber\\
    &= \sum_{t=d(d+1)/2}^{T-1} \Big[\Delta_{A_{t+1}} \sum_{k=1}^{k_0}  \indicator\{\bA_{t,k}\}\Big] + \sum_{s=1}^r \sum_{t=d(d+1)/2}^{T-1} \Big[\Delta_{A_{t+1}} \sum_{k=1}^{k_0}  \indicator\{\bA_{t,k}^s\}\Big]\,.
\end{align}

We begin with the first term of \cref{eq:SeparationRegret}. Let $t\geq d(d+1)/2$, and \(k\in[k_0]\). Then,
\begin{equation*}
        \bA_{t,k}
        = \bigg\{\sum_{i\in S_{t,k}^1}\frac{\SSigma_{i,i}}{C_{A_{t+1},i}} \geq \beta_k m ;\quad \forall l<k, \sum_{i\in S_{t,l}^1}\frac{\SSigma_{i,i}}{C_{A_{t+1},i}} < \beta_l\bigg\}
        \subseteq \bigg\{\frac{1}{\beta_km}\sum_{i\in S_{t,k}^1}\frac{\SSigma_{i,i}}{C_{A_{t+1},i}} \geq 1\bigg\}\,.
\end{equation*}
Therefore,
\begin{equation}
\label{eq:MajorEventAtk}
    \begin{aligned}
        \indicator\{\bA_{t,k}\}\leq \frac{1}{\beta_km}\sum_{i\in[d]}\frac{\SSigma_{i,i}}{C_{A_{t+1},i}}\indicator\Big\{\bA_{t,k} \cap \{i\in S_{t,k}\}\Big\}\,.
    \end{aligned}
\end{equation}

Summing over $t$ and $k$, and including the gaps yields
\begin{align} 
    \sum_{t=d(d+1)/2}^{T-1}& \Delta_{A_{t+1}}\sum_{k=1}^{k_0}\indicator\{\bA_{t,k}^1\} \\
    &\leq \sum_{t=d(d+1)/2}^T\Delta_{A_{t+1}}\sum_{k=1}^{k_0}\frac{1}{\beta_km}\sum_{i\in[d]}\frac{\SSigma_{i,i}}{C_{A_{t+1},i}}\indicator\Big\{\bA_{t,k}^1 \cap \{i\in S_{t,k}\}\Big\} \leftarrow \text{ by Eq.~\eqref{eq:MajorEventAtk}}\nonumber\\
    &\leq \sum_{i\in [d]}\SSigma_{i,i}\sum_{t=d(d+1)/2}^T\sum_{k=1}^{k_0}\frac{1}{\beta_km}\frac{\Delta_{A_{t+1}}}{C_{A_{t+1}, i}}\indicator\{i\in S_{t,k}\}\nonumber\\
    &= \sum_{i\in [d]}\SSigma_{i,i}\sum_{k=1}^{k_0}\frac{1}{\beta_km}\sum_{t=d(d+1)/2}^T\frac{\Delta_{A_{t+1}}}{C_{A_{t+1}, i}}\indicator\Bigg\{n_{t, (i,i)}\leq c_1m\alpha_k\frac{C}{\big(\frac{\Delta_{A_{t+1}}}{C_{A_{t+1}, i}}\big)^2\SSigma_{i,i}}\Bigg\}\,. \leftarrow \text{ by Eq. \eqref{eq:Stk}} \label{eq:SummingGaps}
\end{align}

Let \(i\in[d]\), we consider all the actions associated to it. Let \(q_i\in\bN^*\) be the number of actions associated to item \(i\). Let \(l\in[q_i]\), we denote \(e_{i}^l\in\cA\) the \(l\)-th action associated to item \(i\), sorted by decreasing \(\frac{\Delta_{e_{i}^l}}{C_{e_{i}^l,i}}\), with \(\frac{C_{e_{i}^0,i}}{\Delta_{e_{i}^0}} = 0\) by convention. Then

{\allowdisplaybreaks
\begin{align}
    &\sum_{t=d(d+1)/2}^{T-1}\frac{\Delta_{A_{t+1}}}{C_{A_{t+1}, i}}\indicator\Bigg\{n_{t, (i,i)}\leq c_1m\alpha_k\frac{C}{\big(\frac{\Delta_{A_{t+1}}}{C_{A_{t+1}, i}}\big)^2\SSigma_{i,i}}\Bigg\}\nonumber\\
    &\leq \sum_{t=0}^{T-1}\sum_{l=1}^{q_i}\frac{\Delta_{e_i^l}}{C_{e_i^l, i}}\indicator\Bigg\{n_{t, (i,i)}\leq c_1m\alpha_k\frac{C}{\big(\frac{\Delta_{e_i^l}}{C_{e_i^l, i}}\big)^2\SSigma_{i,i}}, \quad A_{t+1} = e_i^l\Bigg\}\nonumber\\
    &= \sum_{t=0}^{T-1}\sum_{l=1}^{q_i}\frac{\Delta_{e_i^l}}{C_{e_i^l, i}}\indicator\Bigg\{n_{t, (i,i)}\frac{\SSigma_{i,i}}{c_1m\alpha_k C}\leq \frac{1}{\big(\frac{\Delta_{e_i^l}}{C_{e_i^l, i}}\big)^2}, \quad A_{t+1} = e_i^l\Bigg\}\nonumber\\
    &= \sum_{t=0}^{T-1}\sum_{l=1}^{q_i}\frac{\Delta_{e_i^l}}{C_{e_i^l, i}}\sum_{p=1}^l\indicator\Bigg\{\frac{1}{\big(\frac{\Delta_{e_i^{p-1}}}{C_{e_i^{p-1}, i}^2}\big)^2}<n_{t, (i,i)}\frac{\SSigma_{i,i}}{c_1m\alpha_k C}\leq \frac{1}{\big(\frac{\Delta_{e_i^p}}{C_{e_i^p, i}^2}\big)}, \quad A_{t+1} = e_i^l\Bigg\}\leftarrow \text{decomposing the event}\nonumber \\
    &\leq \sum_{t=0}^{T-1}\sum_{l=1}^{q_i}\sum_{p=1}^l\frac{\Delta_{e_i^p}}{C_{e_i^p, i}}\indicator\Bigg\{\frac{1}{\big(\frac{\Delta_{e_i^{p-1}}}{C_{e_i^{p-1}, i}}\big)^2} < n_{t, (i,i)}\frac{\SSigma_{i,i}}{c_1m\alpha_kC} \leq \frac{1}{\big(\frac{\Delta_{e_i^p}}{C_{e_i^p, i}^2}\big)},\quad A_{t+1} = e_i^l\Bigg\} \leftarrow \text{ as } \frac{\Delta_{e_i^l}}{C_{e_i^l, i}} \leq \frac{\Delta_{e_i^p}}{C_{e_i^p, i}}\nonumber\\
    &= \sum_{p=1}^{q_i}\frac{\Delta_{e_i^p}}{C_{e_i^p, i}}\sum_{t=0}^{T-1}\sum_{l=p}^{q_i}\indicator\Bigg\{ \frac{1}{\big(\frac{\Delta_{e_i^{p-1}}}{C_{e_i^{p-1}, i}}\big)^2} < n_{t, (i,i)}\frac{\SSigma_{i,i}}{c_1m\alpha_k C} \leq \frac{1}{\big(\frac{\Delta_{e_i^p}}{C_{e_i^p, i}^2}\big)},\quad A_{t+1} = e_i^l\Bigg\}\nonumber\\
    &\leq \sum_{p=1}^{q_i}\frac{\Delta_{e_i^p}}{C_{e_i^p, i}}\sum_{t=0}^{T-1}\sum_{l=1}^{q_i}\indicator\Bigg\{\frac{1}{\big(\frac{\Delta_{e_i^{p-1}}}{C_{e_i^{p-1}, i}}\big)^2} < n_{t, (i,i)}\frac{\SSigma_{i,i}}{c_1m\alpha_kC}\leq \frac{1}{\big(\frac{\Delta_{e_i^p}}{C_{e_i^p, i}^2}\big)},\quad A_{t+1} = e_i^l\Bigg\} \leftarrow \text{we extend the sum over }l\nonumber\\
    &= \sum_{p=1}^{q_i}\frac{\Delta_{e_i^p}}{C_{e_i^p, i}}\sum_{t=0}^{T-1}\indicator\Bigg\{\frac{1}{\big(\frac{\Delta_{e_i^{p-1}}}{C_{e_i^{p-1}, i}^2}\big)^2} < n_{t, (i,i)}\frac{\SSigma_{i,i}}{c_1m\alpha_k C} \leq \frac{1}{\big(\frac{\Delta_{e_i^p}}{C_{e_i^p, i}}\big)^2},\quad i\in A_{t+1}\Bigg\} \leftarrow\text{we simplify the inner sum}\nonumber\\
    &\leq \sum_{p=1}^{q_i}\frac{\Delta_{e_i^p}}{C_{e_i^p, i}}\Bigg(\bigg\lfloor\bigg(\frac{C_{e_i^p,i}}{\Delta_{e_i^p}}\bigg)^2\frac{c_1m\alpha_kC}{\SSigma_{i,i}}\bigg\rfloor-\bigg\lfloor\bigg(\frac{C_{e_i^{p-1},i}}{\Delta_{e_i^{p-1}}}\bigg)^2\frac{c_1m\alpha_kC}{\SSigma_{i,i}}\bigg\rfloor\Bigg) \leftarrow \text{the event can only happen a given nbr. of times}\nonumber\\
    &= \Bigg(\bigg\lfloor\bigg(\frac{C_{e_i^{q_i}, i}}{\Delta_{e_i^{q_i}}}\bigg)^2\frac{c_1m\alpha_kC}{\SSigma_{i,i}}\bigg\rfloor\frac{\Delta_{e_i^{q_i}}}{C_{e_i^{q_i}, i}}+\sum_{p=1}^{q_i-1}\bigg\lfloor\bigg(\frac{C_{e_i^p, i}}{\Delta_{e_i^p}}\bigg)^2\frac{c_1m\alpha_kC}{\SSigma_{i,i}}\bigg\rfloor\bigg(\frac{\Delta_{e_i^p}}{C_{e_i^p, i}}-\frac{\Delta_{e_i^{p+1}}}{C_{e_i^{p+1}, i}}\bigg)\Bigg)\nonumber \leftarrow \text{ summation by parts}\\
    &\leq \frac{c_1m\alpha_kC}{\SSigma_{i,i}}\Bigg(\frac{C_{e_i^{q_i}, i}}{\Delta_{e_i^{q_i}}}+\sum_{p=1}^{q_i-1}\bigg(\frac{C_{e_i^p, i}}{\Delta_{e_i^p}}\bigg)^2\bigg(\frac{\Delta_{e_i^p}}{C_{e_i^p, i}}-\frac{\Delta_{e_i^{p+1}}}{C_{e_i^{p+1}, i}}\bigg)\Bigg)\nonumber \leftarrow \text{everything is positive}\\
    &\leq \frac{c_1m\alpha_kC}{\SSigma_{i,i}}\Bigg(\frac{C_{e_i^{q_i}, i}}{\Delta_{e_i^{q_i}}}+\int_{\Big(\frac{\Delta_{e_i^{q_i}}}{C_{e_i^{q_i}, i}}\Big)}^{\Big(\frac{\Delta_{e_i^{1}}}{C_{e_i^{1}, i}}\Big)} \frac{1}{x^2}dx\Bigg)\nonumber \\
    &= \frac{c_1m\alpha_kC}{\SSigma_{i,i}}\Bigg(\frac{C_{e_i^{q_i}, i}^2}{\Delta_{e_i^{q_i}}}+\frac{C_{e_i^{q_i}}^2}{\Delta_{e_i^{q_i}}} - \frac{C_{e_i^{1}}^2}{\Delta_{e_i^{1}}}\Bigg)\nonumber \\
    &\leq \frac{2c_1m\alpha_kC}{\SSigma_{i,i}}\frac{C_{e_i^{q_i},i}}{\Delta_{e_i^{q_i}}}\nonumber\\
    &\leq \frac{2c_1m\alpha_kC}{\SSigma_{i,i}}\Big(\max_{a\in\cA/i\in a}\frac{C_{a,i}}{\Delta_{a}}\Big).\label{eq:BigEqSigmai}
\end{align}
}

Reinjecting Eq.~\eqref{eq:BigEqSigmai} into \cref{eq:SummingGaps} yields
\begin{align}  
    \sum_{t=d(d+1)/2}^{T-1}\Delta_{A_{t+1}}\sum_{k=1}^{k_0}\indicator\{\bA_{t,k}^1\}
    &\leq \sum_{i\in [d]}\SSigma_{i,i}\sum_{k=1}^{k_0}\frac{1}{\beta_km} \frac{2c_1mC\alpha_k}{\SSigma_{i,i}}\Big(\max_{a\in\cA / i\in a}\frac{C_{a,i}}{\Delta_{a}}\Big)\nonumber\\
    &= 2c_1C\Big(\sum_{k=1}^{k_0}\frac{\alpha_k}{\beta_k}\Big)\sum_{i\in [d]}\Big(\max_{a\in\cA / i\in a}\frac{C_{a,i}}{\Delta_{a}}\Big)\notag\\
    &= 96c_1C\log(m)^2\sum_{i\in [d]}\Big(\max_{a\in\cA / i\in a}\frac{C_{a,i}}{\Delta_{a}}\Big)\,.
\label{eq:RegretSigma} 
\end{align}

We treat the $r$ other terms in a similar way. 
Let $s\in[r]$, $t\geq d(d+1)/2$, and \(k\in[k_0]\), 
\begin{equation*}
        \bA_{t,k}^s
        = \bigg\{|S_{t,k}^s| \geq \beta_k m^2 ;\quad \forall l<k, |S_{t,l}^s| < \beta_l m^2\bigg\}
        \subseteq \bigg\{\frac{1}{\beta_km^2}|S_{t, k}^2| \geq 1\bigg\}\,.
\end{equation*}
Therefore,
\begin{equation}
\label{eq:MajorEventAtks}
    \begin{aligned}
        \indicator\{\bA_{t,k}^s\}\leq \frac{1}{\beta_km^2}\sum_{(i, j)\in[d]^2}\indicator\Big\{\bA_{t,k}^s \cap \{(i,j)\in S_{t,k}^s\}\Big\}\,.
    \end{aligned}
\end{equation}

Summing over $t$ and $k$ yields
{\allowdisplaybreaks
\begin{align} 
    \sum_{t=d(d+1)/2}^{T-1} & \Delta_{A_{t+1}}\sum_{k=1}^{k_0}\indicator\{\bA_{t,k}^s\} \nonumber\\
    &\leq \sum_{t=d(d+1)/2}^{T-1}\Delta_{A_{t+1}}\sum_{k=1}^{k_0}\frac{1}{\beta_km^2}\sum_{(i,j)\in[d]^2}\indicator\Big\{\bA_{t,k}^s \cap \{(i,j)\in S_{t,k}^s\}\Big\} \leftarrow \text{ by Eq.~\eqref{eq:MajorEventAtks}}\nonumber\\
    & \leq \sum_{(i,j)\in [d]^2}\sum_{t=d(d+1)/2}^{T-1}\sum_{k=1}^{k_0}\frac{1}{\beta_km^2}\Delta_{A_{t+1}}\indicator\{(i,j)\in S_{t,k}^s\}\nonumber\\
    & = \sum_{(i,j)\in [d]^2}\sum_{k=1}^{k_0}\frac{1}{\beta_km^2}\sum_{t=d(d+1)/2}^{T-1}\Delta_{A_{t+1}}\indicator\Bigg\{n_{t, (i,j)}\leq m^{2/e_s}(c_s\alpha_kCC_s)^{1/e_s}\frac{1}{\Delta_{A_{t+1}}^{2/e_s}}\Bigg\}\,. \leftarrow \text{ by Eq. \eqref{eq:Stks}} \label{eq:SummingGapss}
\end{align}
}

Let \((i,j)\in[d]^2\), we consider all the actions which are associated to it. Let $q_{(i,j)}\in \bN^*$ be the number of actions associated to the tuple $(i,j)$. Let \(l\in[q_{(i,j)}]\), this time, we denote \(e_{(i,j)}^l\in\cA\) the \(l\)-th action associated to tuple \((i,j)\), sorted by decreasing \(\Delta_{e_{(i,j)}^l}\), with \(\frac{1}{\Delta_{e_{(i,j)}^0}} = 0\) by convention. Then,

{\allowdisplaybreaks
\begin{align}
    &\sum_{t=d(d+1)/2}^{T-1}\Delta_{A_{t+1}}\indicator\Bigg\{n_{t, (i,j)}\leq m^{2/e_s}(c_s\alpha_kCC_s)^{1/e_s}\frac{1}{\Delta_{A_{t+1}}^{2/e_s}}\Bigg\}\nonumber\\
    &\leq \sum_{t=0}^{T-1}\sum_{l=1}^{q_{(i,j)}}\Delta_{e_{(i,j)}^l}\indicator\Bigg\{n_{t, (i,j)}\leq m^{2/e_s}(c_s\alpha_kCC_s)^{1/e_s}\frac{1}{\Delta_{A_{t+1}}^{2/e_s}}, \quad A_{t+1} = e_{(i,j)}^l\Bigg\}\nonumber\\
    &= \sum_{t=0}^{T-1}\sum_{l=1}^{q_{(i,j)}}\Delta_{e_{(i,j)}^l}\indicator\Bigg\{n_{t, (i,j)}m^{-2/e_s}(c_s\alpha_kCC_s)^{-1/e_s}\leq \frac{1}{\Delta_{e_{(i,j)}^l}^{2/e_s}}, \quad A_{t+1} = e_{(i,j)}^l\Bigg\}\nonumber\\
    &= \sum_{t=0}^{T-1}\sum_{l=1}^{q_{(i,j)}}\Delta_{e_{(i,j)}^l}\sum_{p=1}^l\indicator\Bigg\{\frac{1}{\Delta_{e^{p-1}_{(i,j)}}^{2/e_s}}<n_{t, (i,j)}m^{-2/e_s}(c_s\alpha_kCC_s)^{-1/e_s}\leq \frac{1}{\Delta_{e_{(i,j)}^p}^{2/e_s}}, \quad A_{t+1} = e_{(i,j)}^l\Bigg\}\nonumber\\
    &\leq \sum_{t=0}^{T-1}\sum_{p=1}^{q_{(i,j)}}\Delta_{e_{(i,j)}^p}\sum_{l=p}^{q_{(i,j)}}\indicator\Bigg\{\frac{1}{\Delta_{e^{p-1}_{(i,j)}}^{2/e_s}}<n_{t, (i,j)}m^{-2/e_s}(c_s\alpha_kCC_s)^{-1/e_s}\leq \frac{1}{\Delta_{e_{(i,j)}^p}^{2/e_s}}, \quad A_{t+1} = e_{(i,j)}^l\Bigg\}\nonumber\\
    &\leq \sum_{p=1}^{q_{(i,j)}}\Delta_{e_{(i,j)}^p}\sum_{t=0}^{T-1}\indicator\Bigg\{\frac{1}{\Delta_{e^{p-1}_{(i,j)}}^{2/e_s}}<n_{t, (i,j)}m^{-2/e_s}(c_s\alpha_kCC_s)^{-1/e_s}\leq \frac{1}{\Delta_{e_{(i,j)}^p}^{2/e_s}}, \quad i\in A_{t+1}\Bigg\}\nonumber\\
    &\leq \sum_{p=1}^{q_{(i,j)}}\Delta_{e_{(i,j)}^p}\Bigg(\bigg\lfloor\frac{m^{2/e_s}(c_s\alpha_kCC_s)^{1/e_s}}{\Delta_{e_{(i,j)}^p}^{2/e_s}}\bigg\rfloor-\bigg\lfloor\frac{m^{2/e_s}(c_s\alpha_kCC_s)^{1/e_s}}{\Delta_{e_{(i,j)}^{p-1}}^{2/e_s}}\bigg\rfloor\Bigg)\nonumber\\
    &= \bigg\lfloor\frac{m^{2/e_s}(c_s\alpha_kCC_s)^{1/e_s}}{\Delta_{e_{(i,j)}^{q_{(i,j)}}}^{2/e_s}}\bigg\rfloor\Delta_{e_{(i,j)}^{q_{(i,j)}}}+\sum_{p=1}^{q_{(i,j)}-1}\bigg\lfloor\frac{m^{2/e_s}(c_s\alpha_kCC_s)^{1/e_s}}{\Delta_{e_{(i,j)}^p}^{2/e_s}}\bigg\rfloor\Bigg(\Delta_{e_{(i,j)}^p}-\Delta_{e_{(i,j)}^{p+1}}\Bigg)\nonumber\\
    &\leq m^{2/e_s}(c_s\alpha_kCC_s)^{1/e_s} \Big(\Delta_{e_{(i,j)}^{q_{(i,j)}}}\Big)^{1-2/e_s}+\sum_{p=1}^{q_{(i,j)}-1}\Big(\Delta_{e_{(i,j)}}^p\Big)^{-2/e_s}\Bigg(\Delta_{e_{(i,j)}^p}-\Delta_{e_{(i,j)}^{p+1}}\Bigg)\nonumber\\
    &\leq m^{2/e_s}(c_s\alpha_kCC_s)^{1/e_s} \Bigg(\Big(\Delta_{e_{(i,j)}^{q_{(i,j)}}}\Big)^{1-2/e_s} + \int_{\Delta_{e_{(i,j)}^{q_{(i,j)}}}}^{\Delta_{e_{(i,j)}^{1}}}x^{-2/e_s}dx\Bigg)\,. \nonumber
\end{align}

\underline{If $e_s=2$}, then 
\begin{align*}
    &\sum_{t=d(d+1)/2}^{T-1}\Delta_{A_{t+1}}\indicator\Bigg\{n_{t, (i,j)}\leq m^{2/e_s}(c_s\alpha_kCC_s)^{1/e_s}\frac{1}{\Delta_{A_{t+1}}^{2/e_s}}\Bigg\}\nonumber\\
    &\leq m(c_s\alpha_kCC_s)^{1/2} \Bigg(1 + \int_{\Delta_{e_{(i,j)}^{q_{(i,j)}}}}^{\Delta_{e_{(i,j)}^{1}}}x^{-1}dx\Bigg)\\
    &\leq m(c_s\alpha_kCC_s)^{1/2} \Bigg(1 + \log\Big(\frac{\Delta_{\max}}{\Delta_{\min}}\Big)\Bigg)\,.
\end{align*}
}

Reinjecting this expression into yields \cref{eq:SummingGapss}, for $e_s=2$
\begin{align}
    \sum_{t=d(d+1)/2}^{T-1} \Delta_{A_{t+1}}\sum_{k=1}^{k_0}\indicator\{\bA_{t,k}^s\}
    &\leq \sum_{(i,j)\in[d]^2}\sum_{k\in[k_0]}\frac{1}{\beta_km^2}m(c_s\alpha_kCC_s)^{1/2} \Bigg(1 + \log\Big(\frac{\Delta_{\max}}{\Delta_{\min}}\Big)\Bigg)\nonumber\\
    &=(c_sCC_s)^{1/2}\frac{d^2}{m}\Bigg(1 + \log\Big(\frac{\Delta_{\max}}{\Delta_{\min}}\Big)\Bigg)\sum_{k\in[k_0]}\frac{\alpha_k^{1/2}}{\beta_k} \notag\\
    &\leq 346\Big(c_sCC_s\log(m)\Big)^{1/2}md^2\Bigg(1 + \log\Big(\frac{\Delta_{\max}}{\Delta_{\min}}\Big)\Bigg)\,.\label{eq:RegretS2}
\end{align}

\underline{Else, for $1<e_s< 2$}, then
\begin{align*}
    &\sum_{t=d(d+1)/2}^{T-1}\Delta_{A_{t+1}}\indicator\Bigg\{n_{t, (i,j)}\leq m^{2/e_s}(c_s\alpha_kCC_s)^{1/e_s}\frac{1}{\Delta_{A_{t+1}}^{2/e_s}}\Bigg\}\nonumber\\
    &\leq m^{2/e_s}(c_s\alpha_kCC_s)^{1/e_s} \Bigg(\Big(\Delta_{e_{(i,j)}^{q_{(i,j)}}}\Big)^{1-2/e_s} + \int_{\Delta_{e_{(i,j)}^{q_{(i,j)}}}}^{\Delta_{e_{(i,j)}^{1}}}x^{-2/e_s}dx\Bigg)\\
    &= m^{2/e_s}(c_s\alpha_kCC_s)^{1/e_s} \Bigg(\Big(\Delta_{e_{(i,j)}^{q_{(i,j)}}}\Big)^{1-2/e_s} + \frac{e_s}{e_s-2}\Bigg(\Delta_{e_{(i,j)}^{1}}^{1-2/e_s}-\Delta_{e_{(i,j)}^{q_{(i,j)}}}^{1-2/e_s}\Bigg)\Bigg)\\
    &\leq m^{2/e_s}(c_s\alpha_kCC_s)^{1/e_s} \Bigg(\Big(\Delta_{e_{(i,j)}^{q_{(i,j)}}}\Big)^{1-2/e_s} - e_s\Bigg(\Delta_{e_{(i,j)}^{1}}^{1-2/e_s}-\Delta_{e_{(i,j)}^{q_{(i,j)}}}^{1-2/e_s}\Bigg)\Bigg)\\
    &\leq 3m^{2/e_s}(c_s\alpha_kCC_s)^{1/e_s} \Delta_{\min}^{1-2/e_s}\,.
\end{align*}
This yield
\begin{align}
    \sum_{t=d(d+1)/2}^{T-1} \Delta_{A_{t+1}}\sum_{k=1}^{k_0}\indicator\{\bA_{t,k}^s\}
    &\leq \sum_{(i,j)\in[d]^2}\sum_{k\in[k_0]}\frac{1}{\beta_km^2}3m^{2/e_s}(c_s\alpha_kCC_s)^{1/e_s} \Delta_{\min}^{1-2/e_s}\notag\\
    &=3(c_sCC_s)^{1/e_s}d^2m^{2/e_s-2}\Delta_{\min}^{1-2/e_s}\sum_{k\in[k_0]}\frac{\alpha_k^{1/e_s}}{\beta_k}\notag\\
    &\leq 189.30^{1/e_s}\Big(c_sCC_s\log(m)\Big)^{1/e_s}d^2m^{2/e_s}\Delta_{\min}^{1-2/e_s}\,.\label{eq:RegretS1}
\end{align}

\underline{Finally, for $e_s>2$},
\begin{align*}
    &\sum_{t=d(d+1)/2}^{T-1}\Delta_{A_{t+1}}\indicator\Bigg\{n_{t, (i,j)}\leq m^{2/e_s}(c_s\alpha_kCC_s)^{1/e_s}\frac{1}{\Delta_{A_{t+1}}^{2/e_s}}\Bigg\}\nonumber\\
    &\leq m^{2/e_s}(c_s\alpha_kCC_s)^{1/e_s} \Bigg(\Big(\Delta_{e_{(i,j)}^{q_{(i,j)}}}\Big)^{1-2/e_s} + \int_{\Delta_{e_{(i,j)}^{q_{(i,j)}}}}^{\Delta_{e_{(i,j)}^{1}}}x^{-2/e_s}dx\Bigg)\\
    &= m^{2/e_s}(c_s\alpha_kCC_s)^{1/e_s} \Bigg(\Big(\Delta_{e_{(i,j)}^{q_{(i,j)}}}\Big)^{1-2/e_s} + \frac{e_s}{e_s-2}\Bigg(\Delta_{e_{(i,j)}^{1}}^{1-2/e_s}-\Delta_{e_{(i,j)}^{q_{(i,j)}}}^{1-2/e_s}\Bigg)\Bigg)\\
    &\leq m^{2/e_s}(c_s\alpha_kCC_s)^{1/e_s} \frac{e_s}{e_s-2}\Big(\Delta_{\max}\Big)^{1-2/e_s}\,,
\end{align*}
and
\begin{align}
    \sum_{t=d(d+1)/2}^{T-1} \Delta_{A_{t+1}}\sum_{k=1}^{k_0}\indicator\{\bA_{t,k}^s\}
    &\leq \sum_{(i,j)\in[d]^2}\sum_{k\in[k_0]}\frac{1}{\beta_km^2}m^{2/e_s}(c_s\alpha_kCC_s)^{1/e_s} \frac{e_s}{e_s-2}\Big(\Delta_{\max}\Big)^{1-2/e_s}\notag\\
    &= (c_sCC_s)^{1/e_s} \frac{e_s}{e_s-2}d^2m^{2/e_s-2}\Delta_{\max}^{1-2/e_s}\sum_{k\in[k_0]}\frac{\alpha_k^{1/e_s}}{\beta_k}\notag\\
    &\leq 63.30^{1/e_s}\Big(c_sCC_s\log(m)\Big)^{1/e_s} \frac{e_s}{e_s-2}d^2m^{2/e_s}\Delta_{\max}^{1-2/e_s}\,.\label{eq:RegretS3}
\end{align}

All in all, we reinject \cref{eq:RegretSigma}, \cref{eq:RegretS2}, \cref{eq:RegretS1} and  \cref{eq:RegretS3} into \cref{eq:SeparationRegret}, yielding
\begin{align}
    &\sum_{t=d(d+1)/2}^{T-1} \Delta_{A_{t+1}}\indicator\{\cE_t\}\notag\\
    &\hspace{1cm}\leq \sum_{t=d(d+1)/2}^{T-1} \Big[\Delta_{A_{t+1}} \sum_{k=1}^{k_0}  \indicator\{\bA_{t,k}\}\Big] + \sum_{s=1}^r \sum_{t=d(d+1)/2}^{T-1} \Big[\Delta_{A_{t+1}} \sum_{k=1}^{k_0}  \indicator\{\bA_{t,k}^s\}\Big]\notag\\
    &\hspace{1cm}\leq 96c_1C\log(m)^2\sum_{i\in [d]}\Big(\max_{a\in\cA / i\in a}\frac{C_{a,i}}{\Delta_{a}}\Big)\notag\\
    &\hspace{1.5cm}+\sum_{s=1}^r \Bigg[\indicator\Big\{e_s=2\Big\}346\Big(c_sCC_s\log(m)\Big)^{1/2}md^2\Bigg(1 + \log\Big(\frac{\Delta_{\max}}{\Delta_{\min}}\Big)\Bigg) \notag\\
    &\hspace{2cm} + \indicator\Big\{1<e_s<2\Big\}63.30^{1/e_s}\Big(c_sCC_s\log(m)\Big)^{1/e_s}d^2m^{2/e_s}\Delta_{\min}^{1-2/e_s}\notag\\
    &\hspace{2cm} + \indicator\Big\{2<e_s\Big\}63.30^{1/e_s}\Big(c_sCC_s\log(m)\Big)^{1/e_s} \frac{e_s}{e_s-2}d^2m^{2/e_s}\Delta_{\max}^{1-2/e_s}\Bigg]\,.
\end{align}

\end{proof}

\subsection{Definition of the sequences \texorpdfstring{$(\alpha_k)$ and $(\beta_k)$}{(alpha\_k) and (beta\_k)}}
\label{app:DefAlphaBeta}

Let \(\beta = 1/5\), \(x>0\). We define \(\beta_0=\alpha_0=1\). For \(k\geq1\), we define 
\begin{equation}
    \beta_k = \beta^k, \qquad \qquad \alpha_k = x\beta^k\,.
\end{equation}

Let's first look for an adequate \(k_0\) for \cref{lem:EtDecomp}, taking \(1\leq k_0 = \lceil\frac{2\log(\sqrt{2}m)}{\log(1/\beta)}+1\rceil\leq (2\log(m)+3)\) is sufficient to have \(0< m\beta_{k_0} < \frac{1}{2m}\).
This choice particularly yields
\begin{equation}
\label{eq:MajorAlphaBeta1}
\begin{aligned}
    \Bigg(\sum_{k=1}^{k_0-1}\frac{\beta_{k-1}-\beta_k}{\alpha_k} + \frac{\beta_{k_0-1}}{\alpha_{k_0}}\Bigg) 
    &= \Bigg(\sum_{k=1}^{k_0-1}\frac{1-\beta}{\beta} + \frac{1}{\beta}\Bigg)\frac{1}{x}\\
    &= \Bigg((k_0-1)\frac{1-\beta}{\beta} + \frac{1}{\beta}\Bigg)\frac{1}{x}\\
    &= \Bigg(4{k_0} + 1\Bigg)\frac{1}{x}\\
    &< 1\,,
\end{aligned}
\end{equation}
for \(x = 4k_0+2\).

Besides,
\begin{equation}
    \label{eq:SumAlphaBeta}
        \sum_{k=1}^{k_0} \frac{\alpha_k}{\beta_k} 
        = (4k_0+2)k_0
        \leq 16\log(m)^2+52\log(m)+42
        \leq 48\log(m)^2\,
\end{equation}
as $m\geq 5$.
Let $c\in\bR$, $c>1$. Then
\begin{align*}
        \sum_{k=1}^{k_0} \frac{\alpha_k^{1/c}}{\beta_k} 
        &= (4k_0+2)^{1/c}\ \sum_{k=1}^{k_0} (\beta^{1/c-1})^{k}\nonumber\\
        &\leq (8\log(m)+14)^{1/c}\sum_{k=1}^{k_0} (5^{\frac{c-1}{c}})^k\\
        &\leq 30^{1/c}\log(m)^{1/c}\sum_{k=1}^{k_0} 5^k\\
        &= 30^{1/c}\log(m)^{1/c}\ 5 \frac{5^{k_0}-1}{5-1}\nonumber\\
        &= 30^{1/c}\log(m)^{1/c}\ \frac{5}{4}(5^{k_0}-1)\nonumber\\
        &= 30^{1/c}\log(m)^{1/c}\ \frac{5}{4}\Bigg(\exp\bigg(\log(5)\Big(\frac{2\log(\sqrt{2}m)}{\log(5)}+2\Big)\bigg)-1\Bigg)\nonumber\\
        &= 30^{1/c}\log(m)^{1/c}\ \frac{5}{4}\Bigg(50m^2-1\Bigg)\nonumber\\
        &\leq 63m^2 \big(30^{1/c} \log(m)^{1/c}\big)\\
        &\leq 63.30^{1/c}m^2\log(m)^{1/c}\,.\nonumber\\
\end{align*}

\subsection{Proof of Lemma~\ref{lem:EtDecomp}}
\label{app:EtDecomp}
\EtDecomp*

\begin{proof}

Let's consider the assumptions of \ref{prop:HPRegret}, $\bA_{t,k}$ and $(\bA_{t,k}^s)_{s\in[r]}$ be the events defined in \cref{eq:Atk} and \cref{eq:Atks}. Let $k_0\in\bN^*$ such that \(0< m\beta_{k_0} < \frac{1}{2m}\) and $t\geq d(d+1)/2$. 

We first prove that the events for $k\geq k_0$ cannot happen. Let $k\geq k_0$, 
\begin{align*}
        \bA_{t,k} &= \bigg\{\sum_{i\in S_{t,k}^1}\frac{\SSigma_{i,i}^*}{C_{A_{t+1},i}} \geq \beta_k m ;\quad  \forall l<k, \sum_{i\in S_{t,l}^1}\frac{\SSigma_{i,i}}{C_{A_{t+1},i}} < \beta_l m\bigg\}\,.
\end{align*}
As \(\beta_km < \beta_{k_0}m < \frac{1}{2m} \leq \min_{i,a}\frac{\SSigma_{i,i}}{C_{a,i}}\) and $(S_{t,l}^1)_l$ is a decreasing sequence of sets, $\sum_{i\in S_{t,k_0}}\frac{\SSigma_{i,i}}{C_{A_{t+1},i}} < \beta_{k_0} m$ imply $S_{t, k_0} = \emptyset$ and $\sum_{i\in S_{t,k}}\frac{\SSigma_{i,i}^*}{C_{A_{t+1}}} = 0 < \beta_k m $. Therefore, $\bA_{t,k}$ cannot happen and we denote
\begin{align*}
    \bA_t
    = \bigcup_{k\geq 1}\bA_{t,k}
    = \bigcup_{k\in[k_0]}\bA_{t,k}
    = \bigcup_{k\in[k_0]}\bigg\{\sum_{i\in S_{t,k}}\frac{\SSigma_{i,i}}{C_{A_{t+1},i}} \geq \beta_k m ;\quad  \forall l<k, \sum_{i\in S_{t,l}^1}\frac{\SSigma_{i,i}}{C_{A_{t+1},i}} < \beta_l m\bigg\}\,.
\end{align*}

Likewise, for $k>k_0$ and $s\in[r]$, 
\begin{align*}
        \bA_{t,k}^s &= \bigg\{|S_{t,k}^s| \geq \beta_k m^2 ;\quad  \forall l<k, |S_{t,k}^s| < \beta_l m^2\bigg\}\,.
\end{align*}
As \(\beta_{k_0}m^2 < 1/2 < 1\) and $(S_{t,l}^s)_l$ is a decreasing sequence of sets, then $|S_{t,k_0}^s| < \beta_{k_0} m^2$ imply $S_{t, k_0}^s = \emptyset$ and $|S_{t,k}^s| = 0<\beta_km^2$. Therefore, $\bA_{t,k}^s$ cannot happen and we denote
\begin{align*}
    \bA_t^s
    = \bigcup_{k\geq 1}\bA_{t,k}^s
    = \bigcup_{k\in[k_0]}\bA_{t,k}^s
    = \bigcup_{k\in[k_0]}\bigg\{|S_{t,k}^s| \geq \beta_k m^2 ;\quad  \forall l<k, |S_{t,k}^s| < \beta_l m^2\bigg\}\,.
\end{align*}

The idea is now to prove that
\begin{equation*}
    \Bigg(\bA_t\cup\bigcup_{s=1}^r \bA_t^{s}\Bigg)^c
    = \bA_t^c\cap\cap_{s=1}^r\Big(\bA_t^{s}\Big)^c
    \subseteq \cE_t^c\,.
\end{equation*}

We begin by considering \((\bA_t^1)^c\),
\begin{align}
\label{eq:Atc}
        (\bA_t)^c 
        &= \cap_{k=1}^{k_0}(\bA_{t,k})^c \notag\\
        &= \cap_{k=1}^{k_0}\Bigg(\bigg\{\sum_{i\in S_{t,k}}\frac{\SSigma_{i,i}}{C_{A_{t+1},i}} < \beta_km\bigg\}\bigcup_{l=1}^{k-1}\bigg\{\sum_{i\in S_{t,l}^1}\frac{\SSigma_{i,i}}{C_{A_{t+1},i}}\geq \beta_lm\bigg\}\Bigg)\notag\\
        &= \cap_{k=1}^{k_0}\bigg\{\sum_{i\in S_{t,k}^1}\frac{\SSigma_{i,i}}{C_{A_{t+1},i}}<\beta_km\bigg\}\,.
\end{align}

Then, under \((\bA_t)^c\), denoting \(S_{t,0} = A_{t+1}\), as \(S_{t,k_0} = \emptyset\) and the sets \(S_{t,k}\) are decreasing with respect to \(k\),
{\allowdisplaybreaks
\begin{align}
\label{eq:PartitionSigmai}
        \sum_{i\in A_{t+1}}\frac{C_{A_{t+1}, i}}{n_{t,(i,i)}} 
        &= \sum_{k=1}^{k_0}\sum_{i\in S_{t,k-1}^1\setminus S_{t,k}^1}\frac{C_{A_{t+1}, i}}{n_{t,(i,i)}} \nonumber \\
        &\leq \sum_{k=1}^{k_0}\sum_{i\in S_{t,k-1}^1\setminus S_{t,k}^1}C_{A_{t+1},i}\frac{1}{3m\alpha_k}\frac{\Delta_{A_{t+1}}^2}{C}\frac{\SSigma_{i,i}^*}{C_{A_{t+1},i}^2} \leftarrow \text{ by Eq.~\eqref{eq:Stk}} \nonumber\\
        &= \frac{\Delta_{A_{t+1}}^2}{c_1mC}\sum_{k=1}^{k_0}\frac{1}{\alpha_k}\sum_{i\in S_{t,k-1}\setminus S_{t,k}}\frac{\SSigma_{i,i}}{C_{A_{t+1},i}}  \nonumber\\
        &= \frac{\Delta_{A_{t+1}}^2}{c_1mC}\sum_{k=1}^{k_0}\frac{1}{\alpha_k}\bigg(\sum_{i\in S_{t,k-1}^1}\frac{\SSigma_{i,i}}{C_{A_{t+1},i}}-\sum_{i\in S_{t,k}^1}\frac{\SSigma_{i,i}}{C_{A_{t+1},i}}\bigg)\nonumber\\
        &= \frac{\Delta_{A_{t+1}}^2}{c_1mC}\sum_{k=0}^{k_0-1}\frac{1}{\alpha_{k+1}}\bigg(\sum_{i\in S_{t,k}^1}\frac{\SSigma_{i,i}}{C_{A_{t+1},i}}-\sum_{i\in S_{t,k+1}^1}\frac{\SSigma_{i,i}}{C_{A_{t+1},i}}\bigg)\nonumber\\
        &= \frac{\Delta_{A_{t+1}}^2}{c_1mC}\Bigg(\frac{1}{\alpha_1}\sum_{i\in S_{t,0}^1}\frac{\SSigma_{i,i}}{C_{A_{t+1},i}}+\sum_{k=1}^{k_0-1}\bigg(\frac{1}{\alpha_{k+1}}-\frac{1}{\alpha_k}\bigg)\sum_{i\in S_{t,k}^1}\frac{\SSigma_{i,i}}{C_{A_{t+1},i}}\Bigg)\nonumber\\
        &< \frac{\Delta_{A_{t+1}}^2}{c_1mC}\Bigg(\frac{m}{\alpha_1} + \sum_{k=1}^{k_0-1}m\beta_k\bigg(\frac{1}{\alpha_{k+1}}-\frac{1}{\alpha_{k}}\bigg)\Bigg) \leftarrow S_{t,0} = A_{t+1} \text{ and Eq.~\eqref{eq:Atc}}\nonumber\\
        &= \frac{\Delta_{A_{t+1}}^2}{c_1C}\Bigg(\sum_{k=1}^{k_0-1}\frac{\beta_{k-1}-\beta_k}{\alpha_k} + \frac{\beta_{k_0-1}}{\alpha_{k_0}}\Bigg) \nonumber\\
        &\leq \frac{1}{c}\frac{\Delta_{A_{t+1}}^2}{C}\,. \leftarrow \text{ by Eq.~\eqref{eq:MajorAlphaBeta1}}
\end{align}}

Likewise for $s\in[r]$,
\begin{equation}
\label{eq:Atcs}
    \begin{aligned}
        (\bA_t^s)^c 
        &= \cap_{k=1}^{k_0}\bigg\{|S_{t,k}^s|<\beta_km^2\bigg\}\,.
    \end{aligned}
\end{equation}
Denoting \(S_{t,0}^s = A_{t+1}\times A_{t+1}\), as \(S_{t,k_0}^s = \emptyset\),
{\allowdisplaybreaks
\begin{align}
        \sum_{(i,j)\in A_{t+1}}\frac{C_s}{n_{t,(i,j)}^{e_s}}
        &= \sum_{k=1}^{k_0}\sum_{(i,j)\in S_{t,k-1}\setminus S_{t,k}}\frac{C_s}{n_{t,i}^{e_s}} \nonumber\\
        &\leq \sum_{k=1}^{k_0}\sum_{(i,j)\in S_{t,k-1}\setminus S_{t,k}}C_s\frac{1}{c_2m^2\alpha_k}\frac{\Delta_{A_{t+1}}^2}{C}\frac{1}{C_s} \leftarrow\text{ by Eq.~\eqref{eq:Stks}}\nonumber\\
        &= \frac{\Delta_{A_{t+1}}^2}{c_sm^2C}\sum_{k=1}^{k_0}\frac{1}{\alpha_k}\Big(|S_{t,k-1}|-|S_{t,k}|\Big)\nonumber\\
        &= \frac{\Delta_{A_{t+1}}^2}{c_sm ^2C}\sum_{k=0}^{k_0-1}\frac{1}{\alpha_{k+1}}\Big(|S_{t,k}|-|S_{t,k+1}|\Big)\nonumber\\
        &=\frac{\Delta_{A_{t+1}}^2}{c_sm^2C}\Bigg(\frac{| S_{t,0}|}{\alpha_1} + \sum_{k=1}^{k_0-1}| S_{t,k}|\bigg(\frac{1}{\alpha_{k+1}}-\frac{1}{\alpha_{k}}\bigg)\Bigg)\nonumber\\
        &< \frac{\Delta_{A_{t+1}}^2}{c_sm^2C}\Bigg(\frac{1}{\alpha_1}m^2 + \sum_{k=1}^{k_0-1}\beta_km^2\bigg(\frac{1}{\alpha_{k+1}}-\frac{1}{\alpha_{k}}\bigg)\Bigg) \leftarrow \text{ by Eq.~\eqref{eq:Atcs}}\nonumber\\
        &= \frac{\Delta_{A_{t+1}}^2}{c_2C}\Bigg(\sum_{k=1}^{k_0-1}\frac{\beta_{k-1}-\beta_k}{\alpha_k} + \frac{\beta_{k_0-1}}{\alpha_{k_0}}\Bigg)\nonumber\\
        &\leq \frac{1}{c_2}\frac{\Delta_{A_{t+1}}^2}{C}\,.\leftarrow \text{ by Eq.~\eqref{eq:MajorAlphaBeta1}} 
\label{eq:Partitions}
\end{align}}

Therefore, under  $\bA_t^c\cap\cap_{s=1}^r\Big(\bA_t^{s}\Big)^c$, summing \cref{eq:PartitionSigmai} and \cref{eq:Partitions}
\begin{align*}
    \sum_{i\in A_{t+1}}\frac{C_{A_{t+1},i}}{n_{t,(i,j)}} + \sum_{s\in[r]}\Bigg[\sum_{(i,j)\in A_{t+1}} \frac{C_s}{n_{t,(i,j)}^{e_s}}\Bigg]
    &< \Bigg(\frac{1}{c}+\sum_{s\in[r]}\frac{1}{c_s}\Bigg)\frac{\Delta_{A_{t+1}}^2}{C}\\
    &= \frac{\Delta_{A_{t+1}}^2}{C}\,,
\end{align*}
which contradict  \cref{eq:DeltaAtDecomp}  and thus imply \(\cE_t^c\). By contraposition, we have proved that \(\cE_t\) imply $\bA_t\cap\bigcup_{s=1}^r\big(\bA_{t}^s\big)$. Therefore,
\begin{equation*}
        \indicator\{\cE_t\} \leq \sum_{k=1}^{k_0}\indicator\{\bA_{t,k}\}+\sum_{s=1}^r\sum_{k=1}^{k_0}\indicator\{\bA_{t,k}^s\}\,.
\end{equation*}
\end{proof}

\section{Details for \OLSUCBC (Section~\ref{sec:TechnicalOLSUCBC})}
\label{app:OLSUCBC}

\subsection{Proof of Proposition~\ref{prop:OLSUCBCHPRegret}}
\OLSUCBCHPRegret*
\label{app:OLSUCBCHPRegret}

\begin{proof}
The objective is to use \cref{prop:HPRegret}.
\HPRegret*

We need to check that its hypotheses are satisfied. Let $t\geq d(d+1)/2$ and $\delta>0$, then we have the Lemma.
\begin{restatable}{lemma}{DeltaAtDecomp}
\label{lem:DeltaAtDecomp}
    Let $\delta>0$ and $t\geq d(d+1)/2$. Then under $\{\cG_{t}\cap \cC_t\}$, {\normalfont \OLSUCBC} satisfies
\begin{align*}
    \frac{\Delta_{A_{t+1}}^2}{f_{T, \delta}^2} 
    &\sum_{i\in A_{t+1}} \frac{4\bar \sigma_{A_{t+1}, i}^2}{n_{t,(i,i)}} + \sum_{(i,j)\in A_{t+1}}\frac{(4d+h_{t, \delta}^2)\|B\|^2_\infty}{n_{t,(i,j)}^2} + \sum_{(i,j)\in A_{t+1}} \frac{5h_{t, \delta}\|B\|^2_{\infty}}{n_{t,(i,j)}^{3/2}} + \sum_{(i,j)\in A_{t+1}}\frac{\|B\|^2_{\infty}}{n_{t,(i,j)}^3}\,, 
\end{align*}
where $\bar\sigma_{A_{t+1},i}^2 = 2\sum_{j\in A_{t+1} / \SSigma_{j,j}\leq\SSigma_{i,i}}(\SSigma_{i,j})_+ \leq 2 \sigma_{A_{t+1},i}^2$.
\end{restatable}

Therefore, we can choose $r=3$, $e=(2,\ 3/2,\ 3)$ and $(\cE_t) = (\cG_t\cap\cC_t)$. Taking $c=(4, 4, 4, 4)$ and identifying the rest of the coefficients yields that \OLSUCBC satisfies

\begin{align*}
        &\sum_{t=d(d+1)/2}^{T-1} \Delta_{A_{t+1}}\indicator\{\cG_t\cap\cC_t\}\notag\\
    &\hspace{1cm}\leq 384f_{T, \delta}^2\log(m)^2\sum_{i\in [d]}\Big(\max_{a\in\cA / i\in a}\frac{\bar \sigma_{a,i}^2}{\Delta_{a}}\Big)\\
    &\hspace{1.5cm}+ 692\|B\|_\infty f_{T, \delta}(4d+h_{t, \delta}^2)^{1/2}\log(m)^{1/2}\Bigg(1 + \log\Big(\frac{\Delta_{\max}}{\Delta_{\min}}\Big)\Bigg)\notag\\
    &\hspace{1.5cm} + 1460\|B\|_\infty^{4/3}f_{T, \delta}^{4/3}h_{t, \delta}^{2/3}\log(m)^{2/3}d^2m^{2/3}\Delta_{\min}^{-1/3}\\
    &\hspace{1.5cm} + 296\|B\|_\infty^{2/3}f_{T,\delta}^{2/3}\log(m)^{1/3}d^2m^{2/3}\Delta_{\max}^{1/3}\,.
    \end{align*}

As
\begin{itemize}[nosep, label=]
    \item $\smash  f_{t,\delta} = 6\log(1/\delta) + 6\Big(\log(t)+ (d+2)\log(\log(t))\Big) + 3d\Big(2\log(2) +\log(1+e)\Big)$,
    \item $\smash h_{t,\delta} = \Big(1+2\log(1/\delta)+2\log\big(t\log(t)^2d(d+1)\big)+\log(1+t)\Big)^{1/2}$,
\end{itemize}
we deduce 

\begin{align}
        &\sum_{t=d(d+1)/2}^{T-1} \Delta_{A_{t+1}}\indicator\{\cG_t\cap\cC_t\}= O\Bigg( \log(T)^2\log(m)^2\sum_{i\in [d]}\Big(\max_{a\in\cA / i\in a}\frac{\bar \sigma_{a,i}^2}{\Delta_{a}}\Big)\Bigg)\,.
    \end{align}

\end{proof}

\subsection{Proof of Lemma~\ref{lem:DeltaAtDecomp}}
\DeltaAtDecomp*

\begin{proof}
    Let $t\geq d(d+1)/2$ and $\delta>0$. \OLSUCBC statisfies the following Lemma.

\begin{restatable}{lemma}{DeltaGtC}
\label{lem:DeltaGtC}
Let \(t\geq d(d+1)/2\) and $\delta>0$. Then for \OLSUCBC, under the event \(\{\cG_t\cap\cC_t\}\), 
$$\Delta_{A_{t+1}} \leq f_{t, \delta}\big(\|\diagcounts_t^{-1}A_{t+1}\|_{\Rdesignmat_t} + \|\diagcounts_t^{-1}A_{t+1}\|_{\hat{\Rdesignmat}_t}\big)\,.$$
\end{restatable}
Therefore, under $\{\cG_t \cap \cC_t\}$, and
    \begin{align}
        0\leq \Delta_{A_{t+1}} &\leq f_{t, \delta}\big(\|\diagcounts_t^{-1}A_{t+1}\|_{\Rdesignmat_t} + \|\diagcounts_t^{-1}A_{t+1}\|_{\hat{\Rdesignmat}_t}\big)\notag\\
        \Delta_{A_{t+1}}^2 &\leq f_{t, \delta}^2\big(\|\diagcounts_t^{-1}A_{t+1}\|_{\Rdesignmat_t} + \|\diagcounts_t^{-1}A_{t+1}\|_{\hat{\Rdesignmat}_t}\big)^2 \nonumber\\
        &\leq 2f_{t, \delta}^2\big(\|\diagcounts_t^{-1}A_{t+1}\|_{\Rdesignmat_t}^2 + \|\diagcounts_t^{-1}A_{t+1}\|_{\hat{\Rdesignmat}_t}^2\big) \nonumber\\
        \frac{\Delta_{A_{t+1}}^2}{2f_{t\delta}^2} &\leq \|\diagcounts_t^{-1}A_{t+1}\|_{\Rdesignmat_t}^2 + \|\diagcounts_t^{-1}A_{t+1}\|_{\hat{\Rdesignmat}_t}^2\,. \label{eq:DeltaDecomp}
    \end{align}
    From, here, we develop the right-hand side,
    \begin{align*}
        \|\diagcounts_t^{-1}A_{t+1}\|_{\Rdesignmat_t}^2 
        &= A_{t+1}^\top \diagcounts_t^{-1}\Rdesignmat_t\diagcounts_t^{-1} A_{t+1}\\
        &= \sum_{(i,j)\in A_{t+1}} \frac{(\Rdesignmat_t)_{i,j}}{n_{t,(i,i)}n_{t, (j,j)}}\,.
    \end{align*}
    As $\Rdesignmat_t = \sum_{s=1}^t\diagmat_{A_s}{\SSigma}^*\diagmat_{A_s} + \diagmat_{{\SSigma}^*} \diagcounts_t+\|B\|^2\diagI$, we get
    \begin{align*}
        \|\diagcounts_t^{-1}A_{t+1}\|_{\Rdesignmat_t}^2 
        &= \sum_{(i,j)\in A_{t+1}} \frac{n_{t,(i,j)}\SSigma_{i,j}}{n_{t,(i,i)}n_{t, (j,j)}} + \sum_{i\in A_{t+1}}\frac{n_{t(i,i)}\SSigma_{i,i}}{n_{t,(i,i)}^2} + \sum_{i\in A_{t+1}}\frac{\|B\|^2}{n_{t,(i,i)}^2}\\
        &\leq \sum_{i\in A_{t+1}} \Big(2\sum_{j\in A_{t+1}/\SSigma_{j,j}\leq \SSigma_{i,i}}\frac{n_{t,(i,j)}\SSigma_{i,j}}{n_{t,(i,i)}n_{t, (j,j)}} \Big)+ \sum_{i\in A_{t+1}}\frac{\|B\|^2}{n_{t,(i,i)}^2}\,,
    \end{align*}
    by rearranging terms. 
    
    Now as for all $(i,j)\in[d]^2$, $n_{t, (i,j)}\leq \min\{n_{t, (i,i)}, n_{t, (j,j)}\}$, then 
    \begin{align*}
        \|\diagcounts_t^{-1}A_{t+1}\|_{\Rdesignmat_t}^2 
        &\leq \sum_{i\in A_{t+1}} \frac{1}{n_{t,(i,i)}}\Big(2\sum_{j\in A_{t+1}/\SSigma_{j,j}\leq \SSigma_{i,i}}\SSigma_{i,j} \Big) + \sum_{i\in A_{t+1}}\frac{\|B\|^2}{n_{t,(i,i)}^2}\,.
    \end{align*}
    Denoting $\bar\sigma^2_{A_{t+1}, i}=2\sum_{j\in A_{t+1}/\SSigma_{j,j}\leq \SSigma_{i,i}}(\SSigma_{i,j})_+$ yields
    \begin{equation}
        \|\diagcounts_t^{-1}A_{t+1}\|_{\Rdesignmat_t}^2 
        \leq \sum_{i\in A_{t+1}} \frac{\bar\sigma_{A_{t+1}, i}^2}{n_{t,(i,i)}} + \sum_{i\in A_{t+1}}\frac{\|B\|^2}{n_{t,(i,i)}^2}\,. \label{eq:DeltaDecomp1}
    \end{equation}
    The second term from the right-hand side of \cref{eq:DeltaDecomp} is developed in the same manner but involves more terms.
    \begin{align*}
        \|\diagcounts_t^{-1}A_{t+1}\|_{\hat\Rdesignmat_t}^2 
        &= A_{t+1}^\top \diagcounts_t^{-1}\hat\Rdesignmat_t\diagcounts_t^{-1} A_{t+1}\\
        &= \sum_{(i,j)\in A_{t+1}} \frac{(\hat\Rdesignmat_t)_{i,j}}{n_{t,(i,i)}n_{t, (j,j)}}\,.
    \end{align*}
    We remind that $\hat \Rdesignmat_t = \sum_{s=1}^t\diagmat_{A_s}\hat{\SSigma}_t\diagmat_{A_s} + \diagmat_{\hat{\SSigma}_t} \diagcounts_t+\|B\|^2\diagI$ where for all $(i,j)\in[d]^2$, 
    \begin{align*}
        \hat{\SSigma}_{t,(i,j)} &= \hat{\chi}_{t,(i,j)} + \frac{B_iB_j}{4}\bigg(\frac{5h_{t,\delta}}{\sqrt{n_{t,(i,j)}}}+\frac{h_{t,\delta}^2}{n_{t,(i,j)}}+\frac{1}{n_{t,(i,j)}^2}\bigg)\,.
    \end{align*}
    Being under the event $\cC$, \cref{prop:ChiDeviations2} yields $\hat{\SSigma}_{t,(i,j)} \leq  \SSigma_{i,j} + \frac{B_iB_j}{4}\bigg(\frac{5h_{t,\delta}}{\sqrt{n_{t,(i,j)}}}+\frac{h_{t,\delta}^2}{n_{t,(i,j)}}+\frac{1}{n_{t,(i,j)}^2}\bigg)$. Then,
    \begin{align}
    \|\diagcounts_t^{-1}&A_{t+1}\|_{\hat\Rdesignmat_t}^2 
        \leq \sum_{(i,j)\in A_{t+1}} \frac{n_{t,(i,j)}\SSigma_{i,j}}{n_{t,(i,i)}n_{t, (j,j)}} \notag\\
        & \hspace{.6cm}+ \sum_{(i,j)\in A_{t+1}} \frac{n_{t,(i,j)}}{n_{t,(i,i)}n_{t, (j,j)}}\frac{B_iB_j}{4}\bigg(\frac{5h_{t,\delta}}{\sqrt{n_{t,(i,j)}}}+\frac{h_{t,\delta}^2}{n_{t,(i,j)}}+\frac{1}{n_{t,(i,j)}^2}\bigg) \notag\\
        &\hspace{.6cm} + \sum_{i\in A_{t+1}}\frac{n_{t(i,i)}\SSigma_{i,i}}{n_{t,(i,i)}^2} + \sum_{i\in A_{t+1}}\frac{B_i^2}{4n_{t,(i,i)}}\bigg(\frac{5h_{t,\delta}}{\sqrt{n_{t,(i,i)}}}+\frac{h_{t,\delta}^2}{n_{t,(i,i)}}+\frac{1}{n_{t,(i,i)}^2}\bigg)\notag\\
        &\hspace{.6cm}+\sum_{i\in A_{t+1}}\frac{\|B\|^2}{n_{t,(i,i)}^2}\notag\\
        &\leq \sum_{i\in A_{t+1}} \frac{\bar \sigma_{A_{t+1}, i}^2}{n_{t,(i,i)}} + \sum_{i\in A_{t+1}}\frac{\|B\|^2}{n_{t,(i,i)}^2} \notag\\
        &\hspace{.6cm}+ \sum_{(i,j)\in A_{t+1}} \frac{5h_{t, \delta}\|B\|^2_{\infty}}{4}\frac{1}{n_{t,(i,j)}^{3/2}} + \sum_{(i,j)\in A_{t+1}} \frac{h_{t, \delta}^2\|B\|^2_{\infty}}{4}\frac{1}{n_{t,(i,j)}^2}+ \sum_{(i,j)\in A_{t+1}}\frac{\|B\|^2_{\infty}}{4}\frac{1}{n_{t,(i,j)}^3}\notag\\
        &\hspace{.6cm}+ \sum_{i\in A_{t+1}} \frac{5h_{t, \delta}\|B\|^2_\infty}{4}\frac{1}{n_{t,(i,i)}^{3/2}} + \sum_{i\in A_{t+1}} \frac{h_{t, \delta}^2\|B\|^2_\infty}{4}\frac{1}{n_{t,(i,i)}^2}+ \sum_{i\in A_{t+1}}\frac{\|B\|^2_\infty}{4}\frac{1}{n_{t,(i,i)}^3}\,.\label{eq:DeltaDecomp2}
    \end{align}
    Reinjecting \cref{eq:DeltaDecomp1} and \cref{eq:DeltaDecomp2} into \cref{eq:DeltaDecomp} yields
    \begin{align*}
        \frac{\Delta_{A_{t+1}}^2}{2f_{t,\delta}^2} 
        &\leq \|\diagcounts_t^{-1}A_{t+1}\|_{\Rdesignmat_t}^2 + \|\diagcounts_t^{-1}A_{t+1}\|_{\hat{\Rdesignmat}_t}^2\\
        &\leq \sum_{i\in A_{t+1}} \frac{2\bar \sigma_{A_{t+1}, i}^2}{n_{t,(i,i)}} + \sum_{(i,j)\in A_{t+1}}\frac{(2d+h_{t, \delta}^2/2)\|B\|^2_\infty}{n_{t,(i,j)}^2}\\
        &\hspace{1cm}+ \sum_{(i,j)\in A_{t+1}} \frac{5h_{t, \delta}\|B\|^2_{\infty}/2}{n_{t,(i,j)}^{3/2}} + \sum_{(i,j)\in A_{t+1}}\frac{\|B\|^2_{\infty}/2}{n_{t,(i,j)}^3}\,.
    \end{align*}
    The desired inequality just comes from $f_{T, \delta}\geq f_{t, \delta}$
\end{proof}

\subsection{Proof of Lemma~\ref{lem:DeltaGtC}}
\DeltaGtC*

\begin{proof}
Let $t\geq d(d+1)/2$ and $\delta>0$. The error in estimating the mean reward for action $a$ with $\langle a, \hat \mu_t\rangle$ is bounded as
\begin{align*}
        \big|a^\top(\hat{\mu}_t-\mu) \big| 
        &\leq \|\diagcounts_t^{-1}a\|_{\Rdesignmat_t} \ \big\|\textstyle \sum_{s=1}^t \diagmat_{A_s} \eta_s\big\|_{\Rdesignmat_t^{-1}}\,.
\end{align*}

The definition of $\cG_t = \big\{\ \big\|\sum_{s=1}^t \diagmat_{A_s} \eta_s\big\|_{\Rdesignmat_t^{-1}} \leq f_{t, \delta} \big\}$ yields that in this event,
\[
    \langle A_{t+1}, \hat{\mu}_t \rangle \leq \langle A_{t+1}, \mu \rangle + f_{t, \delta}\|\diagcounts_t^{-1}A_{t+1}\|_{\Rdesignmat_t}\,,
\]
and
\[
\langle a^*, \mu \rangle - f_{t, \delta}\|\diagcounts_t^{-1}a^*\|_{\Rdesignmat_t} \leq \langle a^*, \hat{\mu}_t \rangle \,.
\]

By the definition of $A_{t+1}$ for \OLSUCBC in~\eqref{eq:OLSUCBC},
\begin{equation*}
  \hspace*{-5pt}  \langle a^*, \hat{\mu}_t \rangle + f_{t, \delta}\|\diagcounts_t^{-1}a^*\|_{\hat \Rdesignmat_t} 
     \leq \langle A_{t+1}, \hat{\mu}_t \rangle + 
    f_{t, \delta}\|\diagcounts_t^{-1}A_{t+1}\|_{\hat \Rdesignmat_t}\,.
\end{equation*}
Combining the expressions gives
\begin{align*}
    &\langle a^*, \mu \rangle + f_{t, \delta}(\|\diagcounts_t^{-1}a^*\|_{\hat \Rdesignmat_t}-\|\diagcounts_t^{-1}a^*\|_{\Rdesignmat_t}) \leq \langle A_{t+1}, \mu \rangle + f_{t, \delta}(\|\diagcounts_t^{-1}A_{t+1}\|_{\Rdesignmat_t} + \|\diagcounts_t^{-1}A_{t+1}\|_{\hat \Rdesignmat_t})\,.
\end{align*}

Now, we use the fact that under $\cC_t$, $\hat{\Rdesignmat}_t$ uses coefficient-wise upper bounds of $\SSigma$, which yields that
$$
\|\diagcounts_t^{-1}a^*\|_{\Rdesignmat_t}^2 \leq \|\diagcounts_t^{-1}a^*\|_{\hat \Rdesignmat_t}^2\,.
$$

Rearranging terms the desired result.
\end{proof}

\subsection{Proof of the gap-free bound}
\label{app:OLSUCBCGapFree}
\RegretOLSUCBC*

\begin{proof}
     Let \(\Delta>0\), then
\begin{align*}
    \sum_{t=d(d+1)}^{T-1}\Delta_{A_{t+1}}\indicator\{\cG_{t}\cap \cC_t\}
    &= \sum_{t=d(d+1)}^{T-1}\Delta_{A_{t+1}}\indicator\Big\{\cG_{t}\cap\cC_t\cap (\Delta_{A_{t+1}}\leq \Delta)\Big\} \\
    & \hspace{3cm} + \sum_{t=d(d+1)}^{T-1}\Delta_{A_{t+1}}\indicator\Big\{\cG_{t}\cap\cC_t\cap (\Delta_{A_{t+1}}>\Delta)\Big\}\\
    &\leq T\Delta + \sum_{t=d(d+1)}^{T-1}\Delta_{A_{t+1}}\indicator\Big\{\cG_{t}\cap\cC_t\cap (\Delta_{A_{t+1}}>\Delta)\Big\}\,.
\end{align*}
Adapting Proposition \ref{prop:HPRegret} to account for \(\Delta_{A_{t+1}}>\Delta\) yields
\begin{align*}
    \sum_{t=0}^{T-1} \Delta_{A_{t+1}}&\indicator\{\cG_{t}\cap\cC_t\cap(\Delta_{A_{t+1}}>\Delta)\} 
    \lesssim \frac{1}{\Delta} \log(T)^{2}\log(m)^2\sum_{i\in [d]}\Big(\max_{a\in\cA / i\in a}\sigma_{a,i}^2\Big)\,.
\end{align*}
where $\lesssim$ is an inequality up to constant factors (when $T$ varies).

Balancing $T\Delta$ and $ \frac{1}{\Delta} \log(T)^2\log(m)^2\sum_{i\in [d]}\Big(\max_{a\in\cA / i\in a}\sigma_{a,i}^2\Big)$ yields
\begin{equation*}
    \bE[R_T] = O\Bigg(\log(m)\log(T)\sqrt{T}\sqrt{\sum_{i\in [d]}\max_{a\in\cA/i\in a}\sigma_{a,i}^2}\Bigg)\,.
\end{equation*}
\end{proof}

\section{Details for \COSV (Section~\ref{sec:TechnicalCOSV})}
\label{app:COSV}

\subsection{Proof for Lemma~\ref{lem:ProbaH}}
\begin{restatable}{lemma}{ProbaH}
\label{lem:ProbaH}
Let $\delta>0$. Then {\COSV\normalfont} satisfies 
$$
\sum_{s=d(d+A)/2}^T\bP(\cH_t^c) \leq \delta\sum_{t=1}^T \frac{1}{t\log(t)^2}.
$$
\end{restatable}

\begin{proof}
    Let $\delta>0$, $t\geq d(d+1)/2$. We remind
\begin{equation}
    \textstyle \cH_t= \Bigg\{\forall i\in [d],\ \bigg|\bigg(\hat\mu_{t,i}+(1+g_{t,\delta})f_{t,\delta}\frac{(\hat{\Rdesignmat}_{t,i})^{1/2}}{n_{t,i}}\bigg)-\tilde\mu_{t,i}\bigg|\leq g_{t,\delta}f_t\frac{(\hat{\Rdesignmat}_{t,i})^{1/2}}{n_{t,i}} \Bigg\}\,,
\end{equation}
where $g_{t,\delta}=\Big(2\log\big(2dt\log(t)^2\big)+\log(1/\delta)\Big)^{1/2}$. 

Conditionally to $\cF_{t} = \sigma(A_1, Y_1, \dots, A_{t}, Y_{t})$, for all $i\in[d]$
\begin{equation*}
    \textstyle \tilde\mu_{t,i} \sim \cN \bigg(\hat{\mu}_{t,i}+(1+g_{t,\delta})f_{t,\delta}\frac{\hat{\Rdesignmat}_{t,(i,i)}^{1/2}}{n_{t,(i,i)}},\,\,
    f_{t,\delta}^2\frac{\hat{\Rdesignmat}_{t,(i,i)}}{n_{t,(i,i)}^2}\bigg)\,.
\end{equation*}

Let $i\in a^*$. Then Gaussian concentration yields
\begin{align*}
    \bP_{\cF_t}\Bigg(\Bigg|\Big(\hat{\mu}_{t,i} + (1+g_{t,\delta})f_{t,\delta}\frac{\hat{\Rdesignmat}_{t,(i,i)}^{1/2}}{n_{t,(i,i)}}\Big) - \tilde\mu_{t,i}\Bigg| > \sqrt{2\log(2dt\log(t)^2/\delta)}f_{t,\delta}\frac{\hat{\Rdesignmat}_{t,(i,i)}^{1/2}}{n_{t,(i,i)}}\Bigg) \leq \frac{\delta}{dt\log(t)^2}\,,
\end{align*}
and 
\begin{align*}
    \bP\Bigg(\Bigg| \Big(\hat{\mu}_{t,i} + (1+g_{t,\delta})f_{t,\delta}\frac{\hat{\Rdesignmat}_{t,(i,i)}^{1/2}}{n_{t,(i,i)}}\Big) - \tilde\mu_{t,i}\Bigg|> \sqrt{2\log(2dt\log(t)^2/\delta)}f_{t,\delta}\frac{\hat{\Rdesignmat}_{t,(i,i)}^{1/2}}{n_{t,(i,i)}}\Bigg) \leq \frac{\delta}{dt\log(t)^2}\,.
\end{align*}
by integration.

A union bound on $i\in[d]$ and $t\geq d(d+1)/2$ yields the result
\begin{align*}
\sum_{t=d(d+1)/2}^T\bP(\cH_t^c)\leq \sum_{t\in[T]}\frac{\delta}{t(\log(t)^2}\,.
\end{align*}
\end{proof}

\subsection{Proof for Proposition~\ref{prop:COSVHPRegret}}
\label{app:COSVHPRegret}
\COSVHPRegret*

\begin{proof}
    Let $\delta>0$. We first make use of the following Lemma.
    
\begin{restatable}{lemma}{DecompDeltaSampling}
    \label{lem:DecompDeltaSampling}
    Let $t\geq d(d+1)/2$, $\delta>0$. Then for {\normalfont \COSV}, under $\{\cG_t \cap \cC_t \cap \cH_t\}$,
    \begin{align}
    \frac{\Delta_{A_{t+1}}}{f_{T, \delta}^2g_{T, \delta}^2}
    &\leq \sum_{i\in A_{t+1}} \frac{40m\SSigma_{i,i}}{n_{t,(i,i)}} +  \sum_{i\in A_{t+1}}\frac{29mdh_{t, \delta}^2\|B\|_\infty^2}{n_{t,(i,i)}^2}\\
    &\hspace{2cm}+\sum_{i\in A_{t+1}}\frac{45mh_{t, \delta}\|B\|_{\infty}^2}{n_{t, (i,i)}^{3/2}} + \sum_{i\in A_{t+1}}\frac{9m\|B\|_{\infty}^2}{n_{t, (i,i)}^3}\,.
    \end{align}
\end{restatable}

This enables to use a ``modified'' version of \cref{prop:HPRegret}, which do not consider covariances.
\begin{restatable}{proposition}{HPRegret2}
    \label{prop:HPRegret2}
    Let $r\in\bN$, $e\in(1, +\infty)^r$. Let $(\cE_t)_{t\geq d(d+1)/2}$ be a sequence of events such that for all $t\geq d(d+1)/2$, under $\cE_t$,
    \begin{equation}
        \frac{\Delta_{A_{t+1}}^2}{C} \leq \sum_{i\in A_{t+1}}\frac{C_{i}}{n_{t,(i,j)}} + \sum_{s\in[r]}\Bigg[\sum_{i\in A_{t+1}} \frac{C_s}{n_{t,(i,j)}^{e_s}}\Bigg]\,
    \end{equation}
    where $C$ and $(C_s)_{s\in[r]}$ are problem-dependent positive constants. $C_i$ is a positive constant depending on $i$ so that, $C_{i}\leq 2m\SSigma_{i,i}$. Let $c\in\bR_+^*$ and $(c_s)_{s\in[r]}\in(\bR_+^*)^r$ be positive constants such that $1/c+\sum_{s\in[r]}1/c_s=1$.
    
    Then,
    \begin{align}
        &\sum_{t=d(d+1)/2}^{T-1} \Delta_{A_{t+1}}\indicator\{\cE_t\}\notag\\
    &\hspace{1cm}\leq 96c_1C\log(m)^2\sum_{i\in [d]}\Big(\frac{C_{i}}{\Delta_{i,\min }}\Big)\notag\\
    &\hspace{1.5cm}+\sum_{s=1}^r \Bigg[\indicator\Big\{e_s=2\Big\}346\Big(c_sCC_s\log(m)\Big)^{1/2}m^{3/2}d\Bigg(1 + \log\Big(\frac{\Delta_{\max}}{\Delta_{\min}}\Big)\Bigg) \notag\\
    &\hspace{2cm} + \indicator\Big\{1<e_s<2\Big\}60.30^{1/e_s}\Big(c_sCC_s\log(m)\Big)^{1/e_s}dm^{1+1/e_s}\Delta_{\min}^{1-2/e_s}\notag\\
    &\hspace{2cm} + \indicator\Big\{2<e_s\Big\}60.30^{1/e_s}\Big(c_sCC_s\log(m)\Big)^{1/e_s} \frac{e_s}{e_s-2}dm^{1+1/e_s}\Delta_{\max}^{1-2/e_s}\Bigg]\,,
    \end{align}
    where $(\alpha_k)_{k\in\bN^*}$, $(\beta_k)_{k\in\bN^*}$ and $k_0\in\bN^*$ are defined in \cref{app:DefAlphaBeta}.
\end{restatable}

Applied to \COSV, this yields
\begin{align*}
    &\sum_{t=d(d+1)/2}^{T-1} \Delta_{A_{t+1}}\indicator\big\{\cG_{t}\cap \cC\cap\cH\big\}\notag\\
    &\hspace{1cm}\leq 15360f_{T, \delta}^2g_{T, \delta}^2\log(m)^2\sum_{i\in [d]}\Big(\frac{m\SSigma_{i,i}}{\Delta_{i, \min }}\Big) \notag\\
    &\hspace{1.5cm}+ 3727f_{T, \delta}g_{T, \delta}h_{T, \delta}(\log(m))^{1/2}\|B\|_{\infty}m^2d^2\Bigg(1 + \log\Big(\frac{\Delta_{\max}}{\Delta_{\min}}\Big)\Bigg)\notag\\
    &\hspace{1.5cm}+ 7329(f_{T, \delta}g_{T, \delta})^{4/3}h_{T, \delta}^{2/3}\log(m)^{2/3}\|B\|^{4/3}_{\infty}m^{5/3}d\Delta_{\min}^{-1/3}\\
    &\hspace{1.5cm}+ 3745(f_{T, \delta}g_{T, \delta})^{2/3}\log(m)^{1/3}\|B\|_{\infty}^{2/3}m^{4/3}d\Delta_{\max}^{1/3}\,,
\end{align*}
where
\begin{itemize}[nosep, label=]
    \item $\smash f_{t,\delta} = 6\log(1/\delta) + 6\Big(\log(t)+ (d+2)\log(\log(t))\Big) + 3d\Big(2\log(2) +\log(1+e)\Big)$,
    \item $\smash h_{t,\delta} = (1+2\log(1/\delta)+2\log(d(d+1))+\log(1+t))^{1/2}$,
    \item $\smash g_{t,\delta}=(1+\log(2dt\log(t)^2)+\log(1/\delta))^{1/2}$. 
\end{itemize}

We deduce 
\begin{align*}
    \bE\bigg[\sum_{t=d(d+1)}^{T-1}\Delta_{A_{t+1}}\indicator\big\{\cG_{t}\cap \cC\cap\cH\big\}\bigg] = O \Bigg(  \log(T)^{3}\log (m)^2\Big(\sum_{i=1}^d \frac{m\SSigma_{i,i}}{\Delta_{i,\min}} \Big)\Bigg)\,.
\end{align*}

\end{proof}

\subsection{Proof for Lemma~\ref{lem:DecompDeltaSampling}}
\DecompDeltaSampling*

\begin{proof}
    Let $t\geq d(d+1)/2$ and $\delta>0$. Then
\begin{align*}
    \Delta_{A_{t+1}} 
    &= \langle a^*-A_{t+1}, \mu\rangle\\
    &= \langle a^*, \mu-\tilde \mu_{t}\rangle + \langle a^*-A_{t+1}, \tilde \mu_t\rangle + \langle A_{t+1}, \tilde\mu_t-\mu\rangle\\
    &\leq \langle a^*, \mu-\tilde \mu_{t}\rangle + \langle A_{t+1}, \tilde\mu_t-\mu\rangle\,
\end{align*}
by definition of $A_{t+1}$.

Besides,
\begin{align*}
    \langle a^*,\mu-\tilde\mu_t\rangle
    &= \sum_{i\in a^*} \mu_i-\tilde\mu_{t,i}\\
    &= \sum_{i\in a^*} \Big(\mu_i-\hat\mu_{t,i} + \hat\mu_{t,i}-\tilde\mu_{t,i}\Big)\,.
\end{align*}

Under $\cG_t\cap\cC_t$, for all $i\in a^*$, 
\begin{align*}
    \mu_i-\hat\mu_{t,i} \leq f_{t,\delta}\frac{\Rdesignmat_{t,i}^{1/2}}{n_{t,(i,i)}}.
\end{align*}
Under $\cH_t$, for all $i\in a^*$, 
\begin{align*}
    \hat \mu_{t,i} + (1+g_{t, \delta})f_{t,\delta}\frac{\hat \Rdesignmat_{t,i}^{1/2}}{n_{t,(i,i)}}-\tilde\mu_{t,i} 
    &\leq g_{t, \delta}f_{t,\delta}\frac{\hat \Rdesignmat_{t,i}^{1/2}}{n_{t,(i,i)}}\\
    \hat \mu_{t,i} -\tilde\mu_{t,i} &\leq -f_{t,\delta}\frac{\hat \Rdesignmat_{t,i}^{1/2}}{n_{t,(i,i)}}\,.
\end{align*}

Therefore, under $\{\cG_t\cap\cH_t\cap\cC_t\}$,
\begin{align*}
    0\leq \Delta_{A_{t+1}} \leq \langle A_{t+1}, \tilde\mu_t-\mu\rangle\,.
\end{align*}

We now develop the expression
\begin{align*}
    \Delta_{A_{t+1}}^2 
    &\leq \Big(\langle A_{t+1}, \tilde \mu_{t}-\mu\rangle\Big)^2\\
    &\leq 2\Big(\langle A_{t+1}, \tilde \mu_{t}-\hat\mu_t\rangle\Big)^2 + 2\Big(\langle A_{t+1}, \hat\mu_{t}-\mu\rangle\Big)^2\\
    &\leq 2\Big(\sum_{i\in A_{t+1}} \tilde\mu_{t,i}-\hat\mu_{t,i}\Big)^2 + 2f_{t, \delta}^2\|\diagcounts_{t}^{-1}A_{t+1}\|^2_{\Rdesignmat_t}\\
    &\leq 2m\sum_{i\in A_{t+1}} \Big(\tilde\mu_{t,i}-\hat\mu_{t,i}\Big)^2 + 2f_{t, \delta}^2\|\diagcounts_{t}^{-1}A_{t+1}\|^2_{\Rdesignmat_{t}}\\
    &\leq 2m(1+2g_{t, \delta})^2f_{t, \delta}^2\sum_{i\in A_{t+1}} \frac{\hat\Rdesignmat_{t, (i,i)}}{n_{t,(i,i)}^2} + 2f_{t, \delta}^2\|\diagcounts_{t}^{-1}A_{t+1}\|^2_{\Rdesignmat_{t}}\\
    &\leq 18mg_{t, \delta}^2f_{t, \delta}^2\sum_{i\in A_{t+1}} \frac{\hat\Rdesignmat_{t, (i,i)}}{n_{t,(i,i)}^2} + 2f_{t, \delta}^2\|\diagcounts_{t}^{-1}A_{t+1}\|^2_{\Rdesignmat_{t}}\,.
\end{align*}

As
\begin{align*}
    \sum_{i\in A_{t+1}} \frac{\hat\Rdesignmat_{t, (i,i)}}{n_{t,(i,i)}^2} 
    &= 2\sum_{i\in A_{t+1}} \frac{\hat \SSigma_{i,i}}{n_{t,(i,i)}} + \sum_{i\in A_{t+1}}\frac{d\|B\|_\infty^2}{n_{t,(i,i)}^2}\\
    &\leq 2\sum_{i\in A_{t+1}} \frac{\SSigma_{i,i}}{n_{t,(i,i)}} +  \sum_{i\in A_{t+1}}\frac{d\|B\|_\infty^2}{n_{t,(i,i)}^2} +\sum_{i\in A_{t+1}}\frac{ \frac{5}{2}h_{t, \delta}\|B\|_{\infty}^2}{n_{t, (i,i)}^{3/2}} \\
    &\quad \quad+ \sum_{i\in A_{t+1}} \frac{h_{t, \delta}^2\|B\|_{\infty}^2/2}{n_{t,(i,i)}^2} + \sum_{i\in A_{t+1}}\frac{\|B\|_{\infty}^2/2}{n_{t, (i,i)}^3}\\
    &=\sum_{i\in A_{t+1}} \frac{2\SSigma_{i,i}}{n_{t,(i,i)}} +  \sum_{i\in A_{t+1}}\frac{(d+h_{t, \delta}^2/2)\|B\|_\infty^2}{n_{t,(i,i)}^2} \\
    &\quad \quad + \sum_{i\in A_{t+1}}\frac{ \frac{5}{2}h_{t, \delta}\|B\|_{\infty}^2}{n_{t, (i,i)}^{3/2}} + \sum_{i\in A_{t+1}}\frac{\|B\|_{\infty}^2/2}{n_{t, (i,i)}^3}\,,
\end{align*}
and
\begin{align*}
     \|\diagcounts_t^{-1}A_{t+1}\|_{\Rdesignmat_t}^2 
    &= \sum_{(i,j)\in A_{t+1}} \frac{n_{t,(i,j)}\SSigma_{i,j}}{n_{t,(i,i)}n_{t, (j,j)}} + \sum_{i\in A_{t+1}}\frac{n_{t(i,i)}\SSigma_{i,i}}{n_{t,(i,i)}^2} + \sum_{i\in A_{t+1}}\frac{\|B\|^2}{n_{t,(i,i)}^2}\\
    &\leq \sum_{(i,j)\in A_{t+1}} \frac{n_{t,(i,j)}\sqrt{\SSigma_{i,i}}\sqrt{\SSigma_{j,j}}}{n_{t,(i,i)}n_{t, (j,j)}} + \sum_{i\in A_{t+1}}\frac{\SSigma_{i,i}}{n_{t,(i,i)}} + \sum_{i\in A_{t+1}}\frac{d\|B\|^2_{\infty}}{n_{t,(i,i)}^2}\\
    &\leq \sum_{(i,j)\in A_{t+1}} \frac{n_{t,(i,j)}(\SSigma_{i,i}+\SSigma_{j,j})/2}{n_{t,(i,i)}n_{t, (j,j)}} + \sum_{i\in A_{t+1}}\frac{\SSigma_{i,i}}{n_{t,(i,i)}} + \sum_{i\in A_{t+1}}\frac{d\|B\|^2_{\infty}}{n_{t,(i,i)}^2}\\
    &\leq \sum_{i\in A_{t+1}} \frac{(m+1)\SSigma_{i,i}}{n_{t,(i,i)}} + \sum_{i\in A_{t+1}}\frac{d\|B\|^2_{\infty}}{n_{t,(i,i)}^2}\,.
\end{align*}

Therefore,
{\allowdisplaybreaks
\begin{align*}
    \Delta_{A_{t+1}}^2 
    &\leq \sum_{i\in A_{t+1}} \frac{36mg_{t, \delta}^2f_{t, \delta}^2\SSigma_{i,i}}{n_{t,(i,i)}} +  \sum_{i\in A_{t+1}}\frac{9mg_{t, \delta}^2f_{t, \delta}^2(2d+h_{t, \delta}^2)\|B\|_\infty^2}{n_{t,(i,i)}^2}\\
    &\hspace{1.5cm}+\sum_{i\in A_{t+1}}\frac{45g_{t, \delta}^2f_{t, \delta}^2h_{t, \delta}\|B\|_{\infty}^2}{n_{t, (i,i)}^{3/2}} + \sum_{i\in A_{t+1}}\frac{9mg_{t, \delta}^2f_{t, \delta}^2\|B\|_{\infty}^2}{n_{t, (i,i)}^3}\\
    &\hspace{1.5cm}+\sum_{i\in A_{t+1}} \frac{4mf_{t, \delta}^2\SSigma_{i,i}}{n_{t,(i,i)}} + \sum_{i\in A_{t+1}}\frac{2f_{t, \delta}^2d\|B\|^2_{\infty}}{n_{t,(i,i)}^2}\\
    &\leq \sum_{i\in A_{t+1}} \frac{4mf_{t, \delta}^2\SSigma_{i,i}(9g_{t, \delta}^2+1)}{n_{t,(i,i)}} +  \sum_{i\in A_{t+1}}\frac{f_{t, \delta}^2\|B\|_\infty^2\Big(27mdg_{t, \delta}^2h_{t, \delta}^2+2d\Big)}{n_{t,(i,i)}^2}\\
    &\hspace{1.5cm}+\sum_{i\in A_{t+1}}\frac{45g_{t, \delta}^2f_{t, \delta}^2h_{t, \delta}\|B\|_{\infty}^2}{n_{t, (i,i)}^{3/2}} + \sum_{i\in A_{t+1}}\frac{9mg_{t, \delta}^2f_{t, \delta}^2\|B\|_{\infty}^2}{n_{t, (i,i)}^3}\\
    &\leq \sum_{i\in A_{t+1}} \frac{40mf_{t, \delta}^2g_{t, \delta}^2\SSigma_{i,i})}{n_{t,(i,i)}} +  \sum_{i\in A_{t+1}}\frac{29mdf_{t, \delta}^2\|B\|_\infty^2g_{t, \delta}^2h_{t, \delta}^2}{n_{t,(i,i)}^2}\\
    &\hspace{1.5cm}+\sum_{i\in A_{t+1}}\frac{45g_{t, \delta}^2f_{t, \delta}^2h_{t, \delta}\|B\|_{\infty}^2}{n_{t, (i,i)}^{3/2}} + \sum_{i\in A_{t+1}}\frac{9mg_{t, \delta}^2f_{t, \delta}^2\|B\|_{\infty}^2}{n_{t, (i,i)}^3}\,.
\end{align*}
}

This finally yields
\begin{align*}
    \frac{\Delta_{A_{t+1}}}{f_{T, \delta}^2g_{T, \delta}^2}
    &\leq \sum_{i\in A_{t+1}} \frac{40m\SSigma_{i,i}}{n_{t,(i,i)}} +  \sum_{i\in A_{t+1}}\frac{29mdh_{t, \delta}^2\|B\|_\infty^2}{n_{t,(i,i)}^2}\\
    &\qquad+\sum_{i\in A_{t+1}}\frac{45mh_{t, \delta}\|B\|_{\infty}^2}{n_{t, (i,i)}^{3/2}} + \sum_{i\in A_{t+1}}\frac{9m\|B\|_{\infty}^2}{n_{t, (i,i)}^3}\,.
\end{align*}
\end{proof}

\section{Experimental results}
\label{app:Expe}

This section outlines some experimental results.

\subsection{Theoretical regret upper bound}

In this experiment, the objective is to show the effect of the smallest suboptimality gap $\Delta_{\min}$ over theoretical gap-dependent regret upper bounds for \ESCBC and \OLSUCBC.
To that end, we sampled $100$ environments with different $\Delta_{\min }$, with a constant number of items $d=20$, a horizon of $T=10^5$ rounds, and randomly sampled structures. We represent theoretical upper bounds with respect to $1/\Delta_{\min }$ in \cref{fig:th_regret}. 

For readability reasoning, we have rescaled and reweighted the different components of the sums so that the leading term in the upper-bounds ($1/\Delta_{\min }$ or $1/\Delta_{\min }^2$ for \ESCBC or \OLSUCBC) is greater/smaller than the rest, in a significant number of cases. In particular, all the theoretical upper bounds have the form

\begin{align*}
    R_T &\leq \frac{C}{\Delta_{\min }} + \frac{C'}{\Delta_{\min }^2} + C_r\text{Rest}\,,
\end{align*}
where $C,\ C'$ and $C_r$ are the tuned constants. 

For \OLSUCBC, 
\begin{align*}
    \text{Rest} 
    &= \Delta_{\max }(d(d+1)/2)\\
    &\hspace{1.5cm}+ \|B\|_\infty f_{T, \delta}(4d+h_{t, \delta}^2)^{1/2}\log(m)^{1/2}\Bigg(1 + \log\Big(\frac{\Delta_{\max}}{\Delta_{\min}}\Big)\Bigg)\notag\\
    &\hspace{1.5cm} + \|B\|_\infty^{4/3}f_{T, \delta}^{4/3}h_{t, \delta}^{2/3}\log(m)^{2/3}d^2m^{2/3}\Delta_{\min}^{-1/3}\\
    &\hspace{1.5cm} + \|B\|_\infty^{2/3}f_{T,\delta}^{2/3}\log(m)^{1/3}d^2m^{2/3}\Delta_{\max}^{1/3}\,.
\end{align*}

For \ESCBC, 
\begin{align*}
    \text{Rest} 
    &= \Delta_{\max }(d(d+1)/2)\\
    &\hspace{1.5cm}+\log(T)\log(m)^2\sum_{i\in [d]}\frac{\max_{a\in\cA / i\in a}\bar \sigma_{a,i}^2}{\Delta_{i_{\min }}} + \log(T)\log(m)\sum_{i, j}\log\Big(\frac{\Delta_{(i, j), \max}}{\Delta_{(i, j), \min}}\Big)\notag\\
    &\hspace{1.5cm} + \log(T)\log(m)\sum_{i}\log\Big(\frac{\Delta_{i, \max}}{\Delta_{i, \min}}\Big)+ \log(T)\log(m)\sum_{i, j}\Delta_{(i, j), \min}^{-1/3}\,.
\end{align*}

\begin{figure}[ht]
    \centering
    \includegraphics[width=.75\linewidth]{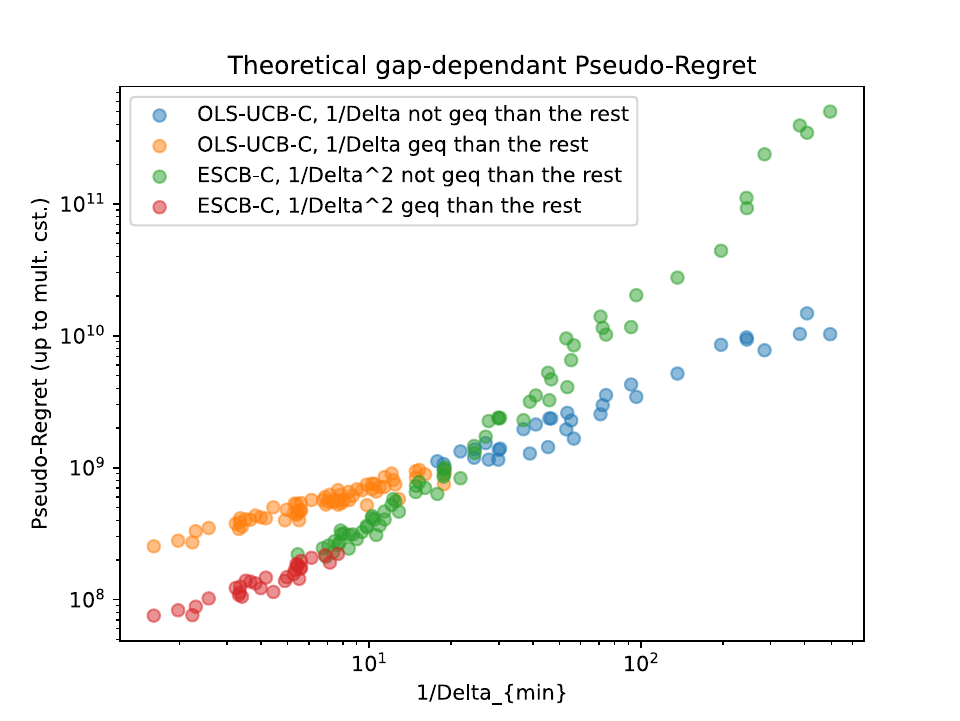}
    \caption{Evolution of regret upper bounds.}
    \label{fig:th_regret}
\end{figure}

When the minimal gap is too small (right part of Fig.~\ref{fig:th_regret}), both upper-bounds are of the magnitude of either $1/\Delta_{\min }^2$ or $1/\Delta_{\min }$ (depending on the algorithm). In this case, the theoretical regret bound of \OLSUCBC outperforms the one of \ESCBC (green dots vs. blue dots).  On the other side, when the gap is big enough, the remaining terms have more impact. In this case, \ESCBC has a better theoretical guarantee (orange dots vs. red dots). 

\subsection{Comparison between \ESCBC and \OLSUCBC}

We evaluate \ESCBC (approximated as proposed in \citealp{perrault2020covariance}) and \OLSUCBC on $d=5$ items, $P=10$ actions, $T=10^5$ rounds and randomly sampled structures.

We represent the pseudo-regret evolutions in \cref{fig:pregrets}. The evolutions remain the same until $10^3$ rounds. After that, \ESCBC seemingly performs better than \OLSUCBC which has a supplementary $\log(t)$ factor and is more conservative. However, just before $10^5$ rounds, we can observe a slight regime change for \ESCBC while the pseudo-regret of \OLSUCBC continues to increase smoothly. The average regret of \ESCBC seems to have an inflexion point upward to meet the q75 curve.

\begin{figure}[ht]
    \centering
    \includegraphics[width=.75\linewidth]{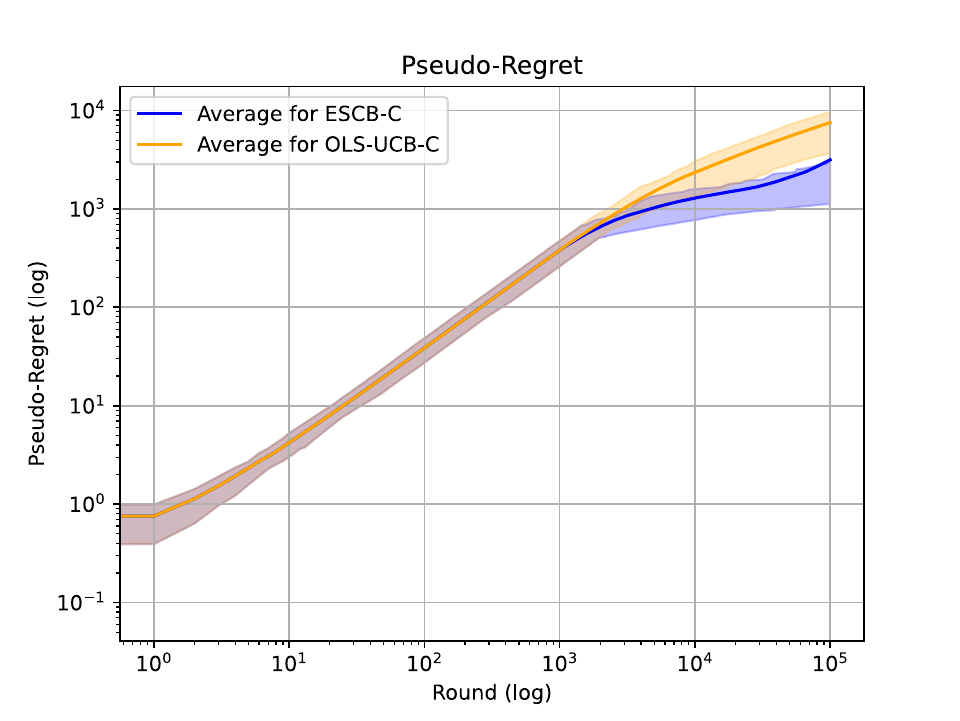}
    \caption{Pseudo-regret for \ESCBC and \OLSUCBC for randomly sampled environments (with q25 and q75 confidence intervals).}
    \label{fig:pregrets}
\end{figure}

When observing the final regret with respect to $1/\Delta_{\min }$ in \cref{fig:pregret_delta}, overall \ESCBC seems to outperform \OLSUCBC except on some corner cases. Those cases skew the distribution for \ESCBC. Especially, for the case with the smallest suboptimality gap (the rightmost part of the figure), \OLSUCBC  outperforms \ESCBC.

\begin{figure}[ht]
    \centering
    \includegraphics[width=0.8\linewidth]{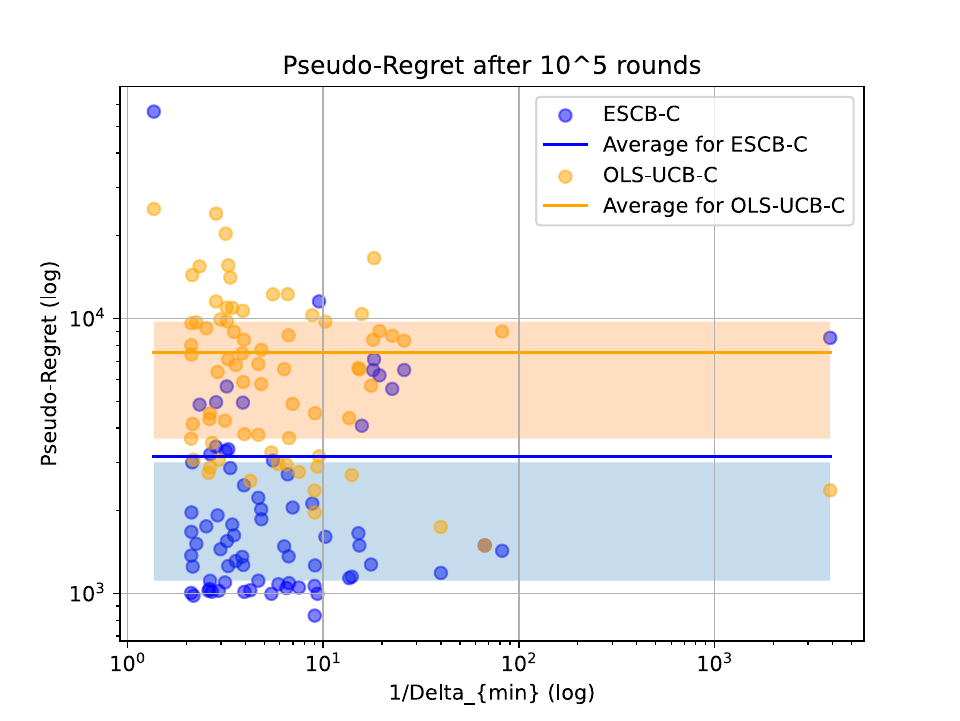}
    \caption{Pseudo-Regret with respect to $1/\Delta_{\min }$.}
    \label{fig:pregret_delta}
\end{figure}

The evolution of the pseudo-regret in this case with the smallest suboptimality gap is presented in \cref{fig:pregret_worst}. While \ESCBC seems to fare better in the beginning, we actually see a sharp increase in its pseudo-regret before $10^5$ rounds. It could have been caused by the computational approximation of \ESCBC (described in \citet{perrault2020covariance}), and/or it could be the impact of the $1/\Delta_{\min }^2$ term.

\begin{figure}[ht]
    \centering
    \includegraphics[width=0.8\linewidth]{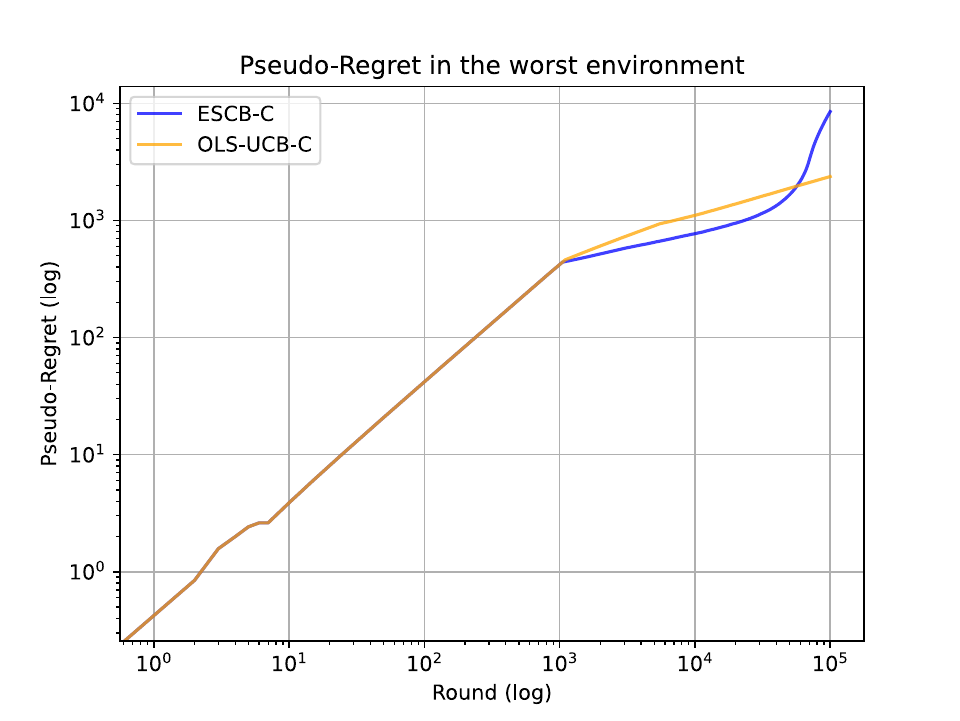}
    \caption{Pseudo-Regret in the ``worst'' environment.}
    \label{fig:pregret_worst}
\end{figure}

\newpage
\section*{NeurIPS Paper Checklist}

\begin{enumerate}

\item {\bf Claims}
    \item[] Question: Do the main claims made in the abstract and introduction accurately reflect the paper's contributions and scope?
    \item[] Answer: \answerYes{}{} 
    \item[] Justification: The main claims of the paper are stated in Theorems and Propositions which for which the technical proofs are in the Appendix.
    \item[] Guidelines:
    \begin{itemize}
        \item The answer NA means that the abstract and introduction do not include the claims made in the paper.
        \item The abstract and/or introduction should clearly state the claims made, including the contributions made in the paper and important assumptions and limitations. A No or NA answer to this question will not be perceived well by the reviewers. 
        \item The claims made should match theoretical and experimental results, and reflect how much the results can be expected to generalize to other settings. 
        \item It is fine to include aspirational goals as motivation as long as it is clear that these goals are not attained by the paper. 
    \end{itemize}

\item {\bf Limitations}
    \item[] Question: Does the paper discuss the limitations of the work performed by the authors?
    \item[] Answer: \answerYes 
    \item[] Justification: The results and theorems are discussed with comments made especially about some limitations. In particular, computational complexity is addressed.
    \item[] Guidelines:
    \begin{itemize}
        \item The answer NA means that the paper has no limitation while the answer No means that the paper has limitations, but those are not discussed in the paper. 
        \item The authors are encouraged to create a separate "Limitations" section in their paper.
        \item The paper should point out any strong assumptions and how robust the results are to violations of these assumptions (e.g., independence assumptions, noiseless settings, model well-specification, asymptotic approximations only holding locally). The authors should reflect on how these assumptions might be violated in practice and what the implications would be.
        \item The authors should reflect on the scope of the claims made, e.g., if the approach was only tested on a few datasets or with a few runs. In general, empirical results often depend on implicit assumptions, which should be articulated.
        \item The authors should reflect on the factors that influence the performance of the approach. For example, a facial recognition algorithm may perform poorly when image resolution is low or images are taken in low lighting. Or a speech-to-text system might not be used reliably to provide closed captions for online lectures because it fails to handle technical jargon.
        \item The authors should discuss the computational efficiency of the proposed algorithms and how they scale with dataset size.
        \item If applicable, the authors should discuss possible limitations of their approach to address problems of privacy and fairness.
        \item While the authors might fear that complete honesty about limitations might be used by reviewers as grounds for rejection, a worse outcome might be that reviewers discover limitations that aren't acknowledged in the paper. The authors should use their best judgment and recognize that individual actions in favor of transparency play an important role in developing norms that preserve the integrity of the community. Reviewers will be specifically instructed to not penalize honesty concerning limitations.
    \end{itemize}

\item {\bf Theory Assumptions and Proofs}
    \item[] Question: For each theoretical result, does the paper provide the full set of assumptions and a complete (and correct) proof?
    \item[] Answer: \answerYes{} 
    \item[] Justification: The Theorems and proofs are numbered and cross-referenced. The proofs for the main results are outlined in the paper and formally written in the Appendix.
    \item[] Guidelines:
    \begin{itemize}
        \item The answer NA means that the paper does not include theoretical results. 
        \item All the theorems, formulas, and proofs in the paper should be numbered and cross-referenced.
        \item All assumptions should be clearly stated or referenced in the statement of any theorems.
        \item The proofs can either appear in the main paper or the supplemental material, but if they appear in the supplemental material, the authors are encouraged to provide a short proof sketch to provide intuition. 
        \item Inversely, any informal proof provided in the core of the paper should be complemented by formal proofs provided in appendix or supplemental material.
        \item Theorems and Lemmas that the proof relies upon should be properly referenced. 
    \end{itemize}

    \item {\bf Experimental Result Reproducibility}
    \item[] Question: Does the paper fully disclose all the information needed to reproduce the main experimental results of the paper to the extent that it affects the main claims and/or conclusions of the paper (regardless of whether the code and data are provided or not)?
    \item[] Answer: \answerNA{} 
    \item[] Justification: The paper is mainly states theoretical results and include no experiments.
    \item[] Guidelines:
    \begin{itemize}
        \item The answer NA means that the paper does not include experiments.
        \item If the paper includes experiments, a No answer to this question will not be perceived well by the reviewers: Making the paper reproducible is important, regardless of whether the code and data are provided or not.
        \item If the contribution is a dataset and/or model, the authors should describe the steps taken to make their results reproducible or verifiable. 
        \item Depending on the contribution, reproducibility can be accomplished in various ways. For example, if the contribution is a novel architecture, describing the architecture fully might suffice, or if the contribution is a specific model and empirical evaluation, it may be necessary to either make it possible for others to replicate the model with the same dataset, or provide access to the model. In general. releasing code and data is often one good way to accomplish this, but reproducibility can also be provided via detailed instructions for how to replicate the results, access to a hosted model (e.g., in the case of a large language model), releasing of a model checkpoint, or other means that are appropriate to the research performed.
        \item While NeurIPS does not require releasing code, the conference does require all submissions to provide some reasonable avenue for reproducibility, which may depend on the nature of the contribution. For example
        \begin{enumerate}
            \item If the contribution is primarily a new algorithm, the paper should make it clear how to reproduce that algorithm.
            \item If the contribution is primarily a new model architecture, the paper should describe the architecture clearly and fully.
            \item If the contribution is a new model (e.g., a large language model), then there should either be a way to access this model for reproducing the results or a way to reproduce the model (e.g., with an open-source dataset or instructions for how to construct the dataset).
            \item We recognize that reproducibility may be tricky in some cases, in which case authors are welcome to describe the particular way they provide for reproducibility. In the case of closed-source models, it may be that access to the model is limited in some way (e.g., to registered users), but it should be possible for other researchers to have some path to reproducing or verifying the results.
        \end{enumerate}
    \end{itemize}

\item {\bf Open access to data and code}
    \item[] Question: Does the paper provide open access to the data and code, with sufficient instructions to faithfully reproduce the main experimental results, as described in supplemental material?
    \item[] Answer: \answerNA{} 
    \item[] Justification: The paper does not include experiments requiring code.
    \item[] Guidelines:
    \begin{itemize}
        \item The answer NA means that paper does not include experiments requiring code.
        \item Please see the NeurIPS code and data submission guidelines (\url{https://nips.cc/public/guides/CodeSubmissionPolicy}) for more details.
        \item While we encourage the release of code and data, we understand that this might not be possible, so “No” is an acceptable answer. Papers cannot be rejected simply for not including code, unless this is central to the contribution (e.g., for a new open-source benchmark).
        \item The instructions should contain the exact command and environment needed to run to reproduce the results. See the NeurIPS code and data submission guidelines (\url{https://nips.cc/public/guides/CodeSubmissionPolicy}) for more details.
        \item The authors should provide instructions on data access and preparation, including how to access the raw data, preprocessed data, intermediate data, and generated data, etc.
        \item The authors should provide scripts to reproduce all experimental results for the new proposed method and baselines. If only a subset of experiments are reproducible, they should state which ones are omitted from the script and why.
        \item At submission time, to preserve anonymity, the authors should release anonymized versions (if applicable).
        \item Providing as much information as possible in supplemental material (appended to the paper) is recommended, but including URLs to data and code is permitted.
    \end{itemize}

\item {\bf Experimental Setting/Details}
    \item[] Question: Does the paper specify all the training and test details (e.g., data splits, hyperparameters, how they were chosen, type of optimizer, etc.) necessary to understand the results?
    \item[] Answer: \answerNA{}{} 
    \item[] Justification: The answer NA means that the paper does not include experiments.
    \item[] Guidelines:
    \begin{itemize}
        \item The answer NA means that the paper does not include experiments.
        \item The experimental setting should be presented in the core of the paper to a level of detail that is necessary to appreciate the results and make sense of them.
        \item The full details can be provided either with the code, in appendix, or as supplemental material.
    \end{itemize}

\item {\bf Experiment Statistical Significance}
    \item[] Question: Does the paper report error bars suitably and correctly defined or other appropriate information about the statistical significance of the experiments?
    \item[] Answer: \answerNA{} 
    \item[] Justification: The paper does not include experiments
    \item[] Guidelines:
    \begin{itemize}
        \item The answer NA means that the paper does not include experiments.
        \item The authors should answer "Yes" if the results are accompanied by error bars, confidence intervals, or statistical significance tests, at least for the experiments that support the main claims of the paper.
        \item The factors of variability that the error bars are capturing should be clearly stated (for example, train/test split, initialization, random drawing of some parameter, or overall run with given experimental conditions).
        \item The method for calculating the error bars should be explained (closed form formula, call to a library function, bootstrap, etc.)
        \item The assumptions made should be given (e.g., Normally distributed errors).
        \item It should be clear whether the error bar is the standard deviation or the standard error of the mean.
        \item It is OK to report 1-sigma error bars, but one should state it. The authors should preferably report a 2-sigma error bar than state that they have a 96\% CI, if the hypothesis of Normality of errors is not verified.
        \item For asymmetric distributions, the authors should be careful not to show in tables or figures symmetric error bars that would yield results that are out of range (e.g. negative error rates).
        \item If error bars are reported in tables or plots, The authors should explain in the text how they were calculated and reference the corresponding figures or tables in the text.
    \end{itemize}

\item {\bf Experiments Compute Resources}
    \item[] Question: For each experiment, does the paper provide sufficient information on the computer resources (type of compute workers, memory, time of execution) needed to reproduce the experiments?
    \item[] Answer: \answerNA{} 
    \item[] Justification: The paper does not include experiments.
    \item[] Guidelines:
    \begin{itemize}
        \item The answer NA means that the paper does not include experiments.
        \item The paper should indicate the type of compute workers CPU or GPU, internal cluster, or cloud provider, including relevant memory and storage.
        \item The paper should provide the amount of compute required for each of the individual experimental runs as well as estimate the total compute. 
        \item The paper should disclose whether the full research project required more compute than the experiments reported in the paper (e.g., preliminary or failed experiments that didn't make it into the paper). 
    \end{itemize}
    
\item {\bf Code Of Ethics}
    \item[] Question: Does the research conducted in the paper conform, in every respect, with the NeurIPS Code of Ethics \url{https://neurips.cc/public/EthicsGuidelines}?
    \item[] Answer: \answerYes{} 
    \item[] Justification: We have reviewed the NeurIPS Code of Ethics. The paper does not involve human subjects or participants, and the data-related concerns are not applicable.
    \item[] Guidelines:
    \begin{itemize}
        \item The answer NA means that the authors have not reviewed the NeurIPS Code of Ethics.
        \item If the authors answer No, they should explain the special circumstances that require a deviation from the Code of Ethics.
        \item The authors should make sure to preserve anonymity (e.g., if there is a special consideration due to laws or regulations in their jurisdiction).
    \end{itemize}

\item {\bf Broader Impacts}
    \item[] Question: Does the paper discuss both potential positive societal impacts and negative societal impacts of the work performed?
    \item[] Answer: \answerNA{} 
    \item[] Justification: The paper is mainly theoretical and is not directly tied to an application.
    \item[] Guidelines:
    \begin{itemize}
        \item The answer NA means that there is no societal impact of the work performed.
        \item If the authors answer NA or No, they should explain why their work has no societal impact or why the paper does not address societal impact.
        \item Examples of negative societal impacts include potential malicious or unintended uses (e.g., disinformation, generating fake profiles, surveillance), fairness considerations (e.g., deployment of technologies that could make decisions that unfairly impact specific groups), privacy considerations, and security considerations.
        \item The conference expects that many papers will be foundational research and not tied to particular applications, let alone deployments. However, if there is a direct path to any negative applications, the authors should point it out. For example, it is legitimate to point out that an improvement in the quality of generative models could be used to generate deepfakes for disinformation. On the other hand, it is not needed to point out that a generic algorithm for optimizing neural networks could enable people to train models that generate Deepfakes faster.
        \item The authors should consider possible harms that could arise when the technology is being used as intended and functioning correctly, harms that could arise when the technology is being used as intended but gives incorrect results, and harms following from (intentional or unintentional) misuse of the technology.
        \item If there are negative societal impacts, the authors could also discuss possible mitigation strategies (e.g., gated release of models, providing defenses in addition to attacks, mechanisms for monitoring misuse, mechanisms to monitor how a system learns from feedback over time, improving the efficiency and accessibility of ML).
    \end{itemize}
    
\item {\bf Safeguards}
    \item[] Question: Does the paper describe safeguards that have been put in place for responsible release of data or models that have a high risk for misuse (e.g., pretrained language models, image generators, or scraped datasets)?
    \item[] Answer: \answerNA{} 
    \item[] Justification: The paper poses no such risks.
    \item[] Guidelines:
    \begin{itemize}
        \item The answer NA means that the paper poses no such risks.
        \item Released models that have a high risk for misuse or dual-use should be released with necessary safeguards to allow for controlled use of the model, for example by requiring that users adhere to usage guidelines or restrictions to access the model or implementing safety filters. 
        \item Datasets that have been scraped from the Internet could pose safety risks. The authors should describe how they avoided releasing unsafe images.
        \item We recognize that providing effective safeguards is challenging, and many papers do not require this, but we encourage authors to take this into account and make a best faith effort.
    \end{itemize}

\item {\bf Licenses for existing assets}
    \item[] Question: Are the creators or original owners of assets (e.g., code, data, models), used in the paper, properly credited and are the license and terms of use explicitly mentioned and properly respected?
    \item[] Answer: \answerNA{} 
    \item[] Justification: The paper does not use existing assets.
    \item[] Guidelines: 
    \begin{itemize}
        \item The answer NA means that the paper does not use existing assets.
        \item The authors should cite the original paper that produced the code package or dataset.
        \item The authors should state which version of the asset is used and, if possible, include a URL.
        \item The name of the license (e.g., CC-BY 4.0) should be included for each asset.
        \item For scraped data from a particular source (e.g., website), the copyright and terms of service of that source should be provided.
        \item If assets are released, the license, copyright information, and terms of use in the package should be provided. For popular datasets, \url{paperswithcode.com/datasets} has curated licenses for some datasets. Their licensing guide can help determine the license of a dataset.
        \item For existing datasets that are re-packaged, both the original license and the license of the derived asset (if it has changed) should be provided.
        \item If this information is not available online, the authors are encouraged to reach out to the asset's creators.
    \end{itemize}

\item {\bf New Assets}
    \item[] Question: Are new assets introduced in the paper well documented and is the documentation provided alongside the assets?
    \item[] Answer: \answerNA{} 
    \item[] Justification: The paper does not release new assets.
    \item[] Guidelines:
    \begin{itemize}
        \item The answer NA means that the paper does not release new assets.
        \item Researchers should communicate the details of the dataset/code/model as part of their submissions via structured templates. This includes details about training, license, limitations, etc. 
        \item The paper should discuss whether and how consent was obtained from people whose asset is used.
        \item At submission time, remember to anonymize your assets (if applicable). You can either create an anonymized URL or include an anonymized zip file.
    \end{itemize}

\item {\bf Crowdsourcing and Research with Human Subjects}
    \item[] Question: For crowdsourcing experiments and research with human subjects, does the paper include the full text of instructions given to participants and screenshots, if applicable, as well as details about compensation (if any)? 
    \item[] Answer: \answerNA{} 
    \item[] Justification: The paper does not involve crowdsourcing nor research with human subjects.
    \item[] Guidelines:
    \begin{itemize}
        \item The answer NA means that the paper does not involve crowdsourcing nor research with human subjects.
        \item Including this information in the supplemental material is fine, but if the main contribution of the paper involves human subjects, then as much detail as possible should be included in the main paper. 
        \item According to the NeurIPS Code of Ethics, workers involved in data collection, curation, or other labor should be paid at least the minimum wage in the country of the data collector. 
    \end{itemize}

\item {\bf Institutional Review Board (IRB) Approvals or Equivalent for Research with Human Subjects}
    \item[] Question: Does the paper describe potential risks incurred by study participants, whether such risks were disclosed to the subjects, and whether Institutional Review Board (IRB) approvals (or an equivalent approval/review based on the requirements of your country or institution) were obtained?
    \item[] Answer: \answerNA{} 
    \item[] Justification: The paper does not involve crowdsourcing nor research with human subjects.
    \item[] Guidelines:
    \begin{itemize}
        \item The answer NA means that the paper does not involve crowdsourcing nor research with human subjects.
        \item Depending on the country in which research is conducted, IRB approval (or equivalent) may be required for any human subjects research. If you obtained IRB approval, you should clearly state this in the paper. 
        \item We recognize that the procedures for this may vary significantly between institutions and locations, and we expect authors to adhere to the NeurIPS Code of Ethics and the guidelines for their institution. 
        \item For initial submissions, do not include any information that would break anonymity (if applicable), such as the institution conducting the review.
    \end{itemize}

\end{enumerate}

\end{document}